\documentclass[12pt,a4paper]{article}
\usepackage{fancyhdr}
\usepackage{titlesec}
\usepackage{titletoc}
\usepackage{graphicx}
  \DeclareGraphicsExtensions{.eps}
\usepackage{amsmath}
\usepackage{amstext}
\usepackage{amsbsy}
\usepackage{amssymb}
\usepackage{theorem}
\usepackage{dsfont}
\usepackage{times}
\usepackage[T1]{fontenc}
\usepackage{eucal}
\usepackage{mathrsfs}
\usepackage{url}
\newcommand{\ds}{\displaystyle}
\newcommand{\ts}{\textstyle}
\newcommand{\tint}{{\ts \int}}

\renewcommand{\mathbb}{\mathds}

\newcommand{\C}[1]{\mathcal{#1}}

\newcommand{\ov}[1]{\overline{#1}}
\newcommand{\wt}[1]{\widetilde{#1}}
\newcommand{\wh}[1]{\widehat{#1}}
\newcommand{\B}[1]{\mathds{#1}}

\DeclareMathOperator{\Tr}{Tr}

\DeclareMathOperator*{\ess}{ess}

\DeclareMathOperator{\Var}{\mathbb{V}ar}

\DeclareMathOperator{\sign}{sign}
\DeclareMathOperator{\Span}{span}

\numberwithin{equation}{section}
\newtheorem{thm}{Theorem}[section]
\newtheorem{prop}[thm]{Proposition}
\newtheorem{proposition}[thm]{Proposition}
\newtheorem{lemma}[thm]{Lemma}

\newtheorem{corollary}[thm]{Corollary}
{\theorembodyfont{\rm} 

\newtheorem{rmk}{Remark}[section]}
\theoremstyle{plain}

\theoremheaderfont{\sc}
\newenvironment{proof}{{\sc Proof.}}{\ $\square$}
\titleformat{\section}[block]{\sc\center}{\thetitle.}{5pt}{}[]
\titlespacing{\section}{0pt}{*4.5}{*3}
\titleformat{\subsection}[runin]{\sc}{\thetitle.}{5pt}{}[.]
\titlespacing{\subsection}{0pt}{*3}{*2}
\titleformat{\subsubsection}[runin]{\it}{\thetitle.}{5pt}{}[.]
\titlespacing{\subsubsection}{0pt}{*2}{*2}
\contentsmargin[30pt]{30pt}
\titlecontents{section}[15pt]{\vspace{6pt}\sc}
{\contentslabel[\thecontentslabel.]{15pt}}{\hspace*{-15pt}}
{\titlerule*[1pc]{.}\contentspage[\rm\thecontentspage]}
\titlecontents{subsection}[37pt]{\vspace{3pt}\small\sc}
{\contentslabel[\thecontentslabel.]{22pt}}{\hspace*{-22pt}}
{\titlerule*[1pc]{.}\contentspage[\normalsize\rm\thecontentspage]}
\titlecontents{subsubsection}[69pt]{\it}
{\contentslabel[\thecontentslabel.]{32pt}}{\hspace*{-32pt}}
{~\titlerule*[1pc]{.}\contentspage[\rm\thecontentspage]}

\newcommand{\thmref}[1]{\ref{#1} (page \pageref{#1})}
\newcommand{\myeq}[1]{{\rm (\ref{#1}, page \pageref{#1})}}

\allowdisplaybreaks

\DeclareMathOperator{\XX}{\text{\bf \textsf X}}
\DeclareMathOperator{\YY}{\text{\bf \textsf Y}}
\begin{document}
\renewcommand{\sectionmark}[1]{\markboth{\thesection\ #1}{}}
\renewcommand{\subsectionmark}[1]{\markright{\thesubsection\ #1}}
\fancyhf{}
\fancyhead[RE]{\small\sc\nouppercase{\leftmark}}
\fancyhead[LO]{\small\sc\nouppercase{\rightmark}}
\fancyhead[LE,RO]{\thepage}
\fancyfoot[RO,LE]{\small\sc Olivier Catoni $\Rightarrow$ Jean-Yves Audibert}
\fancyfoot[LO,RE]{\small\sc\today}
\renewcommand{\footruleskip}{1pt}
\renewcommand{\footrulewidth}{0.4pt}
\newcommand{\mypoint}{\makebox[1ex][r]{.\:\hspace*{1ex}}}
\addtolength{\footskip}{11pt}
\pagestyle{plain}
\begin{center}
{\bf Risk bounds in linear regression through PAC-Bayesian truncation}\\[12pt]
{\sc Jean-Yves Audibert\footnote{Universit\'e Paris-Est, Ecole des Ponts ParisTech, Imagine,
6 avenue Blaise Pascal, 77455 Marne-la-Vall\'ee, France, audibert@imagine.enpc.fr
}$^{\!,}$\footnote{Willow, CNRS/ENS/INRIA --- UMR 8548, 45 rue d'Ulm, F75230 Paris cedex 05, France}, Olivier Catoni
\footnote{D\'epartement de Math\'ematiques et Applications, CNRS -- UMR 8553, 
\'Ecole Normale Sup\'erieure,
45 rue d'Ulm, F75230 Paris cedex 05, olivier.catoni@ens.fr}}\\[12pt]
{\small \it \today }\\[12pt]
\end{center}

{\small
{\sc Abstract :} 
We consider the problem of predicting as well as the best linear combination of $d$ given functions in least squares regression, and variants of this problem including constraints on the parameters of the linear combination. 
When the input distribution is known, there already exists an algorithm having an expected excess risk of order $d/n$, where $n$ is the size of the training data. Without this strong assumption, standard results often contain a multiplicative $\log n$ factor, and require
some additional assumptions like uniform boundedness of the $d$-dimensional input representation and exponential moments of the output.

This work provides new risk bounds for the ridge estimator and the ordinary least squares estimator, and their variants. It also provides shrinkage procedures with convergence rate $d/n$ (i.e., without the logarithmic factor) in expectation and in deviations, under various assumptions. The key common surprising factor of these results is the absence of exponential moment condition on the output distribution while achieving exponential deviations.
All risk bounds are obtained through a PAC-Bayesian analysis on truncated differences of losses. Finally, we show that some of these results are not particular to the least squares loss, but can be generalized to similar strongly convex loss functions. 
\\[12pt]
{\sc 2000 Mathematics Subject Classification:}
62J05, 62J07.\\[12pt]
{\sc Keywords:} 
Linear regression, Generalization error, Shrinkage, PAC-Bayesian
theorems, Risk bounds, Robust statistics, Resistant estimators,
Gibbs posterior distributions,  
Randomized estimators, Statistical learning theory
}
\tableofcontents

\def\bmul{\begin{multline*}}
\def\bigbegar{\begin{eqnarray*}} 
\def\bigendar{\end{eqnarray*}}
\def\lbegar{$$\left\{ \begin{array}{lll}}
\def\rendar{\end{array} \right.$$}
\def\rendarp{\end{array} \right..$$}
\def\begarlab{\begin{equation} \begin{array}{lll} \label}
\def\endarlab{\end{array} \end{equation}}
\def\beglab{\begin{equation} \label}
\def\endlab{\end{equation}}
\newcommand\beglabc[1]{\begin{equation*} }
\def\endlabc{\end{equation*}}
\def\lbegarlab{\begin{equation} \left\{ \begin{array}{lll} \label}
\def\rendarlab{\end{array} \right. \end{equation}}
\def\rendarplab{\end{array} \right.. \end{equation}}

\newcommand\und[2]{\underset{#2}{#1}\;}
\newcommand\undc[2]{{#1}_{#2}\;}

\newcommand\wrt{\text{w.r.t.}} 

\newcommand\cst{\text{Cst}\,}
\newcommand\dsV{{\mathbb V}} 
\newcommand\eqdef{\triangleq}

\newcommand\A{\mathcal{A}}
\newcommand\cB{\mathcal{B}}
\newcommand\cC{\Theta}
\newcommand\cD{\mathcal{D}}
\newcommand\E{\mathbb{E}}
\newcommand\cE{\mathcal{E}}
\newcommand\cF{\mathcal{F}}
\newcommand\cH{\mathcal{H}}
\newcommand\I{\mathcal{I}}
\newcommand\J{\mathcal{J}}
\newcommand\cK{\mathcal{K}}
\newcommand\cL{\mathcal{L}}
\newcommand\M{\mathcal{M}}
\newcommand\N{\mathcal{N}}
\renewcommand\P{\mathbb{P}}
\newcommand\R{\mathbb{R}}
\newcommand\calR{\mathcal{R}}
\renewcommand\S{\mathcal{S}}
\newcommand\cS{\mathcal{S}}
\newcommand\bcR{B} 
\newcommand\cR{\tilde{B}} 
\newcommand\W{\mathcal{W}}
\newcommand\X{\mathcal{X}}
\newcommand\Y{\mathcal{Y}}
\newcommand\Z{\mathcal{Z}}

\newcommand\jyem{\em}

\newcommand\hdelta{\hat{\delta}}
\newcommand\hth{\hat{\theta}}
\newcommand\hbth{\bar{\th}}
\newcommand\tth{\tilde{\theta}}
\newcommand\hlam{\hat{\lam}}
\newcommand\lamerm{\hlam^{\textnormal{(erm)}}}

\newcommand\tb{\tilde{b}}
\newcommand\tc{\tilde{c}}
\newcommand\tf{\tilde{f}}
\newcommand\tlam{\tilde{\lam}}

\newcommand\be{\beta}
\newcommand\ga{\gamma}
\newcommand\kap{\kappa}
\newcommand\lam{\lambda}

\newcommand\sint{{\textstyle \int}}
\newcommand\inth{\sint \pi(d\theta)}
\newcommand\inthj{\sint \pij(d\theta)}
\newcommand\inthp{\sint \pi(d\theta')}

\newcommand\gas{\ga^*}

\newcommand\hpi{\hat{\pi}}
\newcommand\pis{\pi^*}
\newcommand\tpi{\tilde{\pi}}

\newcommand\ovth{\ov{\theta}}
\newcommand\wth{\wt{\theta}}
\newcommand\wthj{\wt{\theta}_j}

\newcommand\eps{\varepsilon}
\renewcommand\epsilon{\varepsilon}
\newcommand\logeps{\log(\eps^{-1})}
\newcommand\leps{\log^2(\eps^{-1})}

\newcommand{\bi}{\begin{itemize}}
\newcommand{\ei}{\end{itemize}}

\newcommand\ovR{\ov{R}}

\newcommand\hhpi{\hat{\hat{\pi}}}
\newcommand\wwth{\wt{\wt{\theta}}}
\newcommand\pia{\pi^{(1)}}
\newcommand\pib{\pi^{(2)}}
\newcommand\pij{\pi^{(j)}}
\newcommand\tpia{\tilde{\pi}^{(1)}}
\newcommand\tpib{\tilde{\pi}^{(2)}}
\newcommand\tpij{\tilde{\pi}^{(j)}}
\newcommand\wtha{\wt{\theta}_1}
\newcommand\wthb{\wt{\theta}_2}
\newcommand\wths{\wt{\theta}_s}
\newcommand\wtht{\wt{\theta}_t}

\newcommand\cmin{c_{\min}}
\newcommand\cmax{c_{\max}}

\newcommand\ra{\rightarrow}

\newcommand\hatt{\hat{t}}
\newcommand\argmax{\textnormal{argmax}}
\newcommand\argmin{\textnormal{argmin}}

\newcommand\diag{\textnormal{Diag}}

\newcommand\Fg{\cF^\#}

\newcommand\vp{\varphi}
\newcommand\tvp{\tilde{\vp}}
\newcommand\tphi{\tilde{\phi}}
\renewcommand\th{\theta}

\newcommand\vsp{\vspace{1cm}}
\newcommand\lhs{\text{l.h.s.}}
\newcommand\rhs{\text{r.h.s.}}

\newcommand\expe[2]{\undc{\E}{#1\sim#2}}
\newcommand\expec[2]{\undc{\E}{#1\sim#2}}
\newcommand\expecc[2]{\E_{#1}}
\newcommand\expecd[2]{\E_{#2}}

\newcommand\hC{\hat{C}}
\newcommand\hf{\hat{f}}
\newcommand\hrho{\hat{\rho}}

\newcommand\br{\bar{r}}
\newcommand\chr{\check{r}}
\newcommand\bR{\bar{R}}

\newcommand\lan{\langle}
\newcommand\ran{\rangle}

\newcommand\logepsg{\log(|\Cg|\eps^{-1})}

\newcommand\Pemp{\hat{\P}}

\newcommand\fracl[2]{{(#1)}/{#2}}
\newcommand\fracc[2]{{#1}/{#2}}
\newcommand\fracr[2]{{#1}/{(#2)}}
\newcommand\fracb[2]{{(#1)}/{(#2)}}

\newcommand\hfproj{\hf^{\textnormal{(proj)}}}
\newcommand\thproj{\hat{\th}^{\textnormal{(proj)}}}

\newcommand\hfols{\hf^{\textnormal{(ols)}}}
\newcommand\thfols{\tilde{f}^{\textnormal{(ols)}}}
\newcommand\thols{\hat{\th}^{\textnormal{(ols)}}}
\newcommand\therm{\hat{\th}^{\textnormal{(erm)}}}
\newcommand\hferm{\hf^{\textnormal{(erm)}}}
\newcommand\zols{\zeta^{\textnormal{(ols)}}}

\newcommand\thrid{\tilde{\th}} 
\newcommand\frid{\tilde{f}} 
\newcommand\freg{f^{\textnormal{(reg)}}}
\newcommand\thrlam{\hth^{\textnormal{(ridge)}}} 
\newcommand\hfrlam{\hf^{\textnormal{(ridge)}}} 
\newcommand\thllam{\hth^{\textnormal{(lasso)}}} 
\newcommand\hfllam{\hf^{\textnormal{(lasso)}}} 

\newcommand\flin{f^*_{\textnormal{lin}}}
\newcommand\thlin{\th^{\textnormal{(lin)}}}
\newcommand\Flin{\mathcal{F}_{\textnormal{lin}}}

\renewcommand\Phi{\XX}
\newcommand\demi{\frac{1}{2}}
\newcommand\demic{\fracc{1}{2}}

\newcommand\substa[2]{\substack{#1\\#2}}
\newcommand\substac[2]{{#1\,;\,#2}}
\newcommand\tpsi{\tilde{\psi}}
\newcommand\tzeta{\tilde{\zeta}}
\newcommand\ta{\tilde{a}}
\newcommand\chis{\chi_\sigma}
\newcommand\tchi{\tilde{\chi}}
\newcommand\tchis{\tchi_\sigma}
\newcommand\psis{\psi_\sigma}

\newcommand\tA{\tilde{A}}
\newcommand\tL{\tilde{L}}

\newcommand\hL{\hat{L}}
\newcommand\hcE{\hat{\cE}}
\newcommand\hDe{\hat{\cE}}
\newcommand\cEb{\cE^{\sharp}}
\newcommand\La{L^{\flat}}
\newcommand\cEa{\cE^{\flat}}
\newcommand\Lb{L^{\sharp}}

\newcommand\Pa{P^{\flat}}
\newcommand\Pb{P^{\sharp}}

\newcommand{\V}[1]{\overline{#1}}
\newcommand\bL{\V{L}}

\newcommand\ela{\tilde{\ell}}
\newcommand\hpig{\hpi^{\textnormal{(Gibbs)}}}

\newcommand\sigmb{\phi}
\newcommand{\tR}{\tilde{\cR}}
\newcommand{\logdeps}{\log(4d\eps^{-1})}
\newcommand{\logddeps}{\log(2d^2\eps^{-1})}

\newcommand{\piobs}{\bar{\pi}}

\newcommand\cdd{\mathcal{D}'}

\section*{Introduction}
\addcontentsline{toc}{section}{Introduction}
\subsection*{Our statistical task}
\addcontentsline{toc}{subsection}{Our statistical task}
Let $Z_1=(X_1,Y_1),\dots,Z_n=(X_n,Y_n)$ be $n\ge 2$ pairs of input-output 
and assume that each pair has been independently drawn from the same unknown distribution $P$. Let 
$\X$ denote the input space and let the output space be the set of real numbers $\R$, so that $P$ is a probability distribution on the product space 
$\Z \eqdef \X\times\R$. 
The target of learning algorithms is to predict the output $Y$ associated with an input $X$
for pairs $Z=(X,Y)$ drawn from the distribution $P$. 
The quality of a (prediction) function 
$f:\X\rightarrow\R$ is measured by the least squares {\jyem risk}: 
        $$R(f) \eqdef \undc{\E}{Z\sim P} \bigl\{ [Y-f(X)]^2 
\bigr\} .$$
Through the paper, we assume that the output and all the prediction functions we consider are square integrable.
Let $\cC$ be a closed convex set of $\R^d$, and $\vp_1,\dots,\vp_d$ be $d$ prediction functions. 
Consider the regression model
	\begin{align*}
	\cF= \bigg\{ f_\th=\sum_{j=1}^d \th_j \vp_j ; (\th_1,\dots,\th_d) \in \cC \bigg\}.
	\end{align*}
The best function $f^*$ in $\cF$ is defined by
	\begin{align*}
	f^* =\sum_{j=1}^d \th^*_j \vp_j\in\und{\argmin}{f\in\cF} \,R(f).
	\end{align*}
Such a function always exists but is not necessarily unique. Besides it is unknown since the probability generating the data is unknown.

We will study the problem of predicting (at least) as well as function $f^*$. In other words, we want to deduce from the observations
$Z_1,\dots,Z_n$ a function $\hf$ having with high probability a risk bounded by the minimal risk $R(f^*)$ on $\cF$ plus a small remainder term, which is typically of order $d/n$ up to a possible logarithmic factor. 
Except in particular settings (e.g., $\cC$ is a simplex and $d\ge \sqrt{n}$), 
it is known that the convergence rate $d/n$ cannot be improved in a minimax sense (see \cite{Tsy03}, and \cite{Yan01b} for related results).

More formally, the target of the paper is to develop estimators $\hf$ for which the excess risk is controlled {\jyem in deviations}, i.e., such that
for an appropriate constant $\kap>0$, for any $\eps>0$, with probability at least $1-\eps$,
	\beglab{eq:devtarget}
	R(\hf) - R(f^*) \le \kap \frac{d+\logeps}{n}.
	\endlab
Note that by integrating the deviations (using the identity 
$\E W = \int_0^{+\infty} \P(W > t) dt$ 
which holds true for any nonnegative random variable $W$), 
Inequality \eqref{eq:devtarget} implies 
	\beglab{eq:exptarget}
	\E R(\hf) - R(f^*) \le \kap \frac{d+1}{n}.
	\endlab
In this work, we do not assume that the function 
	\[
	\freg:x \mapsto \E[ Y | X=x],
	\]
which minimizes the risk $R$ among all possible measurable functions,
belongs to the model $\cF$. So we might
have $f^* \neq \freg$ and in this case, bounds of the form
	\beglab{eq:nottarget}
	\E R(\hf) - R(\freg) \le C [R(f^*)-R(\freg)] + \kap \frac{d}{n},
	\endlab
with a constant $C$ larger than $1$ do not even ensure that $\E R(\hf)$ tends to $R(f^*)$ when $n$ goes to infinity.
This kind of bounds with $C>1$ have been developed to analyze nonparametric estimators using linear approximation spaces,
in which case the dimension $d$ is a function of $n$ chosen so that the bias term $R(f^*)-R(\freg)$ has the order $d/n$ of the estimation term
(see \cite{Gyo04} and references within). Here we intend to assess the
generalization ability of the estimator even when the model is misspecified 
(namely when $R(f^*) > R(\freg)$).
Moreover we do not assume either that $Y-\freg(X)$ and $X$ are independent.

\textbf{Notation.} \hspace*{2mm}
When $\cC=\R^d$, 
the function $f^*$ and the space $\cF$ will be written $\flin$ and $\Flin$ to emphasize that 
$\cF$ is the whole linear space spanned by $\vp_1,\dots,\vp_d$:
\[  \Flin = \Span \{\vp_1,\dots,\vp_d\} \text{\qquad and\qquad} \flin\in\und{\argmin}{f\in\Flin} R(f).\]
The Euclidean norm will simply be written as $\|\cdot\|$, and $\langle \cdot,\cdot\rangle$
will be its associated inner product.
We will consider the vector valued function
$\vp : \C{X} \rightarrow \B{R}^d$ defined
by $\vp(X) = \bigl[ \vp_k(X) \bigr]_{k=1}^d$, so that for any $\th\in\Theta$, we have
  \[
  f_\th(X) = \langle \th , \vp(X) \rangle.
  \]
The Gram matrix is the $d\times d$-matrix $Q=\B{E} \bigl[ \vp(X) \vp(X)^T\bigr]$, and its smallest and largest eigenvalues will 
respectively be written as $q_{\min}$ and $q_{\max}$. 
The empirical risk of a function $f$ is
$$
r(f) = \frac{1}{n} \sum_{i=1}^n \bigl[ f(X_i) - 
Y_i \bigr]^2
$$
and for $\lam\ge 0$, the ridge regression estimator on $\cF$ is defined by
$\hfrlam=f_{\thrlam}$ with
$$
\thrlam \in \arg \min_{\th \in \Theta} 
r(f_{\theta}) + \lambda \lVert \theta \rVert^2,
$$
where $\lam$ is some nonnegative real 
parameter. In the case when $\lambda = 0$, 
the ridge regression $\hfrlam$ is nothing
but the empirical risk minimizer $\hferm$.
In the same way, we introduce the optimal ridge function  
optimizing the expected ridge risk: $\frid=f_{\thrid}$ with
\beglab{eq:frid}
\thrid \in \arg \min_{\theta \in \Theta} \big\{ R(f_{\theta}) 
+ \lambda \lVert \theta \rVert^2 \big\}.
\endlab
Finally, let $Q_{\lambda} = Q + \lambda I$ be the ridge regularization 
of $Q$, where $I$ is the identity matrix. 


\subsection*{Why should we be interested in this task}
\addcontentsline{toc}{subsection}{Why should we be interested in this task}
There are three main reasons.
First we aim at a better understanding of the parametric linear least squares method 
(classical textbooks can be misleading on this subject as we will point out later), and intend to provide a non-asymptotic analysis of it.

Secondly, the task is central in nonparametric estimation for linear approximation spaces (piecewise polynomials based 
on a regular partition, wavelet expansions, trigonometric polynomials\dots)

Thirdly, it naturally arises in two-stage model selection. Precisely,
when facing the data, the statistician has often to choose 
several models which are likely to be relevant for the task.
These models can be of similar structures (like embedded balls of functional spaces) or
on the contrary of very different nature (e.g., based on kernels, splines, wavelets or on parametric approaches). 
For each of these models, we assume that we have a learning scheme which produces 
a 'good' prediction function in the sense that it predicts as well as the best 
function of the model up to some small additive term. Then the question is
to decide on how we use or combine/aggregate these schemes. 
One possible answer is to split the data into two groups, use
the first group to train the prediction function associated with each model, and
finally use the second group to build a prediction function which is as good as (i) the
best of the previously learnt prediction functions, (ii) the best convex combination
of these functions or (iii) the best linear combination of these functions. 
This point of view has been introduced by Nemirovski in \cite{Nem98} and optimal rates of aggregation are given in \cite{Tsy03} and references within. 
This paper focuses more on the linear aggregation task (even if (ii) enters in our setting), assuming implicitly here that
the models are given in advance and are beyond our control and that the goal is to combine them appropriately.

\subsection*{Outline and contributions}
\addcontentsline{toc}{subsection}{Outline and contributions}

The paper is organized as follows.
Section~\ref{sec:lit} is a survey on risk bounds in linear least squares.
Theorems \ref{th:alqa} and \ref{th:bmnew} are the results which come closer to our target.
Section~\ref{sec:ridge} provides a new analysis of the ridge estimator 
and the ordinary least squares estimator, and their variants. 
Theorem \ref{th:hfrlam} provides an asymptotic result for the ridge estimator 
while Theorem \ref{th:ermom} gives a non asymptotic risk bound
of the empirical risk minimizer, which is complementary to the theorems
put in the survey section. In particular, the result has the benefit to
hold for the ordinary least squares estimator and for heavy-tailed outputs.
We show quantitatively that the ridge penalty
leads to an implicit reduction of the input space dimension.
Section \ref{sec:computable} shows a non asymptotic $d/n$ exponential deviation 
risk bound under weak moment conditions on the output $Y$ and on the $d$-dimensional input representation $\vp(X)$.
Section~\ref{sec:main} presents stronger results under boundedness assumption of
$\vp(X)$. However the latter results are concerned with 
a not easily computable estimator.
Section~\ref{sec:gen} gives risk bounds for general 
loss functions from which the results 
of Section~\ref{sec:main} are derived. 

The main contribution of this paper is to show through a PAC-Bayesian analysis on truncated differences of losses that the output distribution does not need to have bounded
conditional exponential moments in order for the excess risk of appropriate
estimators to concentrate exponentially. 
Our results tend to say that truncation leads to more robust algorithms.
Local robustness to contamination is usually invoked 
to advocate the removal of outliers, 
claiming that estimators should be made insensitive to 
small amounts of spurious data.
Our work leads to a different theoretical explanation.
The observed points having unusually large outputs when compared
with the (empirical) variance should be down-weighted in the estimation 
of the mean, since they contain less information than noise. 
In short, huge outputs should be truncated because of their low 
signal to noise ratio.

\section{Variants of known results} \label{sec:lit}

\subsection{Ordinary least squares and empirical risk minimiza\-tion} 

The ordinary least squares estimator is the most standard method in this case. 
It minimizes the empirical risk 
$$
	r(f) = \frac{1}{n} \sum_{i=1}^n [Y_i-f(X_i)]^2,
$$
among functions in $\Flin$ and produces
$$
	\hfols= \sum_{j=1}^d \thols_j \vp_j,
$$
with $\thols=[\thols_j]_{j=1}^d$ a column vector satisfying
	\begarlab{eq:thols}
	\Phi^T\Phi \, \thols = \Phi^T \YY,
	\endarlab
where $\YY=[Y_j]_{j=1}^n$ and $\Phi=(\vp_j(X_i))_{ 1\le i \le n, 1\le j \le d}$. 
It is well-known that 
\bi
\item the linear system \eqref{eq:thols} has at least one solution, and in fact, the set of solutions is exactly $\{\Phi^+ \YY + u ; u \in \text{ker}\, \Phi\}$;
where $\Phi^+$ is the Moore-Penrose pseudoinverse of $\Phi$ and $\text{ker} \,\Phi$ is the kernel of the linear operator $\Phi$.
\item $\Phi \, \thols$ is the (unique) orthogonal projection of the vector $\YY\in\R^n$
on the image of the linear map $\Phi$;
\item if $\sup_{x\in\X} \Var(Y |X=x) = \sigma^2 < +\infty$, we have
(see \cite[Theorem 11.1]{Gyo04}) for any $X_1,\dots,X_n$ in $\X$,
	\begin{multline} \label{eq:fixeddesign}
	\E \bigg\{ \frac{1}{n} \sum_{i=1}^n \big[\hfols(X_i) - \freg(X_i)\big]^2 \bigg| X_1,\dots,X_n\bigg\}\\
		\qquad - \und{\min}{f\in\Flin} \frac{1}{n} \sum_{i=1}^n \big[f(X_i) - \freg(X_i)\big]^2 
		\le \sigma^2 \frac{\text{rank}(\Phi)}{n} \le \sigma^2 \frac{d}{n},
	\end{multline}
where we recall that $\freg: x\mapsto \E[ Y | X=x]$ is the optimal regression function, and 
that when this function belongs to $\Flin$ (i.e., $\freg=\flin$), the minimum term in \eqref{eq:fixeddesign} vanishes; 
\item from Pythagoras' theorem for the (semi)norm $W \mapsto \sqrt{\E W^2}$ on the space of the square integrable random variables, 
\begin{multline} \label{eq:pyt}
	R(\hfols) - R(\flin) \\ = \E \big[ \hfols(X) - \freg(X) \big| Z_1,\dots,Z_n \big]^2 - \E \big[ \flin(X) - \freg(X) \big]^2.
\end{multline}
\ei
The analysis of the ordinary least squares often stops at this point in classical statistical textbooks.
(Besides, to simplify, the strong assumption $\freg=\flin$ is often made.)
This can be misleading since Inequality \eqref{eq:fixeddesign} does not imply a $d/n$ upper bound on the risk
of $\hfols$. Nevertheless the following result holds \cite[Theorem 11.3]{Gyo04}.
\begin{thm}\label{th:weakols}
If $\,\sup_{x\in\X} \Var(Y |X=x) = \sigma^2 < +\infty$ 
and 
$$
\|\freg\|_\infty = \sup_{x\in\X} |\freg(x)| \le H
$$ for some $H>0$, then
the truncated estimator $\hfols_H=(\hfols\wedge H)\vee -H$ satisfies 
	\begin{multline} \label{eq:weakols}
	\E R(\hfols_H) - R(\freg) \le 8 [R(\flin)-R(\freg)] + \kap \frac{(\sigma^2\vee H^2)d\log n }{n} 
	\end{multline}
for some numerical constant $\kap$.
\end{thm}

Using PAC-Bayesian inequalities, 
Catoni \cite[Proposition 5.9.1]{Cat01} has proved a different type of results on the generalization ability of $\hfols$.

\begin{thm} \label{th:oc}
Let $\cF'\subset\Flin$ satisfying for some positive constants $a,M,M'$:
\bi
\item $\text{there exists } f_0 \in \cF' \text{ s.t. for any } x\in\X,$ 
  \[
  \E \Bigl\{ 
  \exp \Bigl[ a \bigl\lvert Y-f_0(X) \bigr\rvert \Bigr]\, \Big| \,X=x \Bigr\} \le M.
  \]
\item $\text{for any } f_1,f_2 \in \cF', \sup_{x\in\X} |f_1(x)-f_2(x)|\le M'$.
\ei
Let $Q=\B{E} \bigl[ \vp(X) \vp(X)^T\bigr]$
and $\hat{Q} = \bigl[ \frac{1}{n} \sum_{i=1}^n \vp(X_i) \vp(X_i)^T \bigr]$
be respectively the expected and empirical Gram matrices. If $\det Q\neq 0$, then
there exist positive constants $C_1$ and $C_2$ (depending only on $a$, $M$ and $M'$) such that
with probability at least $1-\eps$, 
as soon as
 	\beglab{eq:cond}
	\bigg\{
	f\in\Flin:r(f) \le r(\hfols) + C_1 \frac{d}{n}
	\bigg\} \subset \cF',
	\endlab
we have
	\[
	R(\hfols) - R(\flin) \le C_2 \frac{d+\logeps+\log(\frac{\det \hat{Q}}{\det Q})}{n}.
	\]
\end{thm}

This result can be understood as follows. 
Let us assume we have some prior knowledge suggesting 
that $\flin$ belongs to the interior of a set 
$\cF'\subset\Flin$ (e.g., a bound
on the coefficients of the expansion of $\flin$ as a linear
combination of $\vp_1, \dots, \vp_d$). 
It is likely that \eqref{eq:cond} holds, and it is indeed proved in 
Catoni \cite[section 5.11]{Cat01} that the probability that it does not
hold goes to zero exponentially fast with $n$ in the case when $\cF'$
is a Euclidean ball. 
If it is the case, then we know that the excess risk is of order $d/n$
up to the unpleasant ratio of determinants, which, 
fortunately, almost surely tends to $1$ as $n$ goes to infinity.

%
%

%
By using \emph{localized} PAC-Bayes inequalities introduced in Catoni 
\cite{Cat03b, Cat05},
one can derive from Inequality (6.9) and Lemma 4.1 of Alquier \cite{Alq08} the following result.

\begin{thm} \label{th:alqa}
Let $q_{\min}$ be the smallest eigenvalue of the Gram matrix $Q=\B{E} \bigl[ \vp(X) \vp(X)^T\bigr]$.
Assume that there exist a function $f_0 \in \Flin$ and positive constants $H$ and $C$ such that 
	\[
	\| \flin - f_0 \|_{\infty}  \le H.
	\]
and $|Y| \le C$ almost surely.

Then for an appropriate randomized estimator requiring the knowledge of $f_0$, $H$ and $C$, for any $\eps>0$
with probability at least $1-\eps$ $\wrt$ the distribution generating
the observations $Z_1,\dots,Z_n$ and the randomized prediction function $\hf$, we have
    \beglab{eq:alqa}
    R(\hat{f}) - R(\flin) \le \kap (H^2 +C^2) \frac{d \log( 3q_{\min}^{-1} ) + 
        \log ( (\log n) \eps^{-1} ) }{n},
    \endlab
for some $\kap$ not depending on $d$ and $n$.
\end{thm}

Using the result of \cite[Section 5.11]{Cat01},
one can prove that Alquier's result still holds for $\hf= \hfols$,
but with $\kap$ also depending on the determinant of the product matrix $Q$.
The $\log[ \log(n) ]$ factor is unimportant and could be removed in
the special case quoted here (it comes from a union bound on a grid
of possible temperature parameters, whereas the temperature could
be set here to a fixed value). The result
differs from Theorem \ref{th:oc} essentially by the fact that
the ratio of the determinants of the empirical and expected product
matrices has been replaced by the inverse of the smallest eigenvalue
of the quadratic form $\theta \mapsto R(\sum_{j=1}^d \theta_j\vp_j) - R(\flin)$.  
In the case when the expected Gram matrix is known,
(e.g., in the case of a fixed design, and also in the slightly
different context of transductive inference), this
smallest eigenvalue can be set to one by choosing the 
quadratic form $\theta \mapsto R(f_{\theta}) - R(\flin)$
to define the Euclidean metric on the parameter space.

Localized Rademacher complexities \cite{Kol06,BarBouMen05} allow
to prove the following property of the empirical risk minimizer.

\begin{thm} \label{th:bbm}
Assume that the input representation $\vp(X)$, the set of parameters and the output $Y$ are
almost surely bounded, i.e., for some positive constants $H$ and $C$,
  \[
  \sup_{\th\in \Theta} \|\th\| \le 1
  \]
  \[
  \ess\sup \|\vp(X)\| \le H,
  \]
and 
  \[
  |Y| \le C \quad \textnormal{a.s.}.
  \]
Let $\nu_1\ge\dots\ge\nu_d$ be the eigenvalues of the  Gram matrix $Q=\B{E} \bigl[ \vp(X) \vp(X)^T\bigr]$.
The empirical risk minimizer satisfies for any $\eps>0$,
with probability at least $1-\eps$:
	\begin{align*}
	R( \hferm ) - R(f^*) & \le \kap (H+C)^2 
	  \frac{\und{\min}{0\le h\le d} \Big( h+\sqrt{ \frac{n}{(H+C)^2} 
	  \sum_{i>h}\nu_i}\Big)+\logeps}{n}\\
  & \le \kap (H+C)^2 \frac{\textnormal{rank}(Q) +\logeps}{n},
	\end{align*}
where $\kap$ is a numerical constant.
\end{thm}
\begin{proof}
The result is a modified version of Theorem 6.7 in \cite{BarBouMen05} 
applied to the linear kernel $k(u,v)=\langle u , v \rangle/(H+C)^2$.
Its proof follows the same lines as in Theorem 6.7 {\em mutatis mutandi}:
Corollary 5.3 and Lemma 6.5 should be used as intermediate steps 
instead of Theorem 5.4 and Lemma 6.6, 
the nonzero eigenvalues of the integral operator induced 
by the kernel being the nonzero eigenvalues of $Q$.
\end{proof}

When we know that the target function $\flin$ is inside some $L^\infty$ ball, it is natural to consider the empirical risk minimizer on this ball. This allows
to compare Theorem \ref{th:bbm} to excess risk bounds with respect to
$\flin$.

Finally, from the work of Birg\'e and Massart \cite{BirMas98}, we may derive the following risk bound for the empirical risk minimizer on a $L^\infty$ ball (see Appendix \ref{sec:bm}).

\begin{thm} \label{th:bmnew}
Assume that $\cF$ has a diameter $H$ for $L^\infty$-norm, i.e., for any 
$f_1,f_2$ in $\cF$, $\sup_{x\in\X} |f_1(x)-f_2(x)| \le H$ and there exists a function
$f_0 \in\cF$ satisfying the exponential moment condition:
	\beglab{eq:expmom}
	\text{for any } x\in\X, \quad 
  \E \Bigl\{ \exp \Bigl[ A^{-1} \bigl\lvert Y-f_0(X) \bigr\rvert \Bigr]\, \Big| \,X=x \Bigr\} \le M,
	\endlab
for some positive constants $A$ and $M$. Let
	\[
	\cR = \und{\inf}{\phi_1,\dots,\phi_d} \, \und{\sup}{\th\in \R^d - \{0\}} 
		\frac{\|\sum_{j=1}^d \th_j \phi_j\|_\infty^2}{\|\th\|_\infty^2}
	\]
where the infimum is taken with respect to all possible orthonormal basis of $\cF$ for 
the dot product $\lan f_1,f_2 \ran = \E f_1(X)f_2(X)$ (when the set $\cF$ admits no basis with exactly $d$ functions, we set $\cR=+\infty$).
Then the empirical risk minimizer satisfies for any $\eps>0$,
with probability at least $1-\eps$:
	\[
	R( \hferm ) - R(f^*) \le \kap (A^2 + H^2) \frac{d \log[2+(\cR /n)\wedge (n/d)] +\logeps}{n},
	\]
where $\kap$ is a positive constant depending only on $M$.
\end{thm}

This result comes closer to what we are looking for: it gives exponential deviation inequalities of order at worse 
$\fracc{d \log(n/d)}{n}$.
It shows that, even if the Gram matrix $Q$
has a very small eigenvalue, there is an algorithm satisfying a convergence rate of order 
$\fracc{d \log(n/d)}{n}$. With this respect, this result is stronger than Theorem \ref{th:alqa}. However there
are cases in which the smallest eigenvalue of $Q$ is of order $1$, while $\cR$ is large (i.e., $\cR \gg n$). In these cases, 
Theorem \ref{th:alqa} does not contain the logarithmic factor which appears in Theorem \ref{th:bmnew}.

\subsection{Projection estimator} \label{sec:proj}

When the input distribution is known, an alternative to the ordinary least squares estimator is 
the following projection estimator. One first finds an orthonormal 
basis of $\Flin$ for the dot product $\lan f_1,f_2 \ran = \E f_1(X)f_2(X)$, 
and then uses the projection estimator on this basis. Specifically, if $\phi_1,\dots,\phi_d$ form an orthonormal basis of $\Flin$, then the projection estimator on this basis is:
	$$
	\hfproj= \sum_{j=1}^d \thproj_j \phi_j,
	$$
with 
	$$
	\thproj= \frac{1}{n} \sum_{i=1}^n Y_i \phi_j(X_i).
	$$
Theorem 4 in \cite{Tsy03} gives a simple bound of order $d/n$ on the expected excess risk $\E R(\hfproj) - R(\flin)$.

\subsection{Penalized least squares estimator}

It is well established that parameters of the ordinary least squares estimator are numerically unstable, and that the phenomenon can be corrected
by adding an $L^2$ penalty (\cite{Lev44,Ril55}). This solution has been labeled ridge regression in statistics (\cite{Hoe62}), and consists 
in replacing $\hfols$ by $\hfrlam=f_{\thrlam}$ with
	\[
	\thrlam \in \und{\argmin}{\th\in\R^d} 
	  \bigg\{ r(f_\th) + \lam \sum_{j=1}^d \theta_j^2 \bigg\},
	\]
where $\lam$ is a positive parameter. 
The typical value of $\lam$ should be small to avoid excessive shrinkage of the coefficients, but not too small in
order to make the optimization task numerically more stable.

Risk bounds for this estimator can be derived from general results concerning penalized least squares on reproducing kernel Hilbert spaces
(\cite{CapVit07}), but as it is shown in Appendix \ref{sec:capvit}, this ends up with complicated results having
the desired $d/n$ rate only under strong assumptions.

Another popular regularizer is the $L^1$ norm. This procedure is known 
as Lasso \cite{Tib94} and is defined by
	\[
	\thllam \in \und{\argmin}{\th\in\R^d} 
		\bigg\{ r(f_\th) + \lam \sum_{j=1}^d |\theta_j| \bigg\}.
	\]
As the $L^2$ penalty, the $L^1$ penalty shrinks the coefficients. The difference is that for coefficients 
which tend to be close to zero, the shrinkage makes them equal to zero. This allows to select relevant 
variables (i.e., find the $j$'s such that $\th^*_j\neq 0$). 
If we assume that the regression function $\freg$ is a linear combination of only $d^* \ll d$ variables/functions $\vp_j$'s,
the typical result is to prove that the risk of the Lasso estimator for $\lam$ of order $\sqrt{\fracl{\log d}{n}}$
is of order $\fracl{d^* \log d}{n}$. Since this quantity is much smaller than $d/n$, this makes a huge improvement (provided that
the sparsity assumption is true). This kind of results usually requires strong conditions on the
eigenvalues of submatrices of $Q$, essentially assuming that the functions 
$\vp_j$ are near orthogonal. We do not know to which extent these 
conditions are required. However, if we do not consider the specific algorithm of Lasso, but the model selection approach developed in \cite{Alq08}, one can change these conditions into a single condition concerning only the minimal eigenvalue of the submatrix of $Q$ corresponding to relevant variables. In fact, we will see that even this condition can be removed.

\subsection{Conclusion of the survey} 

Previous results clearly leave room to improvements. 
The projection estimator requires the unrealistic assumption that the input distribution is known,
and the result holds only in expectation. Results using $L^1$ or $L^2$ regularizations require strong assumptions, in particular on the eigenvalues of (submatrices of) $Q$.
Theorem \ref{th:weakols} provides a $(d \log n)/n$ convergence rate only when the 
$R(\flin)-R(\freg)$ is at most of order $(d \log n)/n$. 
Theorem \ref{th:oc} gives a different type of guarantee: the $d/n$ is 
indeed achieved, but the random ratio of determinants appearing 
in the bound may raise some eyebrows and forbid an explicit
computation of the bound and comparison with other bounds.
Theorem \ref{th:alqa}
seems to indicate that the rate of convergence will
be degraded when the Gram matrix $Q$ is unknown and ill-conditioned. 
Theorem \ref{th:bbm} does not put any assumption on $Q$ to reach the $d/n$ rate, 
but requires particular boundedness constraints on the parameter set, the input vector $\vp(X)$ and the output.
Finally, Theorem \ref{th:bmnew} comes closer to what we are looking for. Yet there is still an unwanted logarithmic factor, and 
the result holds only when the output has uniformly bounded 
conditional exponential moments, which as we will show is not necessary.

\section{Ridge regression and empirical risk minimization} \label{sec:ridge}

We recall the definition
	$$
	\cF= \big\{ f_{\theta} = \sum_{j=1}^d \th_j \vp_j ; (\th_1,\dots,\th_d) \in \cC \big\},
	$$
where $\cC$ is a closed convex set, not necessarily bounded
(so that $\Theta = \B{R}^d$ is allowed).
In this section, we provide exponential deviation inequalities for the empirical risk 
minimizer and the ridge regression estimator on $\cF$ under weak conditions
on the tail of the output distribution. 

The most general theorem which can be obtained from
the route followed in this section is Theorem 
\thmref{thm6.5} stated along with the proof.
It is expressed in terms of a series of empirical 
bounds. The first deduction we can make from this
technical result is of asymptotic nature.
It is stated under weak hypotheses, taking advantage
of the weak law of large numbers.

\begin{thm} \label{th:hfrlam}
For $\lam\ge 0$, let $\frid$ be its associated optimal ridge function (see \eqref{eq:frid}).
Let us assume that
\begin{gather}
\B{E}\bigl[ \lVert \vp(X) \rVert^4 \bigr] < + \infty,\\
\text{and }\quad \B{E} \Bigl\{ \lVert \vp(X) \rVert^2 \bigl[ \frid(X) - Y \bigr]^2 \Bigr\} < + \infty.
\end{gather}
Let $\nu_1,\dots,\nu_d$ be the eigenvalues of 
the Gram matrix $Q=\B{E} \bigl[ \vp(X) \vp(X)^T\bigr]$, and let $Q_{\lambda} = Q + \lambda I$ be the ridge regularization of $Q$. 
Let us define the {\em effective ridge dimension}
$$
D = \sum_{i=1}^d \frac{\nu_i}{\nu_i+\lam} \B{1}(\nu_i>0) 
= \Tr \bigl[ (Q+ \lambda I)^{-1} Q \bigr] 
= \B{E} \bigl[ \lVert Q_{\lambda}^{-1/2} \vp(X) \rVert^2 \bigr].
$$
When $\lambda = 0$, $D$ is equal to the rank 
of $Q$ and is otherwise smaller.
For any $\eps > 0$, there is $n_{\eps}$, 
such that for any $n \geq n_{\eps}$, 
with probability at least $1 - \eps$, 
\begin{align*}
R(\hfrlam) & + \lambda \lVert \hat{\th}^{\textnormal{(ridge)}} 
\rVert^2 \\
& \leq  \min_{\theta \in \Theta} \big\{ R(f_{\theta}) 
+ \lambda \lVert \theta \rVert^2 \big\} \\ 
& \qquad + \frac{30 \, \B{E} \bigl\{ \lVert Q_\lam^{-1/2} \vp(X) \rVert^2 
\bigl[ \frid(X) - Y \bigr]^2\bigr\}}{ 
\B{E} \bigl\{ \lVert Q_\lam^{-1/2} \vp(X) \rVert^2 \bigr\}} \ \frac{D }{n} \\ 
 & \qquad + 1000 \sup_{
v \in \B{R}^d}
\frac{\B{E} \Bigl[ \langle v, \vp(X) \rangle^2 \bigl[ \frid(X) - Y \bigr]^2 
\Bigr]}{\B{E} ( \langle v, \vp(X) \rangle^2 ) + \lambda \lVert v \rVert^2}
\frac{\log(3\eps^{-1})}{n}\\
& \leq  \min_{\theta \in \Theta} \big\{ R(f_{\theta}) 
+ \lambda \lVert \theta \rVert^2 \big\} \\
& \qquad
+ \ess \sup \E\big\{[Y-\frid(X)]^2 \big| X\big\} \, \frac{30 D + 1000 \log(3\eps^{-1})}{n} 
\end{align*}
\end{thm}
\begin{proof}
See Section \thmref{sec:permom}.
\end{proof}

This theorem shows that the ordinary least squares estimator
(obtained when $\Theta = \B{R}^d$ and $\lambda = 0$), 
as well as the empirical risk minimizer on any closed convex 
set, asymptotically reaches a $d/n$ speed of convergence
under very weak hypotheses. It shows also the regularization 
effect of the ridge regression. There emerges an {\em effective
dimension} $D$, 
where the ridge penalty has a threshold effect on the eigenvalues
of the Gram matrix. 

On the other hand, the weakness of this result is 
its asymptotic nature : $n_\eps$ may be arbitrarily large
under such weak hypotheses, and this shows even in the 
simplest case of the estimation of the mean of a real valued
random variable by its empirical mean (which is the case when 
$d = 1$ and $\vp(X) \equiv 1$).

Let us now give some non asymptotic rate under stronger
hypotheses and for the empirical risk minimizer (i.e., 
$\lambda = 0$).

\begin{thm} \label{th:ermom}
Let $d'=\textnormal{rank}(Q)$.
Assume that \[\E\bigl\{[Y-f^*(X)]^4\bigr\}< +\infty\]
and
\[\bcR= \sup_{f\in\Span \{\vp_1,\dots,\vp_d\}-\{0\}} \fracc{\|f\|_\infty^2}{\E[f(X)^2]}
< +\infty.\] 
Consider the (unique) empirical risk minimizer $\hferm=f_{\therm}: x \mapsto \langle\therm,\vp(x)\rangle$ on $\cF$ for which
$\therm \in \Span\{\vp(X_1),\dots,\vp(X_n)\}$\footnote{When
$\cF=\Flin,$ we have $\therm = \Phi^+ \YY$,
with $\Phi=(\vp_j(X_i))_{ 1\le i \le n, 1\le j \le d}$, $\YY=[Y_j]_{j=1}^n$
and $\Phi^+$ is the Moore-Penrose pseudoinverse of $\Phi$.}.
For any values of  $\epsilon$ and $n$ such that 
$2/n \le \eps \le 1$ and 
$$
n > 1280 B^2 \left[ 3 B d'+ \log(2/\epsilon) + \frac{16 B^2 {d'}^2}{n} \right],
$$
with probability at least $1-\eps$,
\begin{multline*}
  R(\hferm) - R(f^*) \\ \le 1920\, \bcR \sqrt{\E 
[Y-f^*(X)]^4} \Biggl[ \frac{3 B d' + \log(2\eps^{-1})}{n} + \bigg(\frac{4 B d'}n\bigg)^2 \Biggr].
\end{multline*}
\end{thm}

\begin{proof}
See Section \thmref{sec:permom}.
\end{proof}

It is quite surprising that the traditional assumption of uniform boundedness of the conditional exponential moments of the output
can be replaced by a simple moment condition for reasonable confidence 
levels (i.e., $\eps \geq 2/n$). For highest confidence levels,
things are more tricky since we need to control with high probability a term of order $[r(f^*)-R(f^*)]d/n$ (see Theorem \ref{th:ermg}). 
The cost to pay to get the exponential deviations under only a fourth-order moment condition on the output is the appearance of the geometrical quantity $\bcR$ as a multiplicative factor, as opposed to Theorems \ref{th:alqa}
and \ref{th:bmnew}. More precisely, 
from \cite[Inequality (3.2)]{BirMas98}, we have $\bcR \le \cR \le \bcR d$, 
but the quantity $\cR$ appears inside a logarithm in Theorem~\ref{th:bmnew}. 
However, Theorem~\ref{th:bmnew} is restricted to the empirical risk 
minimizer on a $L^\infty$ ball, while the result here is valid for any closed convex set $\cC$, and in particular 
applies to the ordinary least squares estimator.

Theorem~\ref{th:ermom} is still limited in at least three ways: it applies only to 
uniformly bounded $\vp(X)$, the output needs to have a fourth moment, and 
the confidence level should be as great as $\epsilon \geq 2/n$.
These limitations will be addressed in the next sections by
considering more involved algorithms. 

\section{A min-max estimator for robust estimation} 
\label{sec:computable}

\subsection{The min-max estimator and its theoretical guarantee}

This section provides an alternative to the empirical risk minimizer 
with non asymptotic exponential risk deviations of order $d/n$ for 
any confidence level.
Moreover, we will assume only a second order moment condition 
on the output and cover the case of unbounded inputs, the requirement on $\varphi(X)$
being only a finite fourth order moment. On the other hand, we assume 
that the set $\Theta$ of the vectors of coefficients is bounded.
The computability of the proposed estimator and numerical experiments are discussed at the end of the section.

Let $\alpha>0$, $\lam\ge 0$, and consider the truncation function:
$$
\psi(x) = \begin{cases} 
- \log \bigl(1 - x + x^2/2 \bigr) & 0 \leq x \leq 1, \\
\log(2) & x \geq 1, \\ 
- \psi(-x) & x \leq 0, 
\end{cases} 
$$
For any $\th,\th'\in\Theta$, introduce 
  $$
  \cD(\th,\th')=n \alpha \lam (\|\th\|^2-\|\th'\|^2)
    + \sum_{i=1}^n \psi\Big(\alpha\big[Y_i-f_\th(X_i)\big]^2-\alpha\big[Y_i-f_{\th'}(X_i)\big]^2\Big).
  $$
We recall $\frid=f_{\thrid}$ with
$
\thrid \in \arg \min_{\theta \in \Theta} \big\{ R(f_{\theta}) 
+ \lambda \lVert \theta \rVert^2 \big\}
$, and the  effective ridge dimension
$$
D = \sum_{i=1}^d \frac{\nu_i}{\nu_i+\lam} \B{1}(\nu_i>0) 
= \Tr \bigl[ (Q+ \lambda I)^{-1} Q \bigr] 
= \B{E} \bigl[ \lVert Q_{\lambda}^{-1/2} \vp(X) \rVert^2 \bigr].
$$
Let us assume in this section that 
for any $j\in\{1,\dots,d\}$, 
  \beglab{eq:as1}
  \E\big\{\vp_j(X)^2[Y-\frid(X)]^2\big\}<+\infty,
  \endlab
and
  \beglab{eq:as2}
  \E\big[\vp_j^4(X)\big]<+\infty.
  \endlab

Define 
  \begin{align} 
  \cS & = \{ f\in\Flin: \E[f(X)^2]=1\}, \label{eq:cs}\\
\sigma & = \sqrt{\E\big\{[Y-\frid(X)]^2\big\}}= \sqrt{R(\frid)} , \\ 
\chi & = \max_{f\in\cS} \sqrt{\E [f(X)^4] },\\ 
\kappa & = \frac{ \sqrt{\E\big\{[\vp(X)^TQ_\lam^{-1}\vp(X)]^2\big\}}}
  {\E\big[\vp(X)^TQ_\lam^{-1}\vp(X) \big]},\\ 
\kappa' & = \frac{\sqrt{\E\big\{[Y-\frid(X)]^4\big\}}}{\E\big\{[Y-\frid(X)]^2\big\}} 
  = \frac{\sqrt{\E\big\{[Y-\frid(X)]^4\big\}}}{\sigma^2},\\
T & = \max_{\th\in\Theta,\th'\in\Theta} 
  \sqrt{ \lam \|\th-\th'\|^2 + \E[ f_{\th}(X)-f_{\th'}(X) ]^2 }. \label{eq:kapp}
  \end{align}

\begin{thm} \label{th:3.1}
Let us assume that \eqref{eq:as1} and \eqref{eq:as2} hold.
For some numerical constants $c$ and $c'$, 
for
  $$
  n > c \kappa \chi D,
  $$
by taking
  \beglab{eq:alpha}
  \alpha = \frac{1}{2 \chi \bigl[ 2 \sqrt{\kappa'} \sigma
  + \sqrt{\chi} T
  \bigr]^2} \biggl(1  - \frac{c \kappa \chi D}{n} \biggr), 
  \endlab
for any estimator $f_{\hth}$ satisfying $\hth\in\Theta$ a.s., for any $\eps>0$ and any
$\lam\ge 0$,  
with probability at least $1-\eps$, we have
  \begin{align*}  
  R(f_{\hth}) + \lam \lVert \hth \rVert^2 & \le 
  \min_{\theta \in \Theta} \big\{ R(f_{\theta}) + \lambda \lVert \theta \rVert^2 \big\}\\
& \qquad + \frac1{n\alpha}\bigg( \und{\max}{\th_1\in\Theta} \cD(\hth,\th_1)
    - \und{\inf}{\th\in{\Theta}} \und{\max}{\th_1\in{\Theta}} \cD(\th,\th_1)  \bigg) \\
& \qquad + \frac{c \kappa \kappa' D \sigma^2}{n}  + \frac{8 \chi \bigl( \frac{\log(\epsilon^{-1})}{n} + 
\frac{c' \kappa^2 D^2}{n^2} \bigr) \bigl[ 2 \sqrt{\kappa'}\sigma 
+ \sqrt{\chi} T \bigr]^2}{ 1 - \frac{c\kappa \chi D}{n} }.
  \end{align*}  
\end{thm}

\begin{proof}
See Section \thmref{sec:p3.1}.
\end{proof}

By choosing an estimator such that 
  $$\und{\max}{\th_1\in\Theta} \cD(\hth,\th_1) < \und{\inf}{\th\in{\Theta}} \und{\max}{\th_1\in{\Theta}} \cD(\th,\th_1) + \sigma^2 \frac{D}n,$$
Theorem \ref{th:3.1} provides a non asymptotic bound for the excess 
(ridge) risk with a $D/n$ convergence rate and an exponential tail even 
when neither the output $Y$ nor the input vector $\vp(X)$ has exponential moments. This stronger non asymptotic bound compared to the bounds of the previous section comes 
at the price of replacing the empirical risk
minimizer by a more involved estimator. Section \ref{sec:comput} provides
a way of computing it approximately.

\subsection{The value of the uncentered kurtosis coefficient $\chi$}
Let us discuss here the value of constant $\chi$, 
which plays a critical role in the speed of convergence of our bound.
With the convention $\frac00=0$, we have
  $$\chi= \sup_{u \in \B{R}^d} 
\frac{ \B{E} \bigl( \langle u, \varphi(X) \rangle^4 \bigr)^{1/2}}{ 
\B{E} \bigl( \langle u, \varphi(X) \rangle^2 \bigr)}.$$ 

Let us first examine the case when $\varphi_1(X) \equiv 1$ and 
$\bigl[\varphi_j(X), j=2, \dots, d \bigr]$ are independent. To compute
$\chi$, we can assume without loss of generality that they are
centered and of unit variance, which will be the case after $Q^{-1/2}$ 
is applied to them.  
In this situation, introducing 
$$
\chi_* = \max_{j=1, \dots, d} \frac{ \B{E} \bigl[ \varphi_j(X)^4 \bigr]^{1/2}}{\B{E} 
\bigl[ \varphi_j(X)^2 \bigr]},
$$
we see that for any $u\in\R^d$ with $\|u\|=1$, we have
\begin{multline*}
\B{E} \bigl( \langle u, \varphi(X) \rangle^4 \bigr) = 
\sum_{i=1}^d u_i^4 \B{E} (\varphi_i(X)^4) + 6 \sum_{1\leq i<j \leq d} 
u_i^2 u_j^2 \B{E}\bigl[ \varphi_i(X)^2 \bigr] \B{E} \bigl[ 
\varphi_j(X)^2 \bigr] \\+ 4 \sum_{i=2}^d u_1 u_i^3 \B{E} \bigl[ \varphi_i(X)^3 \bigr] 
\\ \leq \chi_*^2 \sum_{i=1}^d u_i^4 
+ 6 \sum_{i < j} u_i^2 u_j^2 + 4 \chi_*^{3/2} \sum_{i=2}^d 
\lvert u_1 u_i \rvert^3 
\\ \leq \sup_{u \in \R_{+}^d, \lVert u \rVert = 1} \bigl( \chi_*^2-3 \bigr) \sum_{i=1}^d u_i^4 + 3 \left( \sum_{i=1}^d
u_i^2 \right)^2 + 4 \chi_*^{3/2} u_1 \sum_{i=2}^d u_i^3 
\\ \leq \frac{3^{3/2}}{4} \chi_*^{3/2} + \begin{cases}
\chi_*^2 , & \chi_*^2 \geq 3,\\
3 + \frac{\chi_*^2 - 3}{d}, & 1 \leq \chi_*^2 < 3.
\end{cases}
\end{multline*}
Thus in this case
$$
\chi \leq \begin{cases}
\chi_* \left( 1 + \frac{3^{3/2}}{4 \sqrt{\chi_*}} \right)^{1/2}, &
\chi_* \geq \sqrt{3},\\ 
\left(3 + \frac{3^{3/2}}{4} \chi_*^{3/2} + \frac{\chi_*^2 - 3}{d} 
\right)^{1/2}, & 1 \leq \chi_* < \sqrt{3}. 
\end{cases}
$$

If moreover the random variables $\varphi_j(X)$ are not skewed, in the sense 
that $\B{E}  \bigl[ \varphi_j(X)^3 \bigr] = 0$, $j = 2, \dots, d$, 
then
$$
\begin{cases}
\chi = \chi_*, & \chi_* \geq \sqrt{3},\\
\chi \leq \left( 3 + \frac{\chi_*^2 - 3}{d} \right)^{1/2}, & 1 \leq \chi_* < 
\sqrt{3}. 
\end{cases}
$$
In particular in the case when $\varphi_j(X)$ are Gaussian variables, 
$\chi = \chi_* = \sqrt{3}$ (as could be seen in a more straightforward
way, since in this case $\langle u, \varphi(X) \rangle$ is also Gaussian !).

In particular, this situation arises in compress sensing using random projections on Gaussian vectors.
Specifically, assume that we want to recover a signal $f\in\R^M$ that we know to be well approximated
by a linear combination of $d$ basis vectors $f_1,\dots,f_d$. We measure $n \ll M$ projections of
the signal $f$ on i.i.d. $M$-dimensional standard normal random vectors $X_1,\dots,X_n$: $Y_i=\langle f,X_i\rangle,$ $i=1,\dots,n$.
Then, recovering the coefficient $\th_1,\dots,\th_d$ such that $f=\sum_{j=1}^d \th_j f_j$
is associated to the least squares regression problem 
  $Y \approx \sum_{j=1}^d \th_j \varphi_j(X),$
with $\varphi_j(x)=\langle f_j,x\rangle$, and $X$ having a $M$-dimensional standard normal distribution.

Let us discuss now a bound which is suited to the case when we are using
a partial basis of regression functions. The functions $\varphi_j$ are usually bounded
(think of the Fourier basis, wavelet bases, histograms, splines ...).

Let us assume that for some positive constant $A$ and any $u \in \B{R}^d$, 
$$
\lVert u \rVert \leq A \B{E} \bigl[ \langle u, \varphi(X) \rangle^2 \bigr]^{1/2}.
$$
This appears as some stability property of the partial basis $\varphi_j$ 
with respect to the $\B{L}_2$-norm, since it can also be written as
$$
\sum_{j=1}^d u_j^2 \leq A^2 \B{E} \Biggl[ \biggl( \sum_{j=1}^d u_j 
\varphi_j(X) \biggr)^2 \Biggr], \qquad u \in \B{R}^d.
$$
This will be the case if $\varphi_j$ is nearly orthogonal in the 
sense that
\begin{align*}
\B{E} \bigl[ \varphi_j(X)^2 \bigr] & \geq 1, \quad  \text{ and } \quad 
\Bigl\lvert \B{E} \bigl[ \varphi_j(X) \varphi_k(X) \bigr]  
\Bigr\rvert \leq \frac{1 - A^2}{d-1}.
\end{align*}
In this situation, by using 
$$
\B{E} \bigl[ \langle u, \varphi(X) \rangle^4 \bigr] 
\leq \lVert u \rVert^2 \ess \sup \lVert \varphi(X) \rVert^2 \B{E} 
\bigl[ \langle u, \varphi(X) \rangle^2 \bigr],
$$
one can check that
$$
\chi \leq A \Biggl\lVert \biggl( \sum_{j=1}^d \varphi_j^2 \biggr)^{1/2} 
\Biggr\rVert_{\infty}.
$$
%
Therefore, if $X$ is the uniform random variable on the unit interval 
and $\varphi_j$, $j=1, \dots, d$ are any functions from the Fourier 
basis (meaning that they are of the form $\sqrt{2} \cos(2 k \pi X)$ or $\sqrt{2} \sin( 2 k \pi X)$), 
then $\chi \leq \sqrt{2 d}$ (because they form an orthogonal
system, so that $A = 1$). 

On the other hand, a localized basis like the evenly spaced histogram 
basis of the unit interval 
$$
\varphi_j(x) = 
\sqrt{d} \B{1}\Bigl(x \in \bigl[(j-1)/d, j/d \bigr[\Bigr),
$$
will also be such that $\chi \leq \sqrt{d}$.
Similar computations could be made for other local bases, like 
wavelet bases.
Note that when $\chi$ is of order $\sqrt{d}$, Theorem \ref{th:3.1}
means that the excess risk of the min-max truncated estimator
$\hf$ is upper bounded by $C\frac{d}n$
provided that $n\ge C d\sqrt{d}$ for a large enough constant $C$.

Let us discuss the case when $X$ is some observed random variable 
whose distribution is only approximately known. 
Namely let us assume that 
$(\varphi_j)_{j=1}^d$ is 
some basis of functions in $\B{L}_2 \bigl[ \wt{\B{P}} \bigr]$
with some known coefficient $\wt{\chi}$, where $\wt{\B{P}}$
is an approximation of the true distribution of $X$ in the sense that the density 
of the true distribution $\P$ of $X$ with respect to the 
distribution $\wt{\B{P}}$ is in the range $(\eta^{-1/2}, \eta)$. 
In this situation, the coefficient $\chi$ satisfies the inequality $\chi \leq \eta \wt{\chi}$.
Indeed 
\begin{align*}
\B{E}_{X\sim \P} \bigl[ \langle u, \varphi(X) \rangle^4 \bigr] 
& \leq \eta {\B{E}}_{X\sim \wt{\P}} \bigl[ \langle u, \varphi(X) \rangle^4 \bigr] \\
& \leq \eta \wt{\chi}^2 {\B{E}}_{X\sim \wt{\P}} \bigl[ \langle u, \varphi(X) \rangle^2 \bigr]^2 
\leq \eta^2 \wt{\chi}^2 \B{E}_{X\sim {\P}} \bigl[ \langle u, \varphi(X) \rangle^2 \bigr]^2.
\end{align*}

Let us conclude this section with some scenario 
for the case when $X$ is a real-valued random variable. Let us consider the 
distribution function of $\wt{\B{P}}$ 
$$
\wt{F}(x) = \wt{\B{P}} ( X \leq x).
$$
Then, if $\wt{\B{P}}$ has no atoms, 
the distribution of $\wt{F}(X)$ is uniform in $(0,1)$. Starting from some 
suitable partial basis $(\varphi_j)_{j=1}^d$ of $\B{L}_2 \bigl[ (0,1), \B{U} \bigr]$ 
where $\B{U}$ is the uniform distribution, like the 
ones discussed above, we can build a basis for our problem as 
$$
\wt{\varphi}_j(X) = \varphi_j \bigl[ \wt{F}(X) \bigr].
$$
Moreover, if $\B{P}$ is absolutely continuous with respect to $\wt{\B{P}}$ with 
density $g$, then $\B{P} \circ \wt{F}^{-1}$ is absolutely continuous with 
respect to $\wt{\B{P}} \circ \wt{F}^{-1}$, with density $g \circ \wt{F}^{-1}$, 
and of course, the fact that $g$ takes values in $(\eta^{-1/2}, \eta)$ implies
the same property for $g \circ \wt{F}^{-1}$. Thus, if $\wt{\chi}$ is 
the coefficient corresponding to $\varphi_j(U)$ when $U$ 
is the uniform random variable on the unit interval, then the true 
coefficient $\chi$ (corresponding to 
$\wt{\varphi}_j(X)$) will be such that $\chi \leq \eta \wt{\chi}$. 

\subsection{Computation of the estimator} \label{sec:comput}

For ease of description of the algorithm, we will write $X$ for $\vp(X)$, which is equivalent to considering without loss of generality that the input space is $\R^d$ and that the functions $\vp_1,\dots,$$\vp_d$ are the coordinate functions.
Therefore, the function $f_\th$ maps an input $x$ to $\langle \theta, x \rangle$.

Let us introduce 
  $$
  \bL_i(\th) = \alpha \bigl( \langle \theta, X_i \rangle - Y_i \bigr)^2.
  $$
For any subset of indices $I \subset \{1, \dots, n\}$, let us define
$$
r_I(\theta) = \lam \|\th\|^2 + \frac{1}{\alpha \lvert I \rvert} \sum_{i \in I} \bL_i(\th).
$$

We suggest the following heuristics to compute an approximation of 
$$
\arg \min_{\theta \in \Theta} 
\sup_{\th' \in \Theta} \cD(\theta, \th').
$$
\begin{itemize}
\item Start from $I_1 = \{1, \dots, n\}$
with the empirical risk minimizer
$$\wh{\theta}_1 = \arg \min_{\R^d} r_{I_1}=\therm.$$ 
\item At step number $k$, compute
$$
\wh{Q}_k = \frac{1}{\lvert I_k \rvert} \sum_{i \in I_k} X_i X_i^T. 
$$
\item Consider the sets
$$
J_{k,1}(\eta) = \Biggl\{ i \in I_k : \bL_i(\wh{\theta}_k) 
X_i^T \wh{Q}_{k}^{-1} X_i \bigg( 
1 + \sqrt{ 1 + \big[\bL_i(\wh{\theta}_k)\big]^{-1}} \; \bigg)^2
< \eta \Biggr\},$$
where $\wh{Q}_{k}^{-1}$ is the (pseudo-)inverse of the matrix $\wh{Q}_{k}$.
\item Let us define 
\begin{align*}
\theta_{k,1}(\eta) & = \arg \min_{\R^d} r_{J_{k,1}(\eta)}, \\
J_{k,2}(\eta) & = \Bigl\{ i \in I_k : \bigl\lvert \bL_i\big(\theta_{k,1}(\eta)\big)
- \bL_i\big( \wh{\theta}_k \big) \bigr\rvert \le  1
\Bigr\}, \\   
\theta_{k,2} (\eta) & = \arg \min_{\R^d} r_{J_{k,2}(\eta)}, \\
(\eta_{k}, \ell_k) & = \arg \min_{\eta \in \B{R}_+, \ell \in \{1, 2\}}  \max_{j=1, \dots, k} 
\cD\bigl( \theta_{k,\ell}(\eta), \wh{\theta}_j \bigr), \\
I_{k+1} & = J_{k, \ell_k}(\eta_k), \\
\wh{\theta}_{k+1} & = \theta_{k, \ell_k}(\eta_k).
\end{align*}
\item Stop when 
$$
\max_{j=1, \dots, k} \cD(\wh{\theta}_{k+1}, \wh{\theta}_j) \geq 0,
$$
and set $\wh{\theta}=\wh{\theta}_k$ as the final estimator of $\thrid$.

\end{itemize}
Note that there will be at most $n$ steps, since $I_{k+1} \varsubsetneq I_k$ and in practice much less in this
iterative scheme.
Let us give some justification for this proposal. Let us notice first that
\begin{multline*}
\cD(\theta + h, \theta) = n \alpha \lam (\|\th+ h\|^2-\|\th\|^2)\\
+ \sum_{i=1}^n \psi \Bigl( 
\alpha \bigl[ 2 \langle h, X_i \rangle \big(\langle \theta, X_i \rangle - Y_i\big) + 
\langle h, X_i \rangle^2 \bigr] \Bigr). 
\end{multline*}
Hopefully, $\thrid= \arg \min_{\th\in\R^d} \big(R(f_\th)+\lam \|\th\|^2\big)$ 
is in some small neighbourhood 
of $\wh{\theta}_k$ already, according to the distance defined by $
Q \simeq \wh{Q}_k$. So we may try to look for improvements of $\wh{\theta}_k$ 
by exploring neighbourhoods of $\wh{\theta}_k$ of increasing 
sizes with respect to some approximation of the relevant 
norm $\lVert \theta \rVert_Q^2 = \B{E} \bigl[ \langle \theta, X \rangle^2 \bigr]$.

Since the truncation function $\psi$ is constant on $(-\infty,-1]$ and $[1,+\infty)$,
the map $\theta \mapsto \cD(\theta, \wh{\theta}_k)$ induces a decomposition of 
the parameter space into cells corresponding to different sets $I$ of examples. Indeed, such a set $I$ is associated
to the set ${\cal C}_I$ of $\th$ such that $\bL_i(\th)-\bL_i(\wh{\theta}_k)< 1$ if and only if $i\in I$.
Although this may not be the case, we will do as if 
the map $\theta \mapsto \cD(\theta, \wh{\theta}_k)$ restricted to the cell ${\cal C}_I$ reached its minimum at some interior point 
of ${\cal C}_I$, and approximates this minimizer by the minimizer of $r_I$.

The idea is to remove first the examples which will become 
inactive in the closest cells to the current estimate $\wh{\theta}_k$.
The cells for which the contribution of example number $i$ 
is constant are delimited by at most four parallel hyperplanes.

It is easy to see that the square of the inverse of the distance of 
$\wh{\theta}_k$ to the closest of these hyperplanes is 
equal to 
$$
\frac1\alpha X_i^T \wh{Q}_{k}^{-1} X_i \bL_i(\wh{\theta}_k) \Biggl( 
1 + \sqrt{ 1 + \frac{1}{\bL_i(\wh{\theta}_k)}} \; \Biggr)^2.
$$ 
Indeed, this distance is the infimum of $\lVert \wh{Q}_k^{1/2} h \rVert$, where $h$ is a solution of 
$$
\langle h, X_i \rangle^2 + 2 \langle h, X_i \rangle 
\big(\langle \wh{\theta}_k, X_i \rangle - Y_i\big) = \frac{1}{\alpha}.
$$
It is computed by considering $h$ of the form $h = \xi \lVert \wh{Q}_k^{-1/2} X_i \rVert^{-1} 
\wh{Q}_k^{-1} X_i$ and solving an equation of order two in $\xi$. 

This explains the proposed choice of $J_{k,1}(\eta)$. Then a first estimate $\theta_{k,1}(\eta)$
is computed on the basis of this reduced sample, and the sample is readjusted to $J_{k,2}(\eta)$
by checking which constraints are really activated in the computation of 
$\cD(\theta_{k,1}(\eta), \wh{\theta}_k)$. The estimated parameter is then readjusted taking 
into account the readjusted sample (this could as a variant be iterated more than once). 
Now that we have some new candidates $\theta_{k,\ell}(\eta)$, we check the minimax 
property against them to elect $I_{k+1}$ and $\wh{\theta}_{k+1}$. Since we did not check 
the minimax property against the whole parameter set $\Theta=\R^d$, we have no theoretical 
warranty for this simplified algorithm. Nonetheless, similar computations to what we 
did could prove that we are close to solving $\min_{j=1, \dots, k} R(f_{\wh{\theta}_j})$, 
since we checked the minimax property on the reduced parameter set $\{ \wh{\theta}_j, 
j=1, \dots, k \}$. Thus the proposed heuristics is capable of improving on the performance
of the ordinary least squares estimator, while being guaranteed not to degrade its
performance significantly. 

\subsection{Synthetic experiments}

In Section \ref{sec:noise}, we detail the three kinds of noises we work with.
Then, Sections \ref{sec:expind}, \ref{sec:hcc} and \ref{sec:ts} describe 
the three types of functional relationships between the input, the output and the noise involved
in our experiments. A motivation for choosing these input-output distributions was the ability to compute
exactly the excess risk, and thus to compare easily estimators.
Section \ref{sec:exper} provides details about the implementation, its computational efficiency 
and the main conclusions of the numerical experiments. Figures and tables are postponed to
Appendix \ref{app:exp}.

\subsubsection{Noise distributions} \label{sec:noise}

In our experiments, we consider three types of noise that are centered and with unit variance:
\begin{itemize}
\item the standard Gaussian noise: $W \sim \N(0,1)$,
\item a heavy-tailed noise defined by: 
  $W=\sign(V)/|V|^{1/q}$, 
with $V \sim \N(0,1)$ a standard Gaussian random variable and $q=2.01$ (the 
real number $q$ is taken strictly larger than $2$ as for $q=2$, the random variable $W$ would not admit
a finite second moment).  
\item a mixture of a Dirac random variable with a low-variance Gaussian random variable defined by: 
with probability $p$, $W=\sqrt{\fracl{1-\rho}p}$, and with probability $1-p$, $W$ is drawn from 
  $$\N\bigg(-\frac{\sqrt{p(1-\rho)}}{1-p},\frac{\rho}{1-p}-\frac{p(1-\rho)}{(1-p)^2}\bigg).$$ 
The parameter $\rho\in[p,1]$ characterizes the 
part of the variance of $W$ explained by the Gaussian part of the mixture.
Note that this noise admits exponential moments, but for $n$ of order $1/p$, the
Dirac part of the mixture generates low signal to noise points.
\end{itemize}

\subsubsection{Independent normalized covariates (INC$(n,d)$)} \label{sec:expind}
In INC$(n,d)$, the input-output pair is such that
  $$
  Y=\langle \th^* , X \rangle + \sigma W,
  $$
where the components of $X$ are independent standard normal distributions,
$\th^*=(10,\dots,10)^T \in\R^d$, and $\sigma=10$.

\subsubsection{Highly correlated covariates (HCC$(n,d)$)} \label{sec:hcc}
In HCC$(n,d)$, the input-output pair is such that
  $$
  Y=\langle \th^* , X \rangle + \sigma W,
  $$
where $X$ is a multivariate centered normal Gaussian with
covariance matrix $Q$ obtained by drawing a $(d,d)$-matrix $A$ of uniform random variables 
in $[0,1]$ and by computing $Q=A A^T$,
$\th^*=(10,\dots,10)^T \in\R^d$, and $\sigma=10$.
So the only difference with the setting of Section \ref{sec:expind} is 
the correlation between the covariates.

\subsubsection{Trigonometric series (TS$(n,d)$)} \label{sec:ts}
Let $X$ be a uniform random variable on $[0,1]$.
Let $d$ be an even number.
Let 
  $$\varphi(X)= \big(\cos(2 \pi X),\dots,\cos(d \pi X),\sin(2 \pi X),\dots,\sin(d \pi X)\big)^T.$$
In TS$(n,d)$, the input-output pair is such that
  $$
  Y=20 X^2-10 X-\frac53+\sigma W,
  $$
with $\sigma=10$.
One can check that this implies
  $$\theta^*=\bigg(\frac{20}{\pi^2},\dots,\frac{20}{\pi^2 (\frac{d}2)^2},
   -\frac{10}{\pi},\dots,-\frac{10}{\pi (\frac{d}2)}\bigg)^T \in\R^d.$$ 
  
\subsubsection{Experiments} \label{sec:exper}

\paragraph{Choice of the parameters and implementation details.}
Our min-max truncated algorithm has two parameters $\alpha$ and $\lambda$.
In the subsequent experiments, we set the ridge parameter $\lam$ to the natural default choice for it: 
$\lam=0$. For the truncation parameter $\alpha$, according to our analysis (see \eqref{eq:alpha}), 
it roughly should be of order $1/\sigma^2$ up to kurtosis coefficients.
By using the ordinary least squares estimator, we roughly estimate this value, 
and test values of $\alpha$ in a geometric grid (of $8$ points) around it (with ratio $3$).
Cross-validation can be used to select the final $\alpha$.
Nevertheless, it is computationally expensive and is significantly outperformed in
our experiments by the following simple procedure:
start with the smallest $\alpha$ in the geometric grid and increase it as long as 
$\hat{\th}=\th_1$, that is as long as we stop at the end of the first iteration
and output the empirical risk minimizer.

To compute $\theta_{k,1}(\eta)$ or $\theta_{k,2}(\eta)$,
one needs to determine a least squares estimate (for a modified sample).
To reduce the computational burden, we do not want to test all possible values of
$\eta$ (note that there are at most $n$ values leading to different estimates).
Our experiments show that testing only three levels of $\eta$ is sufficient.
Precisely, we sort the quantity
$$\bL_i(\wh{\theta}_k) 
X_i^T \wh{Q}_{k}^{-1} X_i \bigg( 
1 + \sqrt{ 1 + \big[\bL_i(\wh{\theta}_k)\big]^{-1}} \; \bigg)^2$$
by decreasing order 
and consider $\eta$ being the first, $5$-th and $25$-th value of the ordered list.
Overall, in our experiments, the computational complexity is approximately fifty times
larger than the one of computing the ordinary least squares estimator.

\paragraph{Results.}
The tables and figures have been gathered in Appendix \ref{app:exp}.
Tables \ref{tab:b01} and \ref{tab:b04} give the results for the mixture noise.
Tables \ref{tab:a201}, \ref{tab:a-201} and \ref{tab:a0} provide the results for the heavy-tailed noise
and the standard Gaussian noise. Each line of the tables has been obtained after $1000$ generations of the training set.
These results show that the min-max truncated estimator is often equal to $\hferm$,
while it ensures impressive consistent improvements when it differs from $\hferm$.
In this latter case, the number of points that are not considered in $\hf$, i.e. the number of points
with low signal to noise ratio, varies a lot from $1$ to $150$ and is often of order $30$. Note that not only the points
that we expect to be considered as outliers (i.e. very large output points) are erased, and
that these points seem to be taken out by local groups: see Figures $1$ and $2$ 
in which the erased points are marked by surrounding circles.

Besides, the heavier the noise tail is (and also the larger the variance of the noise is), the more often the truncation modifies the initial ordinary least squares estimator, and the more improvements we get from the min-max truncated estimator, which also becomes much more robust than the ordinary least squares estimator (see the confidence intervals in the tables).



\section{A simple tight risk bound for a sophisticated PAC-Bayes algorithm} \label{sec:main}

A disadvantage of the min-max estimator proposed in the previous section
is that its theoretical guarantee depends on kurtosis like coefficients.
In this section, we provide a more sophisticated estimator, 
having a simple theoretical excess risk bound, 
which is independent of these kurtosis like quantities
when we assume $L_\infty$-boundedness of the set $\cF$.

We consider that the set $\cC$ is bounded so that
we can define the ``prior'' distribution $\pi$ as the uniform distribution on $\cF$ 
(i.e., the one induced by the Lebesgue distribution on $\cC\subset \R^d$ 
renormalized to get $\pi(\cF)=1$).  
Let $\lam>0$ and 
	\[W_i(f,f') = \lam \bigl\{ \bigl[ Y_i-f(X_i)\bigr]^2 - \bigl[
Y_i-f'(X_i) \bigr]^2 \bigr\}. \]
Introduce 
	\beglab{eq:hcEls}
	\hcE(f) = \log \int \frac{\pi(df')}{\prod_{i=1}^n [ 1 - W_i(f,f') + \frac{1}{2} W_i(f,f')^2]}.
	\endlab
We consider the ``posterior'' distribution $\hpi$ on the set $\cF$ with density:
	\beglab{eq:hpils}
	\frac{d \hpi}{d \pi} (f) = \frac{\ds \exp [-\hcE(f)]}{\ds \tint
\exp [-\hcE(f') ] \pi(df')}.
	\endlab
To understand intuitively why this distribution concentrates on functions with low risk,
one should think that when 
$\lam$ is small enough, $1-W_i(f,f') + \frac{1}{2} W_i(f,f')^2$
is close to $e^{-W_i(f,f')}$, and consequently
  \[
  \hcE(f) \approx \lam \sum_{i=1}^n [Y_i-f(X_i)]^2 + \log \int \pi(df') \exp 
\Bigl\{-\lam \sum_{i=1}^n \bigl[Y_i-f'(X_i)\bigr]^2 \Bigr\},
  \]
and
  \[
  \frac{ d \hpi}{d \pi} (f) \approx 
\frac{\exp \{ -\lam \sum_{i=1}^n [Y_i-f(X_i)]^2 \} }{\int \exp 
\{ -\lam \sum_{i=1}^n [Y_i-f'(X_i)]^2 \} \pi(df')}\,.
  \]
The following theorem gives a $d/n$ convergence rate for the randomized algorithm which draws the prediction function from $\cF$
according to the distribution $\hpi$.

\begin{thm} \label{th:main}
Assume that $\cF$ has a diameter $H$ for $L^\infty$-norm:
	\beglab{eq:hh}
	\sup_{f_1,f_2\in \cF,x\in\X} |f_1(x)-f_2(x)| = H
	\endlab
and that, for some $\sigma>0$, 
	\beglab{eq:sigma}
	\sup_{x\in\X} \E\big\{[Y-f^*(X)]^2 \big| X=x\big\} \le \sigma^2 < +\infty.
	\endlab
Let $\hf$ be a prediction function drawn from the distribution $\hpi$
defined in \myeq{eq:hpils} and depending on the parameter $\lam>0$.
Then for any $0<\eta'<1-\lam (2\sigma+H)^2$ and $\eps>0$, with probability (with respect to the distribution $P^{\otimes n} \hpi$ generating
the observations $Z_1,\dots,Z_n$ and the randomized prediction function $\hf$) at least $1-\eps$, we have
	\[
	R( \hf ) - R(f^*) \le (2\sigma+H)^2 \, \frac{C_1 d + C_2 \log(2\eps^{-1})}{n} 
	\]
with 
	\[
	C_1= \frac{\log(\frac{(1+\eta)^2}{\eta'(1-\eta)}) }{\eta (1-\eta-\eta')} \quad \text{and} \quad
	C_2=\frac{2}{\eta(1-\eta- \eta')} \quad \text{and} \quad \eta= \lam (2\sigma+H)^2.
	\]
In particular for $\lam=0.32 (2\sigma+H)^{-2}$ and $\eta'=0.18$, we get
	\[
	R( \hf ) - R(f^*) \le (2\sigma+H)^2 \, \frac{16.6 \, d 
    		+ 12.5 \log(2\eps^{-1})}{n}.	
	\]
Besides if $f^*\in\argmin_{f\in\Flin} R(f)$, then with
probability at least $1-\eps$, we have
	\[
	R( \hf ) - R(f^*) \le (2\sigma+H)^2 \, \frac{8.3 \, d 
    		+ 12.5 \log(2\eps^{-1})}{n}.	
	\]
\end{thm}

\begin{proof}
This is a direct consequence of Theorem \thmref{th:v2c},
Lemma \thmref{le:complexity} and Lemma \thmref{le:v2}.
\end{proof}

If we know that $\flin$ belongs to some bounded ball in 
$\Flin$, then one can define a bounded $\cF$ as this ball, use the previous 
theorem and obtain an excess risk bound with respect to
$\flin$.


\begin{rmk}
Let us discuss this result. On the positive side, we have a $d/n$ convergence rate in expectation and in deviations.
It has no extra logarithmic factor. It does not require any particular assumption on the smallest eigenvalue of the covariance matrix.
To achieve exponential deviations, a uniformly bounded second moment of the output knowing the input is surprisingly sufficient: we do not
require the traditional exponential moment condition on the output.
Appendix \thmref{sec:lb} argues that the uniformly bounded conditional second moment assumption 
cannot be replaced with just a bounded second moment condition.

On the negative side, the estimator is rather complicated. When the target
is to predict as well as the best linear combination $\flin$ up to a small additive 
term, it requires the knowledge 
of a $L^\infty$-bounded ball in which $\flin$ lies and an upper bound
on $\sup_{x\in\X} \E\big\{[Y-\flin(X)]^2 \big| X=x\big\}$.
The looser this knowledge is, the bigger the constant in front of $d/n$ is.

Finally, we propose a randomized algorithm consisting in drawing the prediction function according to $\hpi$.
As usual, by convexity of the loss function, the risk of the deterministic estimator $\hf_{\text{determ}} = \int f \hpi(df)$
satisfies $R(\hf_{\text{determ}}) \le \int R(f) \hpi(df)$, so that, after some 
pretty standard computations, one can prove that 
for any $\eps>0$, with probability at least $1-\eps$:
	\[
	R( \hf_{\text{determ}} ) - R(\flin) \le \kap (2\sigma+H)^2\frac{d + \log(\eps^{-1})}{n},
	\]
for some appropriate numerical constant $\kap>0$. 
\end{rmk}

\begin{rmk}
The previous result was expressing boundedness in terms of the $L^\infty$ diameter of the 
set of functions $\cF$. By using Lemma \thmref{le:v3} instead of Lemma \thmref{le:v2},
Theorem \ref{th:main} still holds without assuming \eqref{eq:hh} and \eqref{eq:sigma},
but by replacing $(2\sigma+H)^2$ by 
  \begin{align*}
  V= \bigg[ 2& \sqrt{\sup_{f\in\Flin: \E[f(X)^2]=1}\E\big(f^2(X)[Y-f^*(X)]^2\big)}\\
    & \qquad + \sqrt{\sup_{f',f''\in\cF} \E\big([f'(X)-f''(X)]^2\big)} \sqrt{\sup_{f\in\Flin: \E[f(X)^2]=1}\E\big[f^4(X)\big]} \bigg]^2.
  \end{align*}
The quantity $V$ is finite when simultaneously, $\Theta$ is bounded, and for any $j$ in $\{1,\dots,d\}$, the
quantities $\E\big[\vp_j^4(X)\big]$ and $\E\big\{\vp_j(X)^2[Y-f^*(X)]^2\big\}$ 
are finite.    
\end{rmk}

\section{A generic localized PAC-Bayes approach} \label{sec:gen}

\subsection{Notation and setting}

In this section, we drop the restrictions of the linear least squares setting considered in the other sections
in order to focus on the ideas underlying the estimator and the results presented in Section~\ref{sec:main}. 
To do this, we consider that the loss incurred by predicting $y'$ while the correct output is $y$ is $\ela(y,y')$
(and is not necessarily equal to $(y-y')^2$).
The quality of a (prediction) function $f:\X\rightarrow\R$ is measured by its risk 
	\[R(f) = \E \bigl\{ \ela\bigl[Y,f(X) \bigr] \bigr\}.\]
We still consider the problem of predicting (at least) as well as the best function 
in a given set of functions $\cF$ (but $\cF$ is not necessarily a subset of a finite dimensional linear space).
Let $f^*$ still denote a function minimizing the risk among functions in $\cF$: $f^* \in\undc{\argmin}{f\in\cF} R(f)$.
For simplicity, we assume that it exists. 
The excess risk is defined by
	\[\bR(f) = R(f) - R(f^*).\]

Let $\ell: \Z\times\cF\times\cF \ra \R$ be a function 
such that $\ell(Z,f,f')$ represents\footnote{While the natural choice in the least squares setting is
$\ell((X,Y),f,f')=[Y-f(X)]^2 - [Y-f'(X)]^2$, we will see that
for heavy-tailed outputs, it is preferable to consider the following soft-truncated
version of it, up to a scaling factor $\lam>0$: $\ell((X,Y),f,f')=T\big(\lam\big[(Y-f(X))^2 - (Y-f'(X))^2\big]\big)$,
with $T(x)=-\log(1-x+x^2/2).$ Equality \myeq{eq:alih} corresponds to \myeq{eq:hcEls} with this choice of function $\ell$
and for the choice $\pis=\pi$.
}
how worse $f$ predicts than 
$f'$ on the data $Z$. 
Let us introduce the real-valued random processes $L: (f,f') \mapsto \ell(Z,f,f')$ and 
$L_i: (f,f') \mapsto \ell(Z_i,f,f')$, where $Z,Z_1,\dots,Z_n$ 
denote i.i.d. random variables with distribution $P$.

Let $\pi$ and $\pis$ be two (prior) probability distributions on $\cF$. 
We assume the following integrability condition.

{\bf Condition I.} \label{cond:i}
For any $f\in\cF$, we have 
\begin{align}
	\label{eq:intega}
	\int \E \bigl\{ \exp [ L(f,f') ]  \bigr\}^n \pis(df') & < +\infty,\\
\text{and } \quad
	\label{eq:integb}
	\int \frac{ \pi(df) }{\int \E \bigl\{ 
\exp [ L(f,f') ] \big\}^n \pis(df')} & < +\infty.
\end{align}
We consider the real-valued processes
	\begin{align}
	\hL(f,f')& = \sum_{i=1}^n L_i(f,f'),\\
	\hcE(f) & =\log \int \exp [ \hL(f,f')] \pis(df') \label{eq:alih},\\
	\La(f,f')& = - n \log \bigl\{ \E \bigl[ \exp[-L(f,f')] \bigr] \bigr\}, \\
	\Lb(f,f')& = n \log \bigl\{ \E \bigl[ \exp [ L(f,f')] \bigr] \bigr\}, \\ 
\text{and } \quad	\cEb(f)& = \log \Bigl\{ \tint \exp \bigl[ \Lb(f,f') \bigr] 
\pis(df') \Bigr\}. \label{eq:alib}
	\end{align}	
Essentially, the quantities $\hL(f,f')$, $\La(f,f')$ and $\Lb(f,f')$ represent 
how worse is the prediction from $f$ than from $f'$ with respect to the training data or in expectation. 
By Jensen's inequality, we have 
	\beglab{eq:interv}
	\La \le n \E (L) = \E (\hL) \le \Lb.
	\endlab
The quantities $\hcE(f)$ and $\cEb(f)$ should be understood as some kind of (empirical or expected) 
excess risk of the prediction function~$f$ with respect to an implicit reference induced by the integral over $\cF$.

For a distribution $\rho$ on $\cF$ absolutely continuous $\wrt$  $\pi$, let 
$\ds \frac{d \rho}{d \pi}$ denote the 
density of $\rho$ $\wrt$ $\pi$. For any real-valued (measurable) function $h$ defined on $\cF$ such that
$\int \exp [h(f)] \pi(df)<+\infty$, we define the distribution 
$\pi_{h}$ on $\cF$ by its density:
	\beglab{eq:pih}
	\frac{d \pi_{h}}{d \pi}(f) = 
\frac{\ds \exp [ h(f)]}{\ds \tint \exp [h(f')] \pi(df')}.
	\endlab
We will use the posterior distribution:
	\beglab{eq:hpi}
	\frac{d \hpi}{d \pi} (f) = 
\frac{ d \pi_{-\hcE}}{d \pi}(f) = 
\frac{\exp[-\hcE(f)]}{\int \exp[-\hcE(f')] \pi(df')}.
	\endlab	
Finally, for any $\be \ge 0$, we will use the following measures of the size (or complexity) of $\cF$ around 
the target function:
    \[
    \I^*(\be) = - \log \Bigl\{ \tint \exp \bigl[ - \be \bR(f)
\bigr]  \pis(df) \Bigr\}
    \]
and
    \[
    \I(\be) = - \log \Bigl\{ \tint \exp \bigl[ - \be \bR(f) 
\bigr]  \pi(df) \Bigr\}.
    \]

\subsection{The localized PAC-Bayes bound} \label{sec:gbound}

With the notation introduced in the previous section, 
we have the following risk bound for any randomized estimator.

\begin{thm} \label{th:gen}
Assume that $\pi$, $\pi^*$, $\cF$ and $\ell$ satisfy the 
integrability conditions \eqref{eq:intega} and \myeq{eq:integb}.
Let $\rho$ be a (posterior) probability distribution on $\cF$ 
admitting a density with respect to $\pi$ depending on $Z_1,\dots,Z_n$.
Let $\hf$ be a prediction function drawn from the distribution $\rho$.
Then for any $\ga \ge 0$, $\gas \ge 0$ and $\eps>0$, with probability (with respect to the distribution $P^{\otimes n} \rho$ generating
the observations $Z_1,\dots,Z_n$ and the randomized prediction function $\hf$) at least $1-\eps$:
	\begin{multline} \label{eq:gen}
	\int \big[ \La(\hf,f) + \gas \bR(f) \big] \pis_{-\gas \bR}(df) - \ga \bR
\bigl(\hf\,\bigr) \\
		\le \I^*(\gas) - \I(\ga)
	    - \log \Bigl\{ \tint \exp \bigl[ -\cEb(f) \bigr]  \pi(df) \Big\} 
	    \\ + \log \Bigl[  \frac{d \rho}{d \hpi}\bigl(\hf\, \bigr) \Bigr]   +  2\log(2\eps^{-1}).
	\end{multline}
\end{thm}

\begin{proof}
See Section \thmref{sec:pgeneric}.
\end{proof}

Some extra work will be needed to prove that Inequality \eqref{eq:gen} provides an upper bound on the excess risk 
$\bR(\hf)$ of the estimator $\hf$. As we will see in the next sections, despite the $-\ga \bR(\hf)$ term and provided that 
$\ga$ is sufficiently small, 
the lefthand-side will be essentially lower bounded by $\lam \bR(\hf)$ with $\lam>0$, while, 
by choosing $\rho=\hpi$, the estimator does not appear in the righthand-side.

\subsection{Application under an exponential moment condition} \label{sec:gfirst}

The estimator proposed in Section \ref{sec:main} and
Theorem \ref{th:gen} seems rather unnatural (or at least complicated) at first sight. 
The goal of this section is twofold. First it shows that 
under exponential moment conditions (i.e., stronger assumptions than the ones in Theorem \ref{th:main} when the linear least square
setting is considered), one can have 
a much simpler estimator than the one consisting in drawing a function according to the distribution \eqref{eq:hpils} 
with $\hcE$ given by \eqref{eq:hcEls} and yet still obtain a $d/n$ convergence rate.
Secondly it illustrates Theorem~\ref{th:gen} in a different and simpler way than the one we will use 
to prove Theorem \ref{th:main}. 

In this section, we consider the following variance and complexity assumptions.

 
{\bf Condition V1.} \label{cond:v1}
There exist $\lam>0$ and $0<\eta<1$ such that for any function $f\in\cF$, 
we have
$\E \Bigl\{ \exp \bigl\{ \lam \, \ela \bigl[Y,f(X)\bigl] \bigr\} \Bigr\}  < +\infty$,\\[-3ex] 
\begin{multline*}
	\log \Bigl\{ \E \Bigl\{ \exp \Bigl\{ 
\lam \, \Bigl[ \ela \bigl[ Y,f(X) \bigr] -  \ela \bigl[ Y,f^*(X) \bigr] 
\Bigr] \Bigr\} \Bigr\} \Bigr\}   \\ \le \lam(1+\eta) [R(f) - R(f^*)],
\end{multline*}
\vspace{-5ex}
\begin{multline*}
\text{and } \log \Bigl\{ \E \Bigl\{ \exp \Bigl\{ 
- \lam \Bigl[ \ela \bigl[Y,f(X)\bigr] - 
\ela \bigl[ Y,f^*(X) \bigr] \Bigr] \Bigr\} \Bigr\} \Bigr\} \\ \le -\lam(1-\eta) [R(f) - R(f^*)].
\end{multline*}

{\bf Condition C.} \label{cond:c}
There exist a probability distribution $\pi$, and constants $D>0$ and $G>0$ such that for any $0<\alpha<\be$,
	\[
	\log \bigg( \frac{\int \exp  \{ 
-\alpha[R(f)-R(f^*)] \} \pi(df)}{\int \exp \{ -\be[R(f)-R(f^*)] \} \pi(df)} \bigg)
		\le D \log\bigg(\frac{G \be}{\alpha}\bigg).
	\]

\begin{thm} \label{th:v1c}
Assume that \textnormal{V1} and \textnormal{C} are satisfied.
Let $\hpig$ be the probability distribution on $\cF$ defined by
its density
	\[
	\frac{d \hpig}{d \pi} (f)  
=\frac{\exp \{ -\lam \sum_{i=1}^n \ela[Y_i,f(X_i)]\}}
		{\int \exp \{ -\lam \sum_{i=1}^n \ela[Y_i,f'(X_i)]\} \pi(df')},
	\]
where $\lam>0$ and the distribution $\pi$ are 
those appearing respectively in \textnormal{V1}
and \textnormal{C}.
Let $\hf \in \cF$ be a function drawn according to this Gibbs distribution.
Then for any $\eta'$ such that $0 < \eta' <1-\eta$ (where $\eta$ is 
the constant appearing in \textnormal{V1}) 
and any $\eps>0$, with probability at least $1-\eps$, we have
	\[
	R( \hf ) - R(f^*) \le \frac{C'_1 D + C'_2 \log(2\eps^{-1})}{n} 
	\]
with 
	\[
	C'_1= \frac{\log(\frac{G(1+\eta)}{\eta' }) }{\lam(1-\eta-\eta' )} \quad \text{and} \quad
	C'_2=\frac{2}{\lam(1-\eta-\eta' )}.
	\]
\end{thm}

\begin{proof}
We consider $\ell\bigl[ (X,Y),f,f' \bigr] = \lam \bigl\{ 
\ela\bigl[ Y,f(X) \bigr] -\ela \bigl[ Y,f'(X) \bigr] \bigr\}$, 
where $\lam$ is the constant appearing in the variance assumption.
Let us take $\gas=0$ and let $\pis$ be the Dirac distribution at $f^*$: $\pis(\{f^*\})=1$.
Then Condition V1 implies Condition~I (page \pageref{cond:i}) and 
we can apply Theorem \ref{th:gen}. We have
	\begin{align*}
	L(f,f') & = \lam  \bigl\{ \ela \bigl[ Y,f(X) \bigr] 
-\ela \bigl[ Y,f'(X) \bigr] \bigr\},\\
	\hcE(f) & = \lam \sum_{i=1}^n \ela \bigl[ Y_i,f(X_i) \bigr]  
- \lam \sum_{i=1}^n \ela \bigl[ Y_i,f^*(X_i) \bigr], \\
	\hpi & = \hpig,\\
	\La(f) & = - n \log \Bigl\{ \E \Bigr[  \exp \bigl[ -L(f,f^*) 
\bigr] \Bigr] \Bigr\}, \\
	\cEb(f) & = n \log \Bigl\{ \E \Bigl[ \exp \bigl[ L(f,f^*) 
\bigr] \Bigr] \Bigr\}
	\end{align*}
and Assumption V1 leads to: 
\begin{align*}
\log \Bigl\{ \E  \Bigl[ \exp \bigl[ L(f,f^*) \bigr] 
\Bigr] \Bigr\} & \le \lam(1+\eta) [R(f) - R(f^*)]\\
\text{and } 
\log \Bigl\{ \E \Bigl[ \exp \bigl[ -L(f,f^*) \bigr] 
\Bigr] \Bigr\} & \le -\lam(1-\eta) [R(f) - R(f^*)].
\end{align*}
Thus choosing $\rho = \hpi$, \eqref{eq:gen} gives 
	\[
	[\lam n (1-\eta)-\ga] \bR(\hf) 
	    \le - \I(\ga) + \I\bigl[ \lam n (1+\eta) \bigr] + 2\log(2\eps^{-1}).
	\]
Accordingly by the complexity assumption, for $\ga \le \lam n (1+\eta)$, we get
	\[
	[\lam n (1-\eta)-\ga] \bR(\hf) 
		\le D \log\bigg( \frac{G\lam n(1+\eta)}{\ga} \bigg)
		+ 2\log(2\eps^{-1}),
	\]
which implies the announced result.
\end{proof}

Let us conclude this section by mentioning settings in which assumptions V1 and C are satisfied.

\begin{lemma} \label{le:complexity}
Let $\cC$ be a bounded convex set of $\R^d$, and $\vp_1,\dots,\vp_d$ be $d$ square integrable prediction functions. 
Assume that
	$$
	\cF= \big\{ f_\theta=\sum_{j=1}^d \th_j \vp_j ; (\th_1,\dots,\th_d) \in \cC \big\},
	$$
$\pi$ is the uniform distribution on $\cF$ (i.e., the one coming from the uniform distribution on $\cC$),
and that there exist $0<b_1\le b_2$ such that for any $y\in\R$, the function $\ela_y: y' \mapsto \ela(y,y')$
admits a second derivative satisfying: for any $y'\in\R$,
	\[
	b_1 \le \ela''_y(y') \le b_2.
	\]
Then Condition \textnormal{C} holds for the above uniform $\pi$, $G=\sqrt{b_2/b_1}$ and $D=d$.

Besides when $f^*=\flin$ (i.e., $\min_{\cF} R = \min_{\th\in\R^d} R(f_\th)$), Condition \textnormal{C} holds for the above uniform $\pi$, $G=b_2/b_1$ and $D=d/2$.
\end{lemma}

\begin{proof}
See Section \thmref{sec:complexity}.
\end{proof}

\begin{rmk} \label{rem:cls}
In particular, for the least squares loss $\ela(y,y')=(y-y')^2$, we have $b_1=b_2=2$ so 
that condition C holds with $\pi$ the uniform distribution on $\cF$, $D=d$ and $G=1$, and with $D=d/2$ and $G=1$ when $f^*=\flin$.
\end{rmk}


\begin{lemma} \label{le:v1b}
Assume that there exist $0<b_1\le b_2$, $A>0$ and $M>0$ such that
for any $y\in\R$, the functions $\ela_y: y' \mapsto \ela(y,y')$
are twice differentiable and satisfy: 
\begin{gather}
\text{for any }y'\in\R, \qquad b_1 \le \ela''_y(y') \le b_2,
\\ \text{and for any } x\in\X, \quad \E \Bigl\{ 
 \exp \Bigl[ A^{-1} \bigl\lvert \ela'_Y[f^*(X)] 
\bigr\rvert \Bigr]\, \Big| \,X=x \Bigr\} \le M.
\end{gather}
Assume that $\cF$ is convex and has a diameter $H$ for $L^\infty$-norm:
	\[
	\sup_{f_1,f_2\in \cF,x\in\X} |f_1(x)-f_2(x)| = H.
	\]
In this case Condition \textnormal{V1} holds for any $(\lam,\eta)$ such that
    \[
    \eta \ge \frac{\lam A^2}{2 b_1} \exp \Bigl[ M^2 \exp \bigl( H b_2/A 
\bigr)  \Bigr] .
    \]
and $0< \lam \le (2AH)^{-1}$ is small enough to ensure $\eta<1$.
\end{lemma}

\begin{proof}
See Section \thmref{sec:pv1b}.
\end{proof}

\subsection{Application without exponential moment condition} \label{sec:gsecond}

When we do not have finite exponential moments as assumed by 
Condition V1 (page \pageref{cond:v1}), e.g., when
	$\E \bigl\{ \exp \bigl\{ \lam \bigl\{ \ela[Y,f(X)]-\ela[Y,f^*(X)] 
\bigr\} \bigr\} \bigr\}  = +\infty$
for any $\lam>0$ and some function $f$ in $\cF$,
we cannot apply Theorem \ref{th:gen} with 
	$\ell\bigl[ (X,Y),f,f' \bigr]  = \lam  \bigl\{ \ela\bigl[Y,f(X)\bigr]-
\ela\bigl[ Y,f'(X) \bigr] \bigr\}$
(because of the $\cEb$ term). However,
we can apply it to the soft truncated excess loss
	\[
	\ell\bigl[ (X,Y),f,f' \bigr] =T
\Bigl( \lam \bigl\{ \ela\bigl[ Y,f(X) \bigr]  - \ela 
\bigl[ Y,f'(X) \bigr] \bigr\} \Bigr),
	\]
with $T(x)=-\log(1-x+x^2/2).$ 
This section provides a result similar to Theorem \ref{th:v1c} in which condition V1 is replaced by the following condition.

{\bf Condition V2.} \label{cond:v2}
For any function $f$, the random variable $\ela \bigl[ Y,f(X) \bigr] 
 - \ela \bigl[ Y,f^*(X) \bigr]$ is square integrable and 
there exists $V>0$ such that for any function $f$, 
	\[
	\E \Bigl\{ \Bigl[ \ela \bigl[ Y,f(X) \bigr]  
- \ela \bigl[ Y,f^*(X) \bigr] \Bigr]^2 \Bigr\} \le V [R(f) - R(f^*)].
	\]

\begin{thm} \label{th:v2c}
Assume that Conditions \textnormal{V2} above and \textnormal{C} 
(page \pageref{cond:c}) are satisfied.
Let $0 < \lam < V^{-1}$ and 
	\beglab{eq:elltrunc}
	\ell \bigl[ (X,Y),f,f' \bigr] =T \Bigl( \lam \big\{ \ela \bigl[Y,f(X)\bigr] 
- \ela \bigl[Y,
f'(X) \bigr] \bigr\}\Bigr),
	\endlab
with 
    \beglab{eq:deft}
    T(x)=-\log(1-x+x^2/2).
    \endlab
Let $\hf \in \cF$ be a function drawn according to the distribution $\hpi$ defined in \myeq{eq:hpi}
with $\hcE$ defined in \myeq{eq:alih} and $\pis=\pi$ the distribution appearing in Condition \textnormal{C}.
Then for any $0 < \eta' <1 - \lam V$ and $\eps>0$, with probability at least $1-\eps$, we have
	\[
	R( \hf ) - R(f^*) \le V \frac{C'_1 D + C'_2 \log(2\eps^{-1})}{n} 
	\]
with 
	\[
	C'_1= \frac{\log(\frac{G(1+\eta)^2}{\eta'(1-\eta)}) }{\eta (1-\eta-\eta')} \quad \text{and} \quad
	C'_2=\frac{2}{\eta(1-\eta- \eta')} \quad \text{and} \quad \eta= \lam V.
	\]
In particular, for $\lam=0.32 V^{-1}$ and $\eta'=0.18$, we get
	\[
	R( \hf ) - R(f^*) \le V \frac{16.6 D + 12.5 \log(2\sqrt G\eps^{-1})}{n}.
	\]
\end{thm}

\begin{proof}
We apply Theorem \ref{th:gen} for 
$\ell$ given by \eqref{eq:elltrunc} and $\pis=\pi$.
Let 
	\begin{align*}
	W(f,f') & = \lam \bigl\{ \ela \bigl[ Y,f(X) \bigr] -\ela \bigl[ Y,f'(X)
\bigr]  \bigr\}  && \text{for any $f,f'\in\cF$}. \\
	\end{align*}
Since $\log u \le u-1$ for any $u>0$, we have
	\[
	\La = - n \log \E ( 1 - W + W^2/2 ) \ge n ( \E W - \E W^2/2 ).
	\]
Moreover, from Assumption V2, 
	\beglab{eq:usev}
	\frac{\E W(f,f')^2}2\le \E W(f,f^*)^2+\E W(f',f^*)^2
		\le \lam^2 V \bR(f)+ \lam^2 V \bR(f'),
	\endlab
hence, by introducing $\eta= \lam V$,
	\begin{align} \label{eq:bla}
	\La(f,f') & \ge \lam n \big[ \bR(f) - \bR(f') - \lam V \bR(f) - \lam V \bR(f') \big] \notag\\
		& = \lam n \big[ (1-\eta) \bR(f) - (1+\eta) \bR(f') \big].
	\end{align}
Noting that 
	\[
	\exp \bigl[ T(u) \bigr] = \frac{1}{1-u+u^2/2}
	\\ = \frac{1 + u + \frac{u^2}{2}}{\bigl( 1 + \tfrac{u^2}{2}
\bigr)^2 - u^2} = \frac{1 + u + \frac{u^2}{2}}{1 + \frac{u^4}{4}}
\leq 1 + u + \frac{u^2}{2}, 
	\]
we see that 
	\[
	\Lb = n \log \Bigl\{ \E \Bigr[  \exp \bigl[ T(W) \bigr] \Bigr] \Bigr\} \le n 
\bigl[ \E\bigl(W\bigr) + \E \bigl(W^2\bigr)/2  \bigr].
	\]
Using \eqref{eq:usev} and still $\eta= \lam V$, we get
	\bmul
	\Lb(f,f') \le \lam n \big[ \bR(f) - \bR(f') + \eta \bR(f) + \eta \bR(f') \big]\\
		= \lam n (1+\eta) \bR(f) - \lam n (1-\eta) \bR(f'),
	\end{multline*}
and	
	\beglab{eq:bceb}
	\cEb(f) \le \lam n (1+\eta) \bR(f) - \I(\lam n(1-\eta)).
	\endlab
Plugging \eqref{eq:bla} and \eqref{eq:bceb} in \eqref{eq:gen} for $\rho=\hpi$, we obtain
	\bmul
	[ \lam n (1-\eta) - \ga ] \bR(\hf) + [\gas-\lam n (1+\eta)] \tint \bR(f) \pi_{-\gas \bR}(df) \\
	    \le \I(\gas) - \I(\ga) + \I(\lam n (1+\eta)) - \I(\lam n (1-\eta))+2\log(2\eps^{-1}).
	\end{multline*}
By the complexity assumption, choosing $\gas = \lam n (1+\eta)$ and $\ga < \lam n (1-\eta)$, we get
	\[
	[ \lam n (1-\eta) - \ga ] \bR(\hf) 
	    \le D \log\bigg( G \frac{\lam n (1+\eta)^2}{\ga (1-\eta)} \bigg)+2\log(2\eps^{-1}),
	\]
hence the desired result by considering $\ga= \lam n \eta'$ with $\eta'< 1-\eta$.
\end{proof}

\begin{rmk} \label{rem:complicate}
The estimator seems abnormally complicated at first sight.
This remark aims at explaining why we were not able to consider a simpler estimator.

In Section \ref{sec:gfirst}, in which we consider the exponential moment
condition V1, we took $\ell \bigl[ (X,Y),f,f' \bigr] =
\lam\bigl\{ \ela \bigl[ Y,f(X) \bigr]  - \ela \bigl[ Y,f'(X) \bigr] 
\bigr\}$
and $\pis$ as the Dirac distribution at $f^*$. For these choices, one can easily check
that $\hpi$ does not depend on~$f^*$.

In the absence of an exponential moment condition, we cannot
consider the function $\ell \bigl[ (X,Y),f,f' \bigr] =\lam 
\bigl\{ \ela \bigl[ Y,f(X) \bigr]  - \ela \bigl[ Y,f'(X) \bigr] \bigr\} $
but a truncated version of it. The truncation function $T$ we use in Theorem \ref{th:v2c}
can be replaced by the simpler function $u \mapsto (u\vee -M)\wedge M$
for some appropriate constant $M>0$ but this would lead to a bound with worse constants,
without really simplifying the algorithm.
The precise choice $T(x)=-\log(1-x+x^2/2)$ comes from the remarkable property:
there exist second order polynomial $\Pa$ and $\Pb$ such that
	$\frac{1}{\Pa(u)} \le  \exp \bigl[ T(u) \bigr]  \le \Pb(u)$
and
	$\Pa(u)\Pb(u) \le 1+\text{O}(u^4)$ for $u\ra 0$, which 
are reasonable properties to ask in order to ensure 
that \eqref{eq:interv}, and consequently \eqref{eq:gen}, are tight.

Besides, if we take $\ell$ as in \eqref{eq:elltrunc} with $T$ a truncation function 
and $\pis$ as the Dirac distribution at $f^*$,
then $\hpi$ would depend on $f^*$, and is consequently not observable.
This is the reason why we do not consider $\pis$ as the Dirac distribution at $f^*$,
but $\pis=\pi$. This lead to the estimator considered in Theorems \ref{th:v2c} and \ref{th:main}.
\end{rmk}

\begin{rmk} 
Theorem \ref{th:v2c} still holds for the same randomized estimator in which 
\myeq{eq:deft} is replaced with
	\[
	T(x) = \log(1+x+x^2/2).
	\]
\end{rmk}

Condition V2 holds under weak assumptions
as illustrated by the following lemma.

\begin{lemma} \label{le:v2}
Consider the least squares setting: $\ela(y,y') = (y-y')^2$.
Assume that $\cF$ is convex and has a diameter $H$ for $L^\infty$-norm:
	\[
	\sup_{f_1,f_2\in \cF,x\in\X} |f_1(x)-f_2(x)| = H
	\]
and that for some $\sigma>0$, we have
	\beglab{eq:moment}
	\sup_{x\in\X} \E\big\{[Y-f^*(X)]^2 \big| X=x\big\} \le \sigma^2 < +\infty.
	\endlab
Then Condition \textnormal{V2} holds for $V=(2\sigma+H)^2$.
\end{lemma}

\begin{proof}
See Section \thmref{sec:pv2}.
\end{proof}

\begin{lemma} \label{le:v3}
Consider the least squares setting: $\ela(y,y') = (y-y')^2$.
Assume that $\cF$ (i.e., $\Theta$) is bounded, and that for any $j\in\{1,\dots,d\}$,
we have $\E\big[\vp_j^4(X)\big]~<~+~\infty$
and
$\E\big\{\vp_j(X)^2[Y-f^*(X)]^2\big\}<+\infty$.  
Then Condition \textnormal{V2} holds for 
  \begin{align*}
  V= \bigg[ 2& \sqrt{\sup_{f\in\Flin: \E[f(X)^2]=1}\E\big(f^2(X)[Y-f^*(X)]^2\big)}\\
    & \qquad + \sqrt{\sup_{f',f''\in\cF} \E\big([f'(X)-f''(X)]^2\big)} \sqrt{\sup_{f\in\Flin: \E[f(X)^2]=1}\E\big[f^4(X)\big]} \bigg]^2.
  \end{align*}
\end{lemma}

\begin{proof}
See Section \thmref{sec:pv3}.
\end{proof}

\section{Proofs}

\subsection{Main ideas of the proofs}

The goal of this section is to explain the key ingredients appearing in the proofs which both allows to obtain sub-exponential tails 
for the excess risk under a non-exponential moment assumption and get rid of the logarithmic factor in the excess risk bound.

\subsubsection{Sub-exponential tails under a non-exponential moment assumption via truncation}

Let us start with the idea allowing us to prove exponential inequalities under just a moment assumption (instead of the traditional exponential moment assumption).
To understand it, we can consider the (apparently) simplistic $1$-dimensional situation in which we have $\Theta=\R$ and the marginal distribution of $\varphi_1(X)$ is the Dirac distribution at $1$. 
In this case, the risk of the prediction function $f_\th$ is $R(f_\th)=\E(Y-\th)^2=\E(Y-\th^*)^2 + (\E Y -\th)^2,$ 
so that the least squares regression problem boils down to the estimation of the mean of the output variable.
If we only assume that $Y$ admits a finite second moment, say $\E Y^2\le 1$, it is not clear whether for any $\eps>0$, it is possible to find $\hth$ such that
with probability at least $1-2\eps$,
  \beglab{eq:tar1}
  R(f_{\hth})-R(f^*) = (\E (Y) -\hth)^2 \le c \frac{\logeps}{n},
  \endlab
for some numerical constant $c$.
Indeed, from Chebyshev's inequality, the trivial choice $\hth=\frac{\sum_{i=1}^n Y_i}{n}$ just satisfies: with probability at least $1-2\eps$,
  $$R(f_{\hth})-R(f^*) \le \frac1{n\eps},$$
which is far from the objective \eqref{eq:tar1} for small confidence levels (consider $\eps=\exp(-\sqrt{n})$ for instance).
The key idea is thus to average (soft) \emph{truncated} values of the outputs. This is performed by
taking
  $$\hth = \frac1{n\lam}\sum_{i=1}^n \log\bigg(1+\lam Y_i+\frac{\lam^2Y_i^2}{2}\bigg),$$
with $\lam=\sqrt{\frac{2 \logeps}n}$.
Since we have
  $$\log\E \exp(n\lam \hth) = n \log\bigg( 1+ \lam \E(Y) + \frac{\lam^2}2 \E( Y^2 ) \bigg) \le n\lam \E( Y )+ n \frac{\lam^2}2,$$
the exponential Chebyshev's inequality (see Lemma \ref{le:pac}) guarantees that with probability at least $1-\eps$, we have
  $
  n\lam (\hth-\E(Y)) \le n \frac{\lam^2}2 + \logeps
  $,
hence
  $$\hth-\E (Y)\le \sqrt\frac{2\logeps}{n}.$$
Replacing $Y$ by $-Y$ in the previous argument, we obtain that 
with probability at least $1-\eps$, we have
  $$
  n\lam \bigg\{ \E(Y)+\frac1{n\lam}\sum_{i=1}^n \log\bigg(1-\lam Y_i+\frac{\lam^2Y_i^2}{2}\bigg) \bigg\} \le n \frac{\lam^2}2 + \logeps.
  $$
Since $-\log(1+x+x^2/2) \le \log(1-x+x^2/2)$, this implies
  $\E (Y)-\hth\le \sqrt\frac{2\logeps}{n}.$
The two previous inequalities imply Inequality \eqref{eq:tar1} (for $c={2}$), showing that sub-exponential tails are achievable even when we only assume that the random variable admits a finite second moment (see \cite{Cat09} for more details 
on the robust estimation of the mean of a random variable).

\subsubsection{Localized PAC-Bayesian inequalities to eliminate a logarithm factor}

\paragraph{High level description of the PAC-Bayesian approach and the localization argument.}

The analysis of statistical inference generally relies on upper bounding the supremum of an empirical process $\chi$ indexed by the functions in a model $\cF$. One central tool to obtain these bounds is the concentration inequalities.
An alternative approach, called the PAC-Bayesian one, consists in using the entropic equality
  \beglab{eq:iexp}
  \E \exp\Bigg( \sup_{\rho\in\M} \bigg\{ \int \rho(df) \chi(f) - K(\rho,\pi') \bigg\} \Bigg)= \int\pi'(df) \E \exp\big( \chi(f) \big).
  \endlab
where $\M$ is the set of probability distributions on $\cF$ and $K(\rho,\pi')$ is the Kullback-Leibler divergence (whose definition is recalled in \eqref{eq:kl}) between $\rho$ and some fixed distribution $\pi'$.

Let $\chr:\cF \ra \R$ be an observable process such that for any $f\in\cF$, we have
  $$
  \E \exp\big(\chi(f)\big) \le 1
  $$
for $\chi(f) =\lam[ R(f) - \chr(f)]$ and some $\lam>0$. Then \eqref{eq:iexp} leads to: for any $\eps>0$, with probability at least $1-\eps$,
for any distribution $\rho$ on $\cF$, we have
  \beglab{eq:ipac}
  \int \rho(df) R(f) \le \int \rho(df) \chr(f) + \frac{K(\rho,\pi')+\logeps}{\lam}.
  \endlab
The lefthand-side quantity represents the expected risk with respect to the distribution $\rho$. To get the smallest upper bound on this quantity, 
a natural choice of the (posterior) distribution $\rho$ is obtained by minimizing the righthand-side, that is by taking $\rho=\pi'_{-\lam \chr}$
(with the notation introduced in \eqref{eq:pih}). This distribution concentrates on functions $f\in\cF$ for which $\chr(f)$ is small. 
Without prior knowledge, one may want to choose a prior distribution $\pi'=\pi$ which is rather ``flat'' (e.g., the one induced by the Lebesgue measure in the case of
a model $\cF$ defined by a bounded parameter set in some Euclidean space).
Consequently the Kullback-Leibler divergence $K(\rho,\pi')$, which should be seen as the complexity term, might be excessively large.

To overcome the lack of prior information and the resulting high complexity term, one can alternatively use a more ``localized'' prior distribution $\pi'=\pi_{-\be R}$ for some $\be>0$. 
Since the righthand-side of \eqref{eq:ipac} is then no longer observable, an empirical upper bound on 
$K(\rho,\pi_{-\be R})$ is required. It is obtained by writing
  $$
  K(\rho,\pi_{-\be R})=K(\rho,\pi)+\log\bigg( \int \pi(df) \exp[-\be R(f)] \bigg) + \beta \int \rho(df) R(f),
  $$ 
and by controlling the two non-observable terms by their empirical versions, calling for additional PAC-Bayesian inequalities.

\paragraph{Low level description of localization.}

To simplify a more detailed presentation of the PAC-Bayesian localization argument, we will consider a setting in which $\cF$, $\varphi_1$, \dots, $\varphi_d$ and the outputs are bounded almost surely, specifically assume
  $\P( \text{for any }f\in\cF, |Y-f(X)| \le 1 ) = 1.$
  
Introduce $\Psi(u)=[\exp(u)-1-u]/u^2$ for any $u>0$, $\bR(f)=R(f)-R(f^*)$ and $\br(f)=r(f)-r(f^*)$ for any $f\in\cF$.
Let $\pi$ be a distribution on $\cF$ and $\Delta(f,f')=\E \big\{[Y-f(X)]^2-[Y-f^*(X)]^2\big\}^2.$
The starting point is the following PAC-Bayesian inequality: for any $\eps>0$ and $\lam>0$, with probability at least $1-\eps$,
for any distribution $\rho$ on $\cF$, we have 
  \begin{multline}\label{eq:id2a}
  \int \rho(df) \bR(f) \le \int \rho(df) \br(f) 
    + \frac{\lam}{n}\Psi\Big(\frac{2\lam}{n}\Big) \int \rho(df) \Delta(f,f^*) \\
    + \frac{K(\rho,\pi)+\logeps}{\lam}.
  \end{multline}
This inequality derives from the duality formula given in \eqref{eq:legendre}, the inequality
$\E \exp\Big(\frac\lam{n}\big\{[Y-f^*(X)]^2-[Y-f(X)]^2+R(f)-R(f^*)\big\} - \frac{\lam^2}{n^2}\Psi\big(\frac{2\lam}{n} \big) \Delta(f,f^*) \Big)\le 1$, 
and Lemma \ref{le:pac} (see \cite[Theorem 8.1]{Aud03b}). 
Since 
  \begin{multline*}
  \Delta(f,f^*) = \E\big\{[f(X)-f^*(X)]^2[2Y-f(X)-f^*(X)]^2\big\} \\\le 4\E\big\{[f(X)-f^*(X)]^2\big\} \le 4 \bR(f),
  \end{multline*}
by taking $\lam=n/6$, Inequality \eqref{eq:id2a} implies
  \beglab{eq:unloc}
  \int \rho(df) \bR(f) \le 2\int \rho(df) \br(f) + 10\frac{K(\rho,\pi)+\logeps}{n}.
  \endlab
The distribution 
  $\hpi(df)=\frac{\exp[-n \br(f)/5]}{\tint
\exp[-n \br(f')/5]\pi(df')} \cdot \pi(df)$
minimizes the righthand-side, and we have
  $$
  \int \hpi(df) \bR(f) \le 10\frac{-\log \big( \tint \pi(df) \exp[-n \br(f)/5] \big)+\logeps}{n}.
  $$
Let $\pi_U$ be the uniform distribution on $\cF$ (i.e., the one coming from the uniform distribution on $\cC$).
For $\pi=\pi_U$, using similar arguments to the ones
developed in Section \ref{sec:complexity},
it can be shown that $-\log \big( \tint \pi(df) \exp[-n \br(f)/5]\le c d \log(n)$ for some constant $c$ depending only on $\sup_{f,f'\in\cF}\|f-f'\|_\infty$.
This implies a $\frac{d \log n}n$ convergence rate of the excess risk of the randomized algorithm associated with~$\hpi$.
  
The localization idea from \cite{Cat03b} allows to prove
  \beglab{eq:loc}
  \int \rho(df) \bR(f) \le 2\int \rho(df) \br(f) + 10\frac{K(\rho,\hpi')+\logeps}{n},
  \endlab
with $\hpi'(df)=\frac{\exp[-\zeta n \br(f)]}{\tint
\exp[-\zeta n \br(f')]\pi(df')} \cdot \pi(df)$ for some $0<\zeta<1/5$.
The key difference with \eqref{eq:unloc} is that the Kullback-Leibler term is now much smaller for the distributions $\rho$ which concentrates on low empirical risk functions, like $\hpi$.
Since $-\log \big( \tint \hpi'(df) \exp[-n \br(f)/5]\le c d$ for some constant $c$ depending only on $\zeta$ (see Lemma \ref{le:complexity}),
this allows to get rid of the $\log n$ factor and obtain a convergence rate of order $d/n$.

The proof of \eqref{eq:loc} is rather intricate but the central idea is to use 
\eqref{eq:unloc} for $\pi(df)=\frac{\exp[-n \bR(f)/5]}{\tint \exp[-n \bR(f')/5]\pi(df')} \cdot \pi_U(df)$, 
and control the non-observable Kullback-Leibler term by $c \int \rho(df) \bR(f)$ plus $K(\rho,\hpi')$ up to minor additive terms. 

Let us conclude this section by pointing out some difficulties and possibilities when considering unbounded $Y-f_{\th}(X)$.
The sketches of proof presented hereafter are far from being actual proofs as some technical problems are hidden. 
Full proofs will be given in the later sections. 
For unbounded $Y-f_{\th}(X)$, Inequality \eqref{eq:id2a} no longer holds, but by using the soft truncation argument of the previous section, one can prove a similar inequality in which $\int \rho(df) \br(f)$ is replaced with
  $\frac{1}\lam \int\rho(df) \sum_{i=1}^n \log\big(1+W_i(f,f^*)+W_i^2(f,f^*)/2\big)$
for $W_i(f,f^*)=\frac{\lam}n \big\{[Y-f(X_i)]^2-[Y-f^*(X_i)]^2\big\}$ for $\lam>0$ a parameter of the bound.
One significant difficulty is that the minimizer of this quantity is no longer observable (since $f^*$ is unknown).
Nevertheless the quantity can be upper bounded by the observable one:
  $$\max_{f'\in\cF} \frac{1}\lam \int\rho(df) \sum_{i=1}^n \log\bigg(1+W_i(f,f')+\frac{W_i^2(f,f')}2\bigg).$$
This explains why the procedures in Section \ref{sec:computable} make appear a min-max.

Another interesting idea is to use Gaussian distributions for $\pi$ and $\rho$,
which are respectively centered at $\th^*$ and $\hth$ and with covariance matrix proportional to the identity matrix.
The interest of these choices comes essentially from the coexistence of the two following properties:
the distribution $\pi$ concentrates on a neighbourhood of the best prediction function so the complexity term $K(\rho,\pi)$
can be much smaller than the one obtained for $\pi$ the uniform distribution on $\cF$ (this is again the localization idea), and
$K(\rho,\pi)$ and, when $\Theta=\R^d$, the integrals with respect to $\rho$ can be explicitly computed in terms of $\bR(\hth)$ and other rather simple quantities, 
which implies that the modified inequality \eqref{eq:id2a} gets a tractable form for further computations, provided nevertheless 
some assumptions on the eigenvalues of the matrix $Q$. The idea of using PAC-Bayesian inequalities with Gaussian prior and posterior distributions
has first been proposed by Langford and Shawe-Taylor \cite{LanSha02} in the context of linear classification.

\subsection{Proofs of Theorems \ref{th:hfrlam} and \ref{th:ermom}} \label{sec:permom}

To shorten the formulae, we will write $X$ for $\vp(X)$, which is equivalent to considering 
without loss of generality that the input space is $\R^d$ and that the functions $\vp_1,\dots,$$\vp_d$ are
the coordinate functions.
Therefore, the function $f_\th$ maps an input $x$ to $\langle \theta, x \rangle$.
With a slight abuse of notation, $R(\theta)$ will denote the risk of this prediction function.

Let us first assume that the matrix $Q_\lam=Q+\lam I$
is positive definite. This indeed does not restrict the generality of our study, 
even in the case when $\lam= 0$, as we will discuss later (Remark \ref{rmk:degen}).
Consider the change of coordinates
$$
\V{X} = Q_\lam^{-1/2} X.
$$
Let us introduce
$$
\V{R}(\theta) = \B{E} \bigl[ (\langle \theta, \V{X} \rangle - Y)^2 \bigr],
$$
so that
$$
\V{R}(Q_\lam^{1/2} \theta) 
= R(\theta) = \B{E} \bigl[(\langle \theta, X \rangle - Y)^2 \bigr]. 
$$
Let 
$$
\V{\Theta} = \bigl\{ Q_\lam^{1/2} \theta ; \theta \in \Theta \bigr\}.
$$
Consider
\begin{align}
r(\theta) & = \frac{1}{n} \sum_{i=1}^n \bigl( 
\langle \theta, X_i \rangle - Y_i \bigr)^2, \\
\V{r}(\theta) & = \frac{1}{n} \sum_{i=1}^n \bigl(\langle \theta, \V{X}_i 
\rangle - Y_i \bigr)^2,\\
\theta_0 & = \arg \min_{\theta \in \V{\Theta}} \V{R}(\theta) + 
\lam \lVert Q_\lam^{-1/2} \theta \rVert^2,\\
\hat{\th} & \in \arg \min_{\theta \in \Theta} r(\theta) 
+ \lam \lVert \theta \rVert^2, \\
\theta_1 & = Q_\lam^{1/2} \hat{\th}  \in \arg \min_{\theta \in \V{\Theta}} 
\V{r}(\theta) + \lam \lVert Q_\lam^{-1/2} \theta \rVert^2. 
\end{align}

For $\alpha>0$, let us introduce the notation
\begin{align*}
W_i(\theta) & = \alpha\Big\{ \bigl( \langle \theta, \V{X}_i \rangle - Y_i \bigr)^2 -
\bigl( \langle \theta_0, \V{X}_i \rangle - Y_i \bigr)^2 \Big\},\\
W(\theta) & = \alpha\Big\{ \bigl( \langle \theta, \V{X} \rangle - Y \bigr)^2 -
\bigl( \langle \theta_0, \V{X} \rangle - Y \bigr)^2 \Big\}.
\end{align*}
For any $\theta_2 \in \B{R}^d$ and $\beta>0$, let us consider the Gaussian distribution
centered at~$\theta_2$ 
$$
\rho_{\theta_2}(d \theta) = \left( \frac{\beta}{2 \pi}\right)^{d/2} \exp 
\left( - \frac{\beta}{2} \lVert \theta  - \theta_2 \rVert^2 \right) d \theta.
$$
\begin{lemma}
\label{lemma7.1}
For any $\eta>0$ and $\alpha > 0$, 
with probability at least $1 - \exp( - \eta)$, 
for any $\theta_2 \in \B{R}^d$, 
\begin{multline*}
- n \tint \rho_{\theta_2} (d \theta) 
\log \Bigl\{ 1 - \B{E} \bigl[ W(\theta) \bigr] 
+ \B{E} \bigl[ W(\theta)^2 \bigr]/2 \Bigr\}
\\ \leq - \sum_{i=1}^n \left( \tint \rho_{\theta_2}(d \theta) 
\log \Bigl\{ 1 - W_i(\theta) + W_i(\theta)^2/2 \Bigr\} \right) + 
\C{K}(\rho_{\theta_2}, \rho_{\theta_0}) + \eta,
\end{multline*}
where $\C{K}(\rho_{\theta_2}, \rho_{\theta_0})$ is 
the Kullback-Leibler divergence function :
$$ 
\C{K}(\rho_{\theta_2}, \rho_{\theta_0}) = 
\int \rho_{\theta_2} (d \theta) \log \biggl[ \frac{ d 
\rho_{\theta_2}}{d \rho_{\theta_0}}( \theta) \biggr].
$$
\end{lemma}
\begin{proof}
$$
\B{E} \left( \tint \rho_{\theta_0}( 
d \theta) \prod_{i=1}^n \frac{1  - W_i(\theta) 
+ W_i(\theta)^2/2}{ 1 - 
\B{E} \bigl[ W(\theta) \bigr] 
+ \B{E} \bigl[ 
W(\theta)^2 \bigr]/2 } \right) 
\leq 1,
$$
thus with probability at least $1 - \exp(- \eta)$ 
$$
\log \left( \tint \rho_{\theta_0}( 
d \theta) \prod_{i=1}^n \frac{1  - W_i(\theta) 
+ W_i(\theta)^2/2}{ 1 - 
\B{E} \bigl[ W(\theta) \bigr] 
+ \B{E} \bigl[ 
W(\theta)^2\bigr]/2 } \right) \leq \eta.
$$
We conclude from the convex inequality 
(see \cite[page 159]{Cat01})
$$
\log \left( \tint \rho_{\theta_0}(d \theta) \exp  \bigl[ 
h(\theta) \bigr]  \right) 
\geq \tint \rho_{\theta_2}(d \theta) h(\theta) 
- \C{K}(\rho_{\theta_2}, \rho_{\theta_0}).
$$
\end{proof}

Let us compute some useful quantities
\begin{align}
\C{K}(\rho_{\theta_2}, \rho_{\theta_0}) & = \frac{\beta}{2} 
\lVert \theta_2 - \theta_0 \rVert^2, \label{eq:kthth}\\
\tint \rho_{\theta_2}(d \theta) \bigl[ W(\theta) \bigr] 
& = \alpha \tint \rho_{\theta_2}(d \theta) \langle \theta - \theta_2, \V{X} \rangle^2
+ W(\theta_2) \notag\\
& = W(\theta_2) + \alpha \frac{\lVert \V{X} \rVert^2}{\beta},
\label{eq6.13}\\
\tint \rho_{\theta_2}(d \theta) \langle \theta - \theta_2, \V{X} \rangle^4 & = \frac{3 \lVert \V{X} \rVert^4}{\beta^2},
\end{align}
\begin{multline}
\tint \rho_{\theta_2}(d \theta) \bigl[ W(\theta)^2 \bigr] 
= \alpha^2 \tint \rho_{\theta_2}(d \theta) \langle \theta - \theta_0, 
\V{X} \rangle^2 \bigl( \langle \theta + \theta_0, \V{X} \rangle
- 2 Y \bigr)^2 \\ = \alpha^2
\tint \rho_{\theta_2}(d \theta) \Bigl[ \langle \theta - \theta_2 
+ \theta_2 - \theta_0, \V{X} \rangle \bigl( \langle
\theta - \theta_2 + \theta_2 + \theta_0, \V{X} \rangle 
- 2 Y \bigr) \Bigr]^2
\\ = \tint \rho_{\theta_2}(d \theta) \Bigl[ 
\alpha\langle \theta - \theta_2, \V{X} \rangle^2 + 
2 \alpha \langle \theta - \theta_2, \V{X} \rangle 
\bigl( \langle
\theta_2, \V{X} \rangle - Y \bigr) + W(\theta_2) 
\Bigr]^2 \\ 
= \tint \rho_{\theta_2}(d \theta) \Bigl[ 
\alpha^2 \langle \theta - \theta_2, \V{X} \rangle^4 
+ 4 \alpha^2 \langle \theta - \theta_2, \V{X} \rangle^2 
\bigl( \langle \theta_2, \V{X} \rangle - Y \bigr)^2 
+ W(\theta_2)^2 \\ 
+ 2 \alpha \langle \theta - \theta_2, \V{X} \rangle^2 
W(\theta_2) \Bigr] \\ 
= \frac{3 \alpha^2 \lVert \V{X} \rVert^4}{\beta^2}  
+ \frac{2 \alpha \lVert \V{X} \rVert^2}{\beta}
\Bigl[ 2 \alpha \bigl( \langle \theta_2, \V{X} \rangle - Y \bigr)^2 
+ W(\theta_2) \Bigr] + W(\theta_2)^2.
\label{eq6.15}
\end{multline}
Using the fact that 
$$
2 \alpha \bigl( \langle \theta_2 , \V{X} 
\rangle - Y \bigr)^2 + W(\theta_2) = 
2 \alpha \bigl( \langle \theta_0, \V{X} \rangle - Y \bigr)^2 
+ 3 W(\theta_2),
$$
and that for any real numbers $a$ and $b$, $6ab \le 9a^2+b^2$,
we get 
\begin{lemma} \label{le:7.2}
\begin{align}
\tint \rho_{\theta_2}(d \theta) 
\bigl[ W(\theta) \bigr] 
& = W(\theta_2) + \alpha \frac{\lVert \V{X} \rVert^2}{\beta},\\
\tint \rho_{\theta_2}(d \theta) \bigl[ 
W(\theta)^2 \bigr] & = W(\theta_2)^2 + \frac{2 \alpha \lVert \V{X} \rVert^2}{\beta}
\Bigl[ 2 \alpha \bigl( \langle \theta_0, \V{X} \rangle 
- Y \bigr)^2 + 3 W(\theta_2) \Bigr] \nonumber\\ 
& \qquad + \frac{3 \alpha^2 \lVert \V{X} \rVert^4}{\beta^2}
\\ & \leq 10 W(\theta_2)^2 + \frac{4 \alpha^2 \lVert \V{X} \rVert^2}{\beta}
\bigl( \langle \theta_0, \V{X} \rangle - Y \bigr)^2 
+ \frac{4 \alpha^2 \lVert \V{X} \rVert^4}{\beta^2},
\end{align}
and the same holds true when $W$ is replaced with $W_i$ 
and $(\V{X}, Y)$ with $(\V{X}_i, Y_i)$.
\end{lemma}

Another important thing to realize is that
\begin{align} \label{eq:normd}
\B{E} \bigl[ \lVert \V{X} \rVert^2 \bigr] 
& =  \B{E} \bigl[  \Tr \bigl( \V{X}\, \V{X}^T \bigr) \bigr] 
&& = \B{E} \bigl[ \Tr \bigl( Q_\lam^{- 1/2} X X^T Q_\lam^{-1/2} \bigr) 
\bigr] \notag
\\ & = \B{E} \bigl[ \Tr \bigl( Q_\lam^{-1} X X^T \bigr) \bigr]
&& = \Tr \bigl[ Q_\lam^{-1} \B{E} (XX^T ) \bigr] \notag\\
& = \Tr \bigl( Q_\lam^{-1}(Q_\lam- \lam I)\bigr) 
 && = d - \lam \Tr(Q_\lam^{-1}) \qquad =\ D \,. 
\end{align}

We can weaken Lemma \thmref{lemma7.1} noticing that
for any real number $x$, $x\le -\log(1-x)$ and
\begin{multline*}
- \log \biggl(1 - x + \frac{x^2}2
\biggr)
= \log \left( \frac{1 + x + x^2/2}{
1 + x^4/4} \right) \\ \leq 
\log \biggl( 1 + x + \frac{x^2}2 \biggr) 
\leq x + \frac{x^2}2. 
\end{multline*}
We obtain with probability at least $1 - \exp(- \eta)$
\begin{multline*}
n \B{E} \bigl[ W(\theta_2) \bigr] +  
\frac{n \alpha}{\beta} \B{E} \bigl[ \lVert \V{X} \rVert^2 \bigr]  
- 5 n \B{E} \bigl[ W(\theta_2)^2 \bigr]
\\ - \B{E} \Biggl\{ 
\frac{2 n \alpha^2 \lVert \V{X} \rVert^2}{\beta} 
\bigl( \langle \theta_0, \V{X} \rangle 
- Y \bigr)^2 +  
\frac{2 n \alpha^2 \lVert \V{X} \rVert^4}{\beta^2} \Biggr\} 
\\ \leq \sum_{i=1}^n \Biggl\{ W_i(\theta_2) 
+  5 W_i(\theta_2)^2 
\\ + \frac{\alpha \lVert \V{X}_i \rVert^2}{\beta} 
+ \frac{2 \alpha^2 \lVert \V{X}_i \rVert^2}{\beta} 
\bigl( \langle \theta_0 , \V{X}_i \rangle - Y \bigr)^2 
+ \frac{2 \alpha^2 \lVert \V{X}_i \rVert^4}{\beta^2} \Biggr\} 
\\ + \frac{\beta}{2} \lVert \theta_2 - \theta_0 \rVert^2 + \eta. 
\end{multline*}

Noticing that for any real numbers $a$ and $b$, $4ab \le a^2+4b^2$,
we can then bound
\begin{multline*}
\alpha^{-2}W(\theta_2)^2 = \langle \theta_2 - \theta_0, \V{X} \rangle^2 
\bigl( \langle \theta_2 + \theta_0, \V{X} \rangle - 2 Y \bigr)^2 \\ 
= \langle \theta_2 - \theta_0, \V{X} \rangle^2 
\Bigl[ \langle \theta_2 - \theta_0, \V{X} \rangle 
+ 2 \bigl( \langle \theta_0 , \V{X} \rangle - Y \bigr) 
\Bigr]^2 
\\ = \langle \theta_2 - \theta_0, \V{X} \rangle^4
+ 4 \langle \theta_2 - \theta_0, \V{X} \rangle^3 
\bigl( \langle \theta_0, \V{X} \rangle - Y \bigr) 
\\ + 4 \langle \theta_2 - \theta_0, \V{X} \rangle^2 
\bigl( \langle \theta_0, \V{X} \rangle - Y \bigr)^2 
\\ \leq 2 \langle \theta_2 - \theta_0, \V{X} \rangle^4 
+ 8\langle \theta_2 - \theta_0, \V{X} \rangle^2 
\bigl( \langle \theta_0, \V{X} \rangle - Y \bigr)^2.   
\end{multline*}

\begin{thm}
\label{thm7.3}
  Let us put
\begin{align*}
\wh{D} & = \frac{1}{n} \sum_{i=1}^n \lVert \V{X}_i \rVert^2\qquad \text{(let us remind that }D  = \B{E} \bigl[ \lVert \V{X} \rVert^2 \bigr] \text{ from }\eqref{eq:normd}),\\
B_1 & = 2 \B{E} 
\Bigl[ \lVert \V{X} \rVert^2 \bigl( \langle \theta_0, 
\V{X} \rangle - Y \bigr)^2 \Bigr],\\
\wh{B}_1 & = \frac{2}{n} \sum_{i=1}^n 
\Bigl[ \lVert \V{X}_i \rVert^2 \bigl(\langle \theta_0, 
\V{X}_i \rangle - Y_i \bigr)^2 \Bigr],\\
B_2 & = 2 \B{E} \Bigl[ \lVert \V{X} \rVert^4 \Bigr],\\
\wh{B}_2 & = \frac{2}{n} \sum_{i=1}^n \lVert \V{X}_i \rVert^4,\\
B_3 & = 40 \sup \Bigl\{ 
\B{E} \bigl[\langle u, \V{X} \rangle^2 \bigl( 
\langle \theta_0, \V{X} \rangle - Y \bigr)^2 \bigr] : 
u \in \B{R}^d, \lVert u \rVert = 1 \Bigr\},\\
\wh{B}_3 & = \sup \biggl\{ 
\frac{40}{n} \sum_{i=1}^n \langle u, \V{X}_i \rangle^2 
\bigl( \langle \theta_0, \V{X}_i \rangle - Y_i \bigr)^2 
: u \in \B{R}^d, \lVert u \rVert = 1 \Bigr\},\\
B_4 & = 10 \sup \Bigl\{ 
\B{E} \Bigl[ \langle u, \V{X} \rangle^4 \Bigr] : 
u \in \B{R}^d, \lVert u \rVert = 1 \Bigr\},\\
\wh{B}_4 & = \sup \biggl\{ 
\frac{10}{n} \sum_{i=1}^n \langle u , 
\V{X}_i \rangle^4 : u \in \B{R}^d, \lVert u \rVert = 1 \biggr\}.
\end{align*}
With probability at least $1 - \exp( - \eta)$, 
for any $\theta_2 \in \B{R}^d$, 
\begin{multline*}
n \B{E} \bigl[ W(\theta_2) \bigr] 
- \biggl[ n \alpha^2 (B_3 + \wh{B}_3) + \frac{\beta}{2} 
\biggr]  \lVert \theta_2 - \theta_0 \rVert^2 
\\ - n \alpha^2 (B_4 + \wh{B}_4) \lVert \theta_2 - \theta_0 \rVert^4 
\\ \leq \sum_{i=1}^n W_i(\theta_2) + \frac{ n \alpha}{\beta} 
(\wh{D} - D) + \frac{n \alpha^2}{\beta} (B_1 + \wh{B}_1) + 
\frac{n \alpha^2}{\beta^2} (B_2 + \wh{B}_2) + \eta.
\end{multline*}
\end{thm}

Let us now assume that $\theta_2 \in \V{\Theta}$ and let us use the fact 
that $\V{\Theta}$ is a convex set and that $\theta_0 = \arg \min_{
\theta \in \V{\Theta}} 
\V{R}(\theta) + \lam \lVert Q_\lam^{-1/2} \theta 
\rVert^2$. Introduce $\theta_{*} = \arg \min_{\theta \in \B{R}^d} \V{R}(\theta)
+ \lam \lVert Q_\lam^{-1/2} \theta \rVert^2$.
As we have 
$$
\V{R}(\th) + \lam \lVert Q_\lam^{-1/2} \th 
\rVert^2 = \|\th-\th_*\|^2 + \V{R}(\theta_*) + 
\lam \lVert Q_\lam^{-1/2} \theta_* \rVert^2,
$$
the vector $\theta_0$ is uniquely defined
as the projection of $\theta_{*}$ on $\V{\Theta}$ 
for the Euclidean distance, and for any $\theta_2 \in \V{\Theta}$ 
\begin{multline}
\alpha^{-1}\B{E} \bigl[ W(\theta_2) \bigr] + 
\lam \lVert Q_\lam^{-1/2} \theta_2 \rVert^2 
- \lam \lVert Q_\lam^{-1/2} \theta_0 \rVert^2 
\\ = \V{R}(\theta_2) 
- \V{R}( \theta_0) 
+ \lam \lVert Q_\lam^{-1/2} \theta_2 \rVert^2 
- \lam \lVert Q_\lam^{-1/2} \theta_0 \rVert^2 
\\ = \lVert \theta_2 - \theta_* \rVert^2 
- \lVert \theta_0 - \theta_* \rVert^2 
\\ = \lVert \theta_2 - \theta_0 \rVert^2 
+ 2 \langle \theta_2 - \theta_0, \theta_0 - \theta_* \rangle 
\geq \lVert \theta_2 - \theta_0 \rVert^2.  \label{eq:thconv}
\end{multline}

This and the inequality 
$$
\alpha^{-1}\sum_{i=1}^n W_i(\theta_1) + n \lam \lVert Q_\lam^{-1/2} \theta_1 \rVert^2 
- n \lam \lVert Q_\lam^{-1/2} \theta_0 \rVert^2 \leq 0
$$ 
leads to the following result.
\begin{thm}
\label{thm7.4}
With probability at least $1 - \exp( - \eta)$, 
\begin{multline*}
R(\hat{\th}) + \lambda \lVert \hat{\th} \rVert^2 
- \inf_{\theta \in \Theta} \bigl[ R(\theta) + \lambda \lVert \theta 
\rVert^2 \bigr]  
\\ = \alpha^{-1}\B{E} \bigl[ W(\theta_1) \bigr] + \lam 
\lVert Q_\lam^{-1/2} \theta_1 \rVert^2 - \lam \lVert Q_\lam^{-1/2}  
\theta_0 \rVert^2
\end{multline*}
is not greater than the smallest positive non degenerate root of the 
following polynomial equation as soon as it has one 
\begin{multline*}
\Bigl\{ 1 - \bigl[\alpha(B_3 + \wh{B}_3) 
+ \tfrac{\beta}{2n\alpha} 
\bigr] \Bigr\} x 
- \alpha(B_4+\wh{B}_4) x^2 
\\ = \frac{1}{\beta} 
\max(\wh{D} - D,0) + \frac{\alpha}{\beta}(B_1+\wh{B}_1) 
+ \frac{\alpha}{\beta^2}(B_2+\wh{B}_2) + \frac{\eta}{n \alpha}.
\end{multline*}
\end{thm}

\begin{proof}
Let us remark first that when the polynomial appearing
in the theorem has two distinct roots, they are of 
the same sign, due to the sign of its constant coefficient.
Let $\wh{\Omega}$ be the event of probability at least
$1 - \exp(- \eta)$ described in Theorem \thmref{thm7.3}.
For any realization of this event for which 
the polynomial described in Theorem \ref{thm7.4} 
does not have two distinct positive roots, 
the statement of Theorem \ref{thm7.4} is void, 
and therefore fulfilled. Let us consider now 
the case when the polynomial in question 
has two distinct positive roots $x_1 < x_2$. 
Consider in this case the random (trivially nonempty) closed convex set
$$
\wh{\Theta} = \bigl\{ \theta \in \Theta : R(\theta) 
+ \lam \lVert \theta \rVert^2 
\leq \inf_{\theta' \in \Theta} \bigl[ R(\theta') + \lam 
\lVert \theta' \rVert^2 \bigr] + \tfrac{x_1 + x_2}{2} \bigr\}.
$$

Let $\theta_3 \in \arg \min_{\theta \in \wh{\Theta}} r(\theta) 
+ \lam \lVert \theta \rVert^2$ and $\theta_4 \in 
\arg \min_{\theta \in \Theta} r(\theta) + \lam \lVert \theta
\rVert^2$. We see from Theorem \ref{thm7.3} 
that
\begin{equation}
\label{eq7.22}
R(\theta_3) + \lam \lVert \theta_3 \rVert^2 < R(\theta_0) 
+ \lam \lVert \theta_0 \rVert^2 + \frac{x_1 + x_2}{2},
\end{equation}
because it cannot be larger from the construction of $\wh{\Theta}$. 
On the other hand, since $\wh{\Theta} \subset \Theta$, 
the line segment $[\theta_3, \theta_4]$ 
is such that $[\theta_3, \theta_4] \cap \wh{\Theta} 
\subset \arg \min_{\theta \in \wh{\Theta}} r(\theta) 
+ \lam \lVert \theta \rVert^2$. We can therefore
apply equation \eqref{eq7.22} to any point of $[\theta_3, \theta_4]
\cap \wh{\Theta}$, which proves that $[\theta_3, \theta_4] \cap \wh{\Theta}$ is an open subset 
of $[\theta_3, \theta_4]$. But it is also a closed subset 
by construction, and therefore, as it is non empty and
$[\theta_3, \theta_4]$ is connected, it proves that 
$[\theta_3, \theta_4] \cap \wh{\Theta} = [\theta_3, \theta_4]$,
and thus that $\theta_4 \in \wh{\Theta}$. This can be applied
to any choice of $\theta_3 \in \arg \min_{\theta \in \wh{\Theta}} r(
\theta) + \lam \lVert \theta \rVert^2$
and $\theta_4 \in \arg \min_{\theta \in \Theta} r(\theta) 
+ \lam \lVert \theta \rVert^2$, 
proving that $\arg \min_{\theta \in \Theta} r(\theta) 
+ \lam \lVert \theta \rVert^2
\subset \arg \min_{\theta \in \wh{\Theta}} r(\theta) 
+ \lam \lVert \theta \rVert^2$ and therefore 
that any $\theta_4 \in \arg \min_{\theta \in \Theta} r(\theta) 
+ \lam \lVert \theta \rVert^2$ 
is such that 
$$
R(\theta_4) + \lam \lVert \theta_4 \rVert^2  
\leq \inf_{\theta \in \Theta} \big[ R(\theta) + \lam \lVert \theta 
\rVert^2 \big] + x_1.
$$
because the values between $x_1$ and $x_2$ are excluded by Theorem 
\ref{thm7.3}.
\end{proof}

The actual convergence speed of the least squares estimator $\hat{\th}$
on $\Theta$ will depend on the speed of convergence of the ``empirical
bounds'' $\wh{B}_k$ towards their expectations. We can rephrase 
the previous theorem in the following more practical way:
\begin{thm}
\label{thm6.5}
Let $\eta_0,\eta_1,\dots,\eta_5$ be positive real numbers.
With probability at least 
$$
1 - \B{P} ( \wh{D} > D + \eta_0)
- \sum_{k=1}^4 
\B{P}(\wh{B}_k - B_k> \eta_k) 
- \exp( - \eta_5), 
$$
$R(\hat{\th}) + \lam \lVert \hat{\th} \rVert^2 
- \inf_{\theta \in \Theta} \bigl[ 
R(\theta) + \lam \lVert \theta \rVert^2 \bigr] $ is smaller than
the smallest non degenerate positive root of
\begin{multline} \label{eq:deg}
\Bigl\{ 1 - \bigl[ \alpha(2B_3+\eta_3) + 
\tfrac{\beta}{2n\alpha} \bigr]  \Bigr\} x 
- \alpha(2B_4 + \eta_4)x^2 
\\ = \frac{\eta_0}{\beta} + \frac{\alpha}{\beta}
(2 B_1 + \eta_1) + \frac{\alpha}{\beta^2} 
(2 B_2 + \eta_2) + 
\frac{\eta_5}{n \alpha},
\end{multline}
where we can optimize the values of $\alpha > 0$ and 
$\beta > 0$, since this equation has non random coefficients.
For example, taking for simplicity 
\begin{align*}
\alpha & = \frac{1}{8B_3+4\eta_3},\\
\beta & = \frac{n \alpha}{2},
\end{align*}
we obtain
\begin{multline*}
x - \frac{2 B_4 + \eta_4}{4 B_3 + 2 \eta_3} x^2 
= \frac{16 \eta_0 (2 B_3 + \eta_3)}{n} 
+ \frac{8 B_1 + 4 \eta_1}{n} \\ + \frac{32(2B_3 + \eta_3)(2B_2 
+ \eta_2)}{n^2} + 
\frac{8\eta_5(2B_3+\eta_3)}{n}.
\end{multline*}
\end{thm}

\subsubsection{Proof of Theorem \ref{th:hfrlam}}

Let us now deduce Theorem \thmref{th:hfrlam} from Theorem \ref{thm6.5}.
Let us first remark that with probability at least $1 - \eps/2$ 
$$
\wh{D} \leq D + \sqrt{\frac{B_2}{\eps n}},
$$
because the variance of $\wh{D}$ is less than $\frac{B_2}{2n}$. 
For a given $\eps > 0$, let us 
take $\eta_0 = \sqrt{\frac{B_2}{\eps n}}$, 
$\eta_1 = B_1$, $\eta_2 = B_2$, $\eta_3 = B_3$ and 
$\eta_4 = B_4$. We get that $R_{\lambda}(\hat{\th}) 
- \inf_{\theta \in \Theta} R_{\lambda}(\theta)$ 
is smaller than the smallest positive non degenerate root
of 
$$
x - \frac{B_4}{2B_3}x^2 = \frac{48 B_3}{n} \sqrt{\frac{B_2}{n\eps}}
+ \frac{12 B_1}{n} + \frac{288 B_2 B_3}{n^2} + 
\frac{24 \log(3/\eps) B_3}{n},
$$
with probability at least
$$
1 - \frac{5\,\eps}{6} - \sum_{k=1}^4 
\B{P}(\wh{B}_k > B_k + \eta_k).
$$
According to the weak law of large numbers, there is $n_{\eps}$
such that for any $n \geq n_{\eps}$, 
$$
\sum_{k=1}^4 \B{P}(\wh{B}_k > B_k + \eta_k) \leq \eps/6.
$$
Thus, increasing $n_\eps$ and the constants 
to absorb the second order terms, we see that for some
$n_{\eps}$ and any $n \geq n_\eps$, 
with probability at least $1 - \eps$, 
the excess risk is less than the smallest positive 
root of
$$
x - \frac{B_4}{2B_3} x^2 = 
\frac{13 B_1}{n} + \frac{24 \log(3/\eps)B_3}{n}.
$$
Now, as soon as $ac < 1/4$, the smallest positive root 
of $x - a x^2 = c$ is $\frac{2c}{1 + \sqrt{1 - 4 ac}}$.
This means that for $n$ large enough, with probability 
at least $1 - \eps$, 
$$
R_{\lambda}(\hat{\th}) - \inf_{\theta} R_{\lambda}(\theta) 
\leq \frac{15 B_1}{n} + \frac{25 \log(3/\eps) B_3}{n},
$$
which is precisely the statement of Theorem \thmref{th:hfrlam},
up to some change of notation.

\subsubsection{Proof of Theorem \ref{th:ermom}}

Let us now weaken Theorem \ref{thm7.4} in order to make a more 
explicit non asymptotic result and obtain Theorem \ref{th:ermom}. From now 
on, we will assume that $\lam = 0$.
We start by giving bounds on the quantity defined in Theorem \ref{thm7.3} in terms of 
\[\bcR= \sup_{f\in\Span \{\vp_1,\dots,\vp_d\}-\{0\}} \fracc{\|f\|_\infty^2}{\E[f(X)]^2}.\]
Since we have 
  \[ 
  \lVert \V{X} \rVert^2
    = \lVert Q_\lam^{-1/2} X \rVert^2 \le d B,
    \]
we get
  \begin{align*}
  \wh{d} & = \frac{1}{n} \sum_{i=1}^n \lVert \V{X}_i \rVert^2 \le d B,\\
  B_1 & = 2 \B{E} \Bigl[ \lVert \V{X} \rVert^2 \bigl( \langle \theta_0, 
\V{X} \rangle - Y \bigr)^2 \Bigr] \le 2 d B \,R(f^*),\\
\wh{B}_1 & = \frac{2}{n} \sum_{i=1}^n 
\Bigl[ \lVert \V{X}_i \rVert^2 \bigl(\langle \theta_0, 
\V{X}_i \rangle - Y_i \bigr)^2 \Bigr] \le 2 d B \,r(f^*),\\
B_2 & = 2 \B{E} \Bigl[ \lVert \V{X} \rVert^4 \Bigr] \le 2 d^2 B^2,\\
\wh{B}_2 & = \frac{2}{n} \sum_{i=1}^n \lVert \V{X}_i \rVert^4 \le 2 d^2 B^2,\\
B_3 & = 40 \sup \Bigl\{ 
\B{E} \bigl[\langle u, \V{X} \rangle^2 \bigl( 
\langle \theta_0, \V{X} \rangle - Y \bigr)^2 \bigr] : 
u \in \B{R}^d, \lVert u \rVert = 1 \Bigr\} \le 40 \bcR \,R(f^*),\\
\wh{B}_3 & = \sup \biggl\{ 
\frac{40}{n} \sum_{i=1}^n \langle u, \V{X}_i \rangle^2 
\bigl( \langle \theta_0, \V{X}_i \rangle - Y_i \bigr)^2 
: u \in \B{R}^d, \lVert u \rVert = 1 \Bigr\} \le 40 \bcR \,r(f^*),\\
B_4 & = 10 \sup \Bigl\{ 
\B{E} \Bigl[ \langle u, \V{X} \rangle^4 \Bigr] : 
u \in \B{R}^d, \lVert u \rVert = 1 \Bigr\} \le 10 \bcR^2,\\
\wh{B}_4 & = \sup \biggl\{ 
\frac{10}{n} \sum_{i=1}^n \langle u , 
\V{X}_i \rangle^4 : u \in \B{R}^d, \lVert u \rVert = 1 \biggr\} \le 10 \bcR^2.
  \end{align*}
Let us put
  \[
  a_0= \frac{2 d B + 4 dB\alpha[R(f^*)+r(f^*)] + \eta}{\alpha n} + \frac{16B^2d^2}{\alpha n^2},
  \]
  \[
  a_1=3/4-40 \alpha \bcR [R(f^*)+r(f^*)],
  \]
and 
  \[
  a_2=20 \alpha \bcR^2.
  \]
Theorem \ref{thm7.4} applied with $\beta=n\alpha/2$ implies that
with probability at least $1-\eta$
the excess risk $R(\hferm) - R(f^*)$ is upper bounded by the smallest positive root
of $a_1 x - a_2 x^2 = a_0$
as soon as $a_1^2 > 4 a_0 a_2$.
In particular, setting $\eps=\exp(-\eta)$ when \eqref{eq:condr} holds, we have
  \[R(\hferm) - R(f^*) \le \frac{2a_0}{a_1+\sqrt{a_1^2 - 4 a_0 a_2}} \le \frac{2a_0}{a_1}.\] 
We conclude that
\begin{thm} \label{th:ermg}
For any $\alpha>0$ and $\eps>0$,
with probability at least $1-\eps$, if the inequality 
  	\begin{multline} \label{eq:condr}
   80 \Bigg( \frac{(2+4\alpha [R(f^*)+r(f^*)]) B d + \logeps}{n} + \bigg(\frac{4Bd}n\bigg)^2 
    \Bigg) \\< \bigg(\frac3{4\bcR} - 40 \alpha [R(f^*)+r(f^*)]\bigg)^2
	\end{multline}
holds, then we have
  \beglab{eq:ermg}
  R(\hferm) - R(f^*) \le \J \Bigg( \frac{(2+4\alpha [R(f^*)+r(f^*)]) B d + \logeps}{n} + \bigg(\frac{4Bd}n\bigg)^2 \Bigg),
  \endlab
where $\J = \fracr{8}{3 \alpha - 160 \alpha ^2\bcR [R(f^*)+r(f^*)]}$
\end{thm}

Now,  the Bienaym\'e-Chebyshev inequality implies
  \[
  \P\big( r(f^*) -R(f^*) \ge t \big) \le \frac{\E\big(r(f^*) -R(f^*)\big)^2}{t^2} 
  \le \E[Y-f^*(X)]^4/nt^2.
  \]
Under the finite moment assumption of Theorem \ref{th:ermom}, we obtain
that for any $\eps\ge 1/n$, with probability at least $1-\eps$,
  \[
  r(f^*) < R(f^*) + \sqrt{\E[Y-f^*(X)]^4}.
  \]
From Theorem \ref{th:ermg} and a union bound,
by taking 
  \[
  \alpha=\Big(80\bcR [2R(f^*)+\sqrt{\E[Y-f^*(X)]^4}\Big)^{-1},
  \]
we get that with probability $1-2\eps$,
  \beglabc{eq:ermom}
  R(\hferm) - R(f^*) \le \J_1 \bcR \Bigg( \frac{3 B d' + \log(\eps^{-1})}{n} + \bigg(\frac{4 B d'}n\bigg)^2 \Bigg),
  \endlabc
with $\J_1 = 640 \left( 2 R(f^*) + \sqrt{\E\bigl\{ 
[Y-f^*(X)]^4 \bigr\}}\, \right)$.
This concludes the proof of Theorem \ref{th:ermom}.

\begin{rmk} \label{rmk:degen}
Let us indicate now how to handle the case when $Q$ is degenerate. 
Let us consider the linear subspace $S$ of $\B{R}^d$ spanned by 
the eigenvectors of $Q$ corresponding to positive eigenvalues.
Then almost surely $\text{Span} \{ X_i, i=1, \dots, n \} 
\subset S$. Indeed for any $\theta$ in the kernel of $Q$,
$\B{E}\bigl( \langle \theta, X \rangle^2\bigr) = 0$ 
implies that $\langle \theta, X \rangle = 0$ almost
surely, and considering a basis of the kernel, we see
that $X \in S$ almost surely, $S$ being orthogonal 
to the kernel of $Q$. Thus we can restrict the problem
to $S$, as soon as we choose
$$
\hat{\th} \in \Span \bigl\{ X_1, \dots, X_n \bigr\}
\cap \arg \min_{\theta} \sum_{i=1}^n \bigl( \langle \theta, X_i \rangle - Y_i \bigr)^2,
$$
or equivalently with the notation $\Phi=(\vp_j(X_i))_{ 1\le i \le n, 1\le j \le d}$
and $\YY=[Y_j]_{j=1}^n$,
$$
\hat{\th} \in \text{im} \Phi^T \cap \arg \min_{\theta} \| \Phi \th - Y \|^2
$$
This proves that the results of this section apply to this special 
choice of the empirical least squares estimator. Since we have 
$\R^d=\text{ker}\, \Phi \oplus \text{im}\,\Phi^T$, this choice is unique.
\end{rmk}

\subsection{Proof of Theorem \ref{th:3.1}} \label{sec:p3.1}
    
We use a similar notation as in Section \ref{sec:permom}:
we write $X$ for $\vp(X)$.
Therefore, the function $f_\th$ maps an input $x$ to $\langle \theta, x \rangle$.
We consider the change of coordinates
$$
\V{X} = Q_\lam^{-1/2} X.
$$
Thus, from \eqref{eq:normd}, we have
  $\B{E} \bigl[ \lVert \V{X} \rVert^2 \bigr] =D.$
We will use
$$
\V{R}(\theta) = \B{E} \bigl[ (\langle \theta, \V{X} \rangle - Y)^2 \bigr],
$$
so that $
\V{R}(Q^{1/2} \theta) 
= \B{E} \bigl[(\langle \theta, X \rangle - Y)^2 \bigr]= R(f_\theta) . 
$
Let 
$$
\V{\Theta} = \bigl\{ Q_\lam^{1/2} \theta ; \theta \in \Theta \bigr\}.
$$

Consider
\begin{align*}
\theta_0 & = \arg \min_{\theta \in \V{\Theta}} \Big\{ \V{R}(\theta) 
  + \lam \|Q_\lam^{-1/2} \theta\|^2\Big\}.  
\end{align*}
We thus have $\thrid = Q_\lam^{-1/2} \theta_0$, and 
\begin{align*}
\sigma & = \sqrt{\B{E} \bigl[ \big(\langle \th_0, \V{X} \rangle - Y \big)^2\bigr]}, \\ 
\chi & = 
\sup_{u \in \B{R}^d} \frac{ \B{E} \bigl( \langle u, \V{X} \rangle^4 \bigr)^{1/2}}{ 
\B{E} \bigl( \langle u, \V{X} \rangle^2 \bigr)},\\ 
\kappa & = \frac{\B{E} \bigl( \lVert \V{X} \rVert^4 \bigr)^{1/2}}{\B{E} 
\bigl( \lVert \V{X} \rVert^2 \bigr)} 
= \frac{ \B{E} \bigl( \lVert \V{X} \rVert^4 \bigr)^{1/2}}{ 
D} ,\\ 
\kappa' & = \frac{\B{E} \bigl[ \big(\langle \th_0, \V{X} \rangle - Y \big)^4 \bigr]^{1/2}}{\sigma^2},\\
T & = \|\V{\Theta}\| = \max_{\th,\th'\in\V{\Theta}} \|\th-\th'\|.
\end{align*}

For $\alpha>0$, we introduce
\begin{align*}
J_i(\theta) & = \langle \theta , \V{X}_i \rangle - Y_i , 
& J(\theta) & = \langle \theta, \V{X} \rangle - Y \\
\bL_i(\theta) & = \alpha \bigl( \langle \theta , \V{X}_i \rangle - Y_i \bigr)^2, 
& \bL(\theta) & = \alpha \bigl( \langle \theta, \V{X} \rangle - Y \bigr)^2\\
W_i(\theta) & = \bL_i(\theta) - \bL_i(\theta_0), & 
W(\theta) & = \bL(\theta) - \bL(\theta_0),
\end{align*}
and
\begin{align*}
r'(\th,\th') = \lam (\|Q_\lam^{-1/2} \th\|^2-\|Q_\lam^{-1/2} \th'\|^2) 
  + \frac1{n\alpha} \sum_{i=1}^n \psi\big( \bL(\th)-\bL(\th') \big).
\end{align*}
Let $\hbth=Q_\lam^{1/2}\hth \in \V{\Theta}$. We have 
  \begarlab{eq:core}
  -r'(\th_0,\hbth)=r'(\hbth,\th_0) & \le 
  \und{\max}{\th_1\in\V{\Theta}} r'(\hbth,\th_1)
  & \le \gamma + \und{\max}{\th_1\in\V{\Theta}} r'(\th_0,\th_1),
  \endarlab
where $\gamma=\und{\max}{\th_1\in\V{\Theta}} r'(\hbth,\th_1)
    - \und{\inf}{\th\in\V{\Theta}} \und{\max}{\th_1\in\V{\Theta}} r'(\th,\th_1)$
is a quantity which can be made arbitrary small by choice of the estimator.
By using an upper bound $r'(\th_0,\th_1)$ that holds uniformly 
in $\th_1$, we will control both left and right hand sides of \eqref{eq:core}.

To achieve this, we will 
upper bound 
  \beglab{eq:rp}
  r'(\th_0,\th_1) = \lam (\|Q_\lam^{-1/2} \th_0\|^2-\|Q_\lam^{-1/2} \th_1\|^2) + \frac1{n\alpha} 
    \sum_{i=1}^n \psi\big[ - W_i(\th_1) \big]
  \endlab
by the expectation of
a distribution depending on $\th_1$ of a \emph{quantity that does not depend on $\th_1$},
and then use the PAC-Bayesian argument to control this expectation uniformly in $\th_1$.
The distribution depending on $\th_1$ should therefore be taken such that
for any $\th_1\in\V{\Theta}$, its Kullback-Leibler divergence with
respect to some fixed distribution is small (at least when $\th_1$ is close to $\th_0$).

Let us start with the following result.
\begin{lemma} \label{le:robust}
Let $f, g : \B{R} \rightarrow \B{R}$ be two Lebesgue  measurable functions 
such that $f(x) \leq g(x)$, $x \in \B{R}$. Let us 
assume that there exists $h\in\R$ such that $x\mapsto g(x)+h \frac{x^2}2$ is convex.
Then for any probability distribution $\mu$ on the real line, 
$$
f \biggl( \int x \mu(dx)  \biggr) 
\leq \int g(x) \mu(dx) + \min \Bigl\{ \sup f - \inf f, 
\frac{h}{2} \Var(\mu) \Bigr\}.
$$ 
\end{lemma}
\begin{proof}
Let us put $x_0 = \int x \mu(dx)$ The function 
$$
x \mapsto g(x) + \frac{h}{2}(x - x_0)^2
$$ is convex. Thus, by Jensen's inequality 
$$
f(x_0) \leq g(x_0) \leq \int \mu(dx) \biggl[ 
g(x) + \frac{h}{2} (x - x_0)^2 \biggr]  
= \int g(x) \mu(dx) + \frac{h}{2} \Var(\mu).
$$ 
On the other hand
\begin{multline*}
f(x_0) \leq \sup f \leq \sup f + \int \bigl[ g(x) - \inf f \bigr] 
\mu(dx) \\ = \int g(x) \mu(d x) + \sup f - \inf f. 
\end{multline*} 
The lemma is a combination of these two inequalities.
\end{proof}

The above lemma will be used with $f=g=\psi$, where $\psi$
is the increasing influence function 
$$
\psi(x) = \begin{cases}
- \log(2), & x \leq -1,\\
\log(1 + x + x^2/2), & -1 \leq x \leq 0,\\
- \log(1 - x + x^2/2), & 0 \leq x \leq 1,\\
\log(2), & x \geq 1.
\end{cases}
$$
Since we have for any $x \in \B{R}$ 
$$
- \log \biggl( 1 - x + \frac{x^2}{2} \biggr) = 
\log \biggl( \frac{1 + x + \frac{x^2}{2}}{1 + \frac{x^4}{4}} \biggr) 
< \log \biggl( 1 + x + \frac{x^2}{2} \biggr), 
$$
the function $\psi$ satisfies for any $x \in \B{R}$  
$$
- \log \biggl( 1 - x + \frac{x^2}{2} \biggr) < \psi(x) 
< \log \biggl(1 + x + \frac{x^2}{2} \biggr).
$$
Moreover
$$
\psi'(x) = \frac{1 - x}{1 - x + \frac{x^2}{2}}, \quad 
\psi''(x) = 
\frac{x(x-2)}{2 \bigl( 1 - x + \frac{x^2}{2} \bigr)^2} 
\geq - 2, \quad 0 \leq x \leq 1,
$$
showing (by symmetry) that 
the function $x\mapsto \psi(x) + 2 x^2$ is convex on the real line.

For any $\theta' \in \B{R}^d$ and $\beta>0$, we consider the Gaussian distribution with mena~$\theta'$ and covariance $\be^{-1} I$:
$$
\rho_{\theta'}(d \theta) = \left( \frac{\beta}{2 \pi}\right)^{d/2} \exp 
\left( - \frac{\beta}{2} \lVert \th  - \th' \rVert^2 \right) d \theta.
$$

From Lemmas \ref{le:7.2} and \ref{le:robust} (with $\mu$ the distribution of
 $- W_i(\theta) + \frac{\alpha \lVert \V{X}_i \rVert^2}{\beta}$ when $\th$ is drawn from $\rho_{\th_1}$ and for a fixed pair $(X_i,Y_i)$), we can see that 
\begin{align*}
\psi \bigl[ - W_i(\theta_1) 
\bigr] & = \psi \biggl\{ \tint \rho_{\theta_1}(d \theta) 
\biggl[ - W_i(\theta) + \frac{\alpha 
\lVert \V{X}_i \rVert^2}{\beta} \biggr] \biggr\}  \\ 
& \leq \int \rho_{\theta_1}(d \theta) 
\psi \biggl[ - W_i(\theta) + \frac{\alpha \lVert \V{X}_i \rVert^2}{\beta} 
\biggr]  \\ & \qquad + \min \Bigl\{ \log(4), 
\Var_{\rho_{\theta_1}} \bigl[\bL_i(\theta) \bigr] \Bigr\}.
\end{align*}
Let us compute
\begin{align}
\frac1{\alpha^2}\Var_{\rho_{\theta_1}} \bigl[ \bL_i(\theta) \bigr] & = \Var_{\rho_{\theta_1}} \bigl[ J^2_i(\theta) - J^2_i(\theta_1) \bigr] \notag \\
& = \int \rho_{\theta_1}(d \theta) \bigl[ J^2_i(\theta) - J^2_i(\theta_1) \bigr]^2  
- \frac{\lVert \V{X}_i \rVert^4}{\beta^2} \notag \\ 
& = \int \rho_{\theta_1}(d \theta) 
\Bigl[ \langle \theta - \theta_1, \V{X}_i \rangle^2 + 2 \langle \theta- 
\theta_1, \V{X}_i \rangle J_i (\theta_1) \Bigr]^2 - \frac{\lVert \V{X}_i \rVert^4}{
\beta^2} \notag \\ 
& = \frac{2 \lVert \V{X}_i \rVert^4}{\beta^2} + \frac{4 \bL_i(\theta_1) 
\lVert \V{X}_i \rVert^2}{\alpha \beta}.   \label{eq:varL}
\end{align}
Let $\xi\in(0,1)$. Now we can remark that
$$
\bL_i(\theta_1) \leq \frac{\bL_i(\theta)}{\xi} + \frac{\alpha \langle \theta - \theta_1, 
\V{X}_i \rangle^2}{1 - \xi}.
$$
We get 
\begin{multline*}
\min \Bigl\{ \log(4), \Var_{\rho_{\theta_1}} \bigl[ 
\bL_i(\theta) \bigr] \Bigr\} 
\\ = \min \Bigl\{ \log(4), \frac{4 \alpha \lVert \V{X}_i \rVert^2 \bL_i(\theta_1)}{\beta} 
+ \frac{2 \alpha^2 \lVert \V{X}_i \rVert^4}{\beta^2} \Bigr\}
\\ \shoveleft{\leq \int \rho_{\theta_1}(d \theta) \min 
\Bigl\{ \log(4),} \\ \shoveright{\frac{4 \alpha \lVert \V{X}_i \rVert^2 \bL_i(\theta)}{\beta
\xi} + \frac{2 \alpha^2 \lVert \V{X}_i \rVert^4}{\beta^2} + 
\frac{4 \alpha^2 \lVert \V{X}_i \rVert^2 \langle \theta - \theta_1, \V{X}_i  \rangle^2}{
\beta (1 - \xi)}  \Bigr\}} \\ 
\leq \int \rho_{\theta_1}(d \theta) \min \Bigl\{ 
\log(4), \frac{4 \alpha \lVert \V{X}_i \rVert^2 \bL_i(\theta)}{\beta \xi} 
+ \frac{2 \alpha^2 \lVert \V{X}_i \rVert^4}{\beta^2} 
\Bigr\} \\ + \min \Bigl\{ \log(4), \frac{4 \alpha^2 \lVert \V{X}_i \rVert^4}{
\beta^2(1 - \xi)} \Bigr\}.
\end{multline*}
Let us now put $a = \frac{3}{\log(4)} < 2.17$, 
$b = a + a^2 \log(4) < 8.7$ and let us remark 
that
\begin{multline*}
\min \bigl\{ \log(4), x \bigr\} + \min \bigl\{ \log(4), y \bigr\}
\\ \leq \log \bigl[  1 + a  \min \{\log(4), x \} \bigr]   + 
\log (1 + a y) \\ \leq \log \bigl( 1 + a x + 
b y \bigr), \qquad x, y \in \B{R}_+.
\end{multline*}
Thus 
\begin{multline*}
\min \Bigl\{ \log(4), \Var_{\rho_{\theta_1}} \bigl[ \bL_i(\theta) 
\bigr] \Bigr\} 
\\ \leq \int \rho_{\theta_1}(d \theta) 
\log \biggl[ 1 + \frac{4 a \alpha \lVert \V{X}_i \rVert^2 \bL_i(\theta)}{\beta 
\xi} + \frac{2  \alpha^2 \lVert \V{X}_i \rVert^4}{\beta^2} 
\biggl( a + \frac{2 b}{1-\xi} \biggr) \biggr].
\end{multline*}
We can then remark that 
\begin{multline*}
\psi(x) + \log(1 + y) = \log \bigl[ \exp[\psi(x)] + y \exp[\psi(x)] \bigr] \\ 
\leq \log \bigl[ \exp[\psi(x)]  + 2y \bigr] \leq \log \biggl( 
1 + x + \frac{x^2}{2} + 2 y \biggr), \qquad x \in \B{R}, y \in \B{R}_+.
\end{multline*}
Thus, putting $\ds c_0 = a + \frac{2b}{1-\xi}$, we get  
\beglab{eq:psiA}
\psi \bigl[ - W_i( \theta_1) \bigr] 
\leq \int \rho_{\theta_1}(d \theta) \log[ A_i(\th)],
\endlab
with 
\begin{multline*}
A_i(\th)=
1 - W_i(\theta) + \frac{\alpha \lVert \V{X}_i \rVert^2}{
\beta}  + \frac{1}{2} \biggl(-W_i(\theta) + 
\frac{\alpha \lVert \V{X}_i \rVert^2}{\beta} \biggr)^2 \\
+ \frac{8 a \alpha \lVert \V{X}_i \rVert^2 \bL_i(\theta)}{\beta \xi} 
+ \frac{4 c_0 \alpha^2 \lVert \V{X}_i \rVert^4}{\beta^2}.
\end{multline*}
Similarly, we define $A(\th)$ by replacing $(X_i,Y_i)$ by $(X,Y)$.
Since we have
  $$
  \E \exp\bigg(\sum_{i=1}^n\log[ A_i(\th) ] - n\log[ \E A(\th) ] \bigg)=1,
  $$
from the usual PAC-Bayesian argument, we have with probability at least
$1-\eps$, for any $\theta_1 \in \B{R}^d$,
  \begin{align*}
  \int \rho_{\theta_1}(d \theta) \bigg(\sum_{i=1}^n\log[ A_i(\th) ]\bigg)
   -n \int \rho_{\theta_1}(d \theta) \log[ A(\th) ] 
    & \le K(\rho_{\th_1}, \rho_{\th_0}) + \logeps\\
    & \le \frac{\beta \lVert \theta_1 - \theta_0 \rVert^2}{2}  + \logeps
  \end{align*}
From \eqref{eq:rp} and \eqref{eq:psiA}, with probability at least
$1-\eps$, for any $\theta_1 \in \B{R}^d$, we get
\begin{multline*}
r'(\theta_0, \theta_1) \leq \frac{1}{\alpha} 
\log \biggl\{ 
1  + \B{E} \biggl[ \int \rho_{\theta_1}(d \theta) 
\biggl( - W(\theta) + \frac{\alpha \lVert \V{X} \rVert^2}{\beta} 
\\ + \frac{1}{2} 
\biggl( - W(\theta) + \frac{\alpha \lVert \V{X} \rVert^2}{\beta} 
\biggr)^2 
+ \frac{ 8 a \alpha \lVert \V{X} \rVert^2 \bL(\theta)}{\beta \xi} 
+ \frac{4c_0\alpha^2\lVert \V{X} \rVert^4}{\beta^2} \biggr) \biggr] 
\biggr\} \\ 
+ \frac{\beta \lVert \theta_1 - \theta_0 \rVert^2}{2 n \alpha} 
+ \frac{\log(\epsilon^{-1})}{n \alpha} + \lam (\|Q_\lam^{-1/2} \th_0\|^2-\|Q_\lam^{-1/2} \th_1\|^2) .
\end{multline*}
Now from \eqref{eq:varL} and 
  $\frac{\alpha \|\V{X}\|^2}\beta = - \bL(\th_1) + \int\rho_{\theta_1}(d \theta) \bL(\th)$, we have 
\begin{multline*}
\int \rho_{\theta_1}(d \theta) 
\biggl( - W(\theta) + \frac{\alpha \lVert \V{X} \rVert^2}{\beta} \biggr)^2 
= \Var_{\rho_{\theta_1}} \bigl[ \bL(\theta) \bigr] + W(\theta_1)^2 
\\ = W(\theta_1)^2 + 
\frac{4 \alpha \bL(\theta_1) \lVert \V{X} \rVert^2}{\beta} 
+ \frac{2 \alpha^2 \lVert \V{X} \rVert^4}{\beta^2}.
\end{multline*}
\begin{prop}
With probability at least $1 - \epsilon$, 
for any $\theta_1 \in \B{R}^d$, 
\begin{align*}
r'(\theta_0, \theta_1) & \le \frac{1}{\alpha} \log \biggl\{ 1 +  
\B{E} \biggl[ - W(\theta_1)  
+ \frac{W(\theta_1)^2}{2} 
+ \frac{\bigl(2+8a/\xi\bigr) \alpha \lVert \V{X} \rVert^2 \bL(\theta_1)}{\beta} \\ 
& \qquad\qquad + \frac{ \bigl( 1 + 8 a/\xi + 4 c_0 \bigr) \alpha^2 \lVert \V{X} \rVert^4}{\beta^2} 
\biggr] \biggr\} + 
\frac{\beta \lVert \theta_1 - \theta_0 \rVert^2}{2 n \alpha} 
+ \frac{\log(\epsilon^{-1})}{n \alpha} \\
& \qquad\qquad\qquad\qquad\qquad + \lam (\|Q_\lam^{-1/2} \th_0\|^2-\|Q_\lam^{-1/2} \th_1\|^2) \\ 
& \le \B{E} \biggl[ J(\theta_0)^2-J(\theta_1)^2 
+ \frac{1}{2\alpha} W(\theta_1)^2 
+ \frac{(2 + 8a/\xi) \lVert \V{X} \rVert^2 \bL(\theta_1)}{\beta} \\ 
& \qquad\qquad 
  + \frac{(1 + 8a/\xi + 4 c_0) \alpha \lVert \V{X} \rVert^4}{\beta^2} \biggr]   
  + \frac{\beta \lVert \theta_1 - \theta_0 \rVert^2}{2 n \alpha} 
+ \frac{\log(\epsilon^{-1})}{n \alpha}\\
& \qquad\qquad\qquad\qquad\qquad  + \lam (\|Q_\lam^{-1/2} \th_0\|^2-\|Q_\lam^{-1/2} \th_1\|^2) .
\end{align*}
\end{prop}

By using the triangular inequality and Cauchy-Scwarz's inequality, 
we get
\begin{align*}
\frac1{\alpha^2}\B{E} \bigl[ W(\theta_1)^2 \bigr] 
& = \B{E} \Bigl\{ \bigl[ \langle \theta_1 - \theta_0, \V{X} \rangle^2 + 2 
\langle \theta_1 - \theta_0, \V{X} \rangle J(\theta_0) \bigr]^2  \Bigr\} \\ 
& \leq \Bigl\{ \B{E} \bigl[ \langle \theta_1 - \theta_0, \V{X} \rangle^4 
\bigr]^{1/2} + 2 \B{E} \bigl[ \langle \theta_1 - \theta_0, \V{X} \rangle^4 
\bigr]^{1/4} \B{E} \bigl[ J(\theta_0)^4 \bigr]^{1/4} \Bigr\}^2 \\
& \leq \Biggl\{ \chi \lVert \theta_1 - \theta_0 \rVert^2 \B{E} \biggl[ 
  \biggl\langle \frac{\theta_1 - \theta_0}{\lVert \theta_1 - \theta_0 \rVert}, \V{X} \biggr\rangle^2 
\biggr] \\
& \qquad\qquad\qquad + 2 \lVert \theta_1 - \theta_0 \rVert \sigma \sqrt{\kappa' \chi}
\sqrt{\B{E} \biggl[ 
  \biggl\langle \frac{\theta_1 - \theta_0}{\lVert \theta_1 - \theta_0 \rVert}, \V{X} \biggr\rangle^2 
\biggr]} \Biggr\}^2 \\
& \le \frac{\chi q_{\max}}{q_{\max}+\lam} \lVert \theta_1 - \theta_0 \rVert^2 \biggl\{ 
  \lVert \theta_1 - \theta_0 \rVert \sqrt{\frac{\chi q_{\max}}{q_{\max}+\lam}}
  + 2 \sigma \sqrt{\kappa'} \biggr\}^2,
\end{align*}
and
\begin{align*}
\frac1\alpha\B{E} \bigl[ \lVert \V{X} \rVert^2 \bL(\theta_1) \bigr] 
& = \B{E} \Bigl\{ \bigl[ \lVert \V{X} \rVert \langle \theta_1 - \theta_0, \V{X} \rangle 
+ \lVert \V{X} \rVert J(\theta_0) \bigr]^2 \Bigr\} \\ 
& \le \B{E} \bigl[ \lVert \V{X} \rVert^4 \bigr]^{1/2} \Bigr\{ 
\B{E} \bigl[ \langle \theta_1 - \theta_0, \V{X} \rangle^4 \bigr]^{1/4} 
+ \B{E} \bigl[ J(\theta_0)^4 \bigr]^{1/4} \Bigr\}^2\\
& \le \kappa D \biggl\{ 
  \lVert \theta_1 - \theta_0 \rVert \sqrt{\frac{\chi q_{\max}}{q_{\max}+\lam}}
  + 2 \sigma \sqrt{\kappa'} \biggr\}^2,
\end{align*}

Let us put 
\begin{align*}
\wt{R}(\th) & =\V{R}(\th) + \lam \| Q_\lam^{-1/2} \th \|^2,\\
c_1 & = 4(2+8a/\xi), \\
c_2 & = 4(1 + 8a/\xi+4c_0), \\
\delta & = \frac{c_1 \kappa \kappa' D \sigma^2}{n}  
+ \frac{2 \chi \bigl( \frac{\log(\epsilon^{-1})}{n} + 
\frac{c_2 \kappa^2 D^2}{n^2} \bigr) \bigl[ 2 \sqrt{\kappa'}\sigma 
+ \lVert \V{\Theta} \rVert \sqrt{\chi} \bigr]^2}{ 1 - \frac{4 c_1\kappa \chi D}{n} }.
\end{align*}

We have proved 
the following result.
\begin{prop}
With probability at least $1 - \epsilon$, for any $\theta_1 \in \B{R}^d$, 
\begin{multline*}
r'(\theta_0, \theta_1) \leq 
\wt{R}(\theta_0) - \wt{R}(\theta_1) + \frac{\alpha}{2} \chi \lVert \theta_1 - \theta_0 
\rVert^2 \bigl[ 2 \sqrt{\kappa'} \sigma + \lVert \theta_1 - \theta_0 \rVert \sqrt{\chi}
 \bigr]^2  \\
+ \frac{c_1 \alpha}{4\beta} \kappa D \bigl[ 
\sqrt{\kappa'} \sigma + \lVert \theta_1 - \theta_0 \rVert \sqrt{\chi} \bigr]^2 
+ \frac{c_2 \alpha \kappa^2 D^2}{4\beta^2} \\ 
+ \frac{\beta \lVert \theta_1 - \theta_0 \rVert^2}{2 n \alpha} 
+ \frac{\log(\epsilon^{-1})}{n \alpha}.
\end{multline*}
\end{prop}
Let us assume from now on that $\theta_1 \in \V{\Theta}$, our convex bounded parameter set.
In this case, as seen in \eqref{eq:thconv}, we have
$\lVert \theta_0 - \theta_1 \rVert^2 \leq \wt{R}(\theta_1) - \wt{R}(\theta_0)$.
We can also use the fact that 
$$
\bigl[ \sqrt{\kappa'} \sigma + \lVert \theta_1 - \theta_0 \rVert \sqrt{\chi} \bigr]^2
\leq 2 \kappa' \sigma^2 + 2 \chi \lVert \theta_1 - \theta_0 \rVert^2.
$$ 
We deduce from these remarks that with 
probability at least $1 - \epsilon$, 
\begin{multline*}
r'(\theta_0, \theta_1) \leq 
\biggl\{ -1 + \frac{\alpha \chi}{2} \bigl[ 2 \sqrt{\kappa'} \sigma + \lVert \V{\Theta} \rVert \sqrt{\chi} 
\bigr]^2  + \frac{\beta}{2 n \alpha} + \frac{c_1 \alpha \kappa D \chi}{2\beta} 
\biggr\} \bigl[ \wt{R}(\theta_1) - \wt{R}(\theta_0) \bigr] 
\\ + \frac{c_1 \alpha \kappa D \kappa' \sigma^2}{2\beta} + 
\frac{c_2 \alpha \kappa^2 D^2}{4\beta^2} + \frac{\log(\epsilon^{-1})}{n \alpha}. 
\end{multline*}
Let us assume that $n > 4c_1 \kappa \chi D$
and let us choose
\begin{align*}
\beta & = \frac{n \alpha}{2}, \\ 
\alpha & = \frac{1}{2 \chi \bigl[ 2 \sqrt{\kappa'} \sigma
+ \lVert \V{\Theta} \rVert \sqrt{\chi} \bigr]^2} \biggl(1  - \frac{4c_1\kappa \chi D}{n} 
\biggr), 
\end{align*}
to get
$$
r'(\theta_0, \theta_1) \leq - \frac{\wt{R}(\theta_1) - \wt{R}(\theta_0)}{2} + \delta.
$$
Plugging this into \eqref{eq:core}, we get
\begin{align*}
\frac{\wt{R}(\hbth) - \wt{R}(\theta_0)}{2} - \delta \le r'(\hbth, \theta_0) 
\le \und{\max}{\th_1\in\V{\Theta}} \bigg( \frac{\wt{R}(\th_0) - \wt{R}(\th_1)}{2} \bigg)
+ \gamma + \delta =  \gamma + \delta,
\end{align*}
hence
\begin{align*}
\wt{R}(\hbth) - \wt{R}(\theta_0) \le 2 \gamma + 4 \delta.
\end{align*}
Computing the numerical values of the constants when $\xi = 0.8$ gives 
$c_1 < 95$ and $c_2 < 1511$.

\subsection{Proof of Theorem \ref{th:gen}} \label{sec:pgeneric}

We use the standard way of obtaining PAC bounds through 
upper bounds on Laplace transform of appropriate random variables.
This argument is synthetized in the following result.

\begin{lemma} \label{le:pac}
For any $\eps>0$ and any real-valued random variable $V$ such that $\E \bigl[ 
\exp(V) \bigr] \le 1$,
with probability at least $1-\eps$, we have
    \[
    V \le \logeps.
    \]
\end{lemma}
\begin{multline*}
\text{Let }  V_1(\hf) = \int \big[ \La(\hf,f) + \gas \bR(f) \big] \pis_{-\gas \bR}(df) - \ga \bR(\hf) \\
		- \I^*(\gas) + \I(\ga)
	    + \log \bigg(\int \exp \bigl[ -\hcE(f) \bigr] \pi(df)\bigg) 
	    - \log \biggl[  \frac{d \rho}{d \hpi}\bigl(\hf\bigr) \biggr],
    \end{multline*}
    \[
\text{and }    V_2 = - \log \bigg(\int \exp\bigl[ -\hcE(f)
\bigr]  \pi(df)\bigg) + \log \bigg(\int \exp \bigl[ -\cEb(f)\bigr] \pi(df)\bigg)
    \]
To prove the theorem, according to Lemma \ref{le:pac}, it suffices to prove that
$$
\E \Bigl\{ \tint \exp \bigl[ V_1(\hf) \bigr]  \rho(d \hf) 
\Bigr\}  \le 1
\quad \text{and} \quad \E \Bigl[ \tint \exp (V_2 ) \rho(d \hf) \Bigr]  \le 1.
$$
These two inequalities are proved in the following two sections.

\subsubsection{Proof of $\E \Bigl\{ \int \exp \bigl[ V_1(\hf) 
\bigr]  \rho(d \hf) \Bigr\} \le 1$}

From Jensen's inequality, we have
    \begin{align*}
    \int & \big[ \La(\hf,f) + \gas \bR(f) \big] \pis_{-\gas \bR}(df)\\
        & = \int \big[ \hL(\hf,f) + \gas \bR(f) \big]  \pis_{-\gas \bR}(df) + \int \big[ \La(\hf,f) - \hL(\hf,f) \big] \pis_{-\gas \bR}(df)\\
        & \le \int \big[ \hL(\hf,f) + \gas \bR(f) \big] \pis_{-\gas \bR}(df) + \log \int \exp \bigl[ \La(\hf,f) - \hL(\hf,f) \bigr] \pis_{-\gas \bR}(df).
    \end{align*}

From Jensen's inequality again, 
    \begin{align*}
    - \hcE(\hf) & = - \log \int \exp \bigl[ \hL(\hf,f) \bigr] \pis(df)\\
    & = - \log \int \exp \bigl[ \hL(\hf,f)+\gas \bR(f) \bigr] \pis_{-\gas \bR}(df)
        - \log \int \exp \bigl[ -\gas \bR(f)\bigr]  \pis(df)\\
    & \le - \int [ \hL(\hf,f)+\gas \bR(f) ] \pis_{-\gas \bR}(df) + \I^*(\gas).
    \end{align*}
    
From the two previous inequalities, we get
    \begin{align*}
    V_1(\hf) & \le \int \big[ \hL(\hf,f) + \gas \bR(f) \big]  \pis_{-\gas \bR}(df) \\ 
& \qquad + \log \int \exp \bigl[ \La(\hf,f)- \hL(\hf,f) \bigr]  
\pis(df) - \ga \bR(\hf) \\
	& \qquad - \I^*(\gas) + \I(\ga)
	    + \log \bigg(\int \exp \bigl[ -\hcE(f) \bigr] \pi(df)\bigg) 
	    - \log \biggl[  \frac{d \rho}{d \hpi}(\hf) \biggr],\\
    & = \int \big[ \hL(\hf,f) + \gas \bR(f) \big] \pis_{-\gas \bR}(df) 
\\ & \qquad + \log \int \exp \bigl[ \La(\hf,f)- \hL(\hf,f)  \bigr]  \pis(df) - \ga \bR(\hf) \\
	& \qquad - \I^*(\gas) + \I(\ga) - \hcE(\hf) - \log\biggl[ 
\frac{d \rho}{d \pi}(\hf) \biggr],\\
    & \le \log \int \exp \bigl[ \La(\hf,f)- \hL(\hf,f) \bigr]  
\pis_{-\gas \bR}(df)(df) \\ & \qquad - \ga \bR(\hf) + \I(\ga) - \log\biggl[ 
\frac{d \rho}{d \pi}(\hf) \biggr]\\
    & = \log \int \exp \bigl[ \La(\hf,f)- \hL(\hf,f) \bigr]  \pis_{-\gas \bR}(df) + \log \biggl[ \frac{d \pi_{-\ga \bR}}{d \rho}(\hf) \biggr],
	\end{align*}
hence, by using Fubini's inequality and the equality 
\begin{multline*}
{} \hfill \E \Bigl\{ \exp \bigl[ - \hL(\hf,f) \bigr] \Bigr\}  = \exp \bigl[ -\La(\hf,f)
\bigr], \hfill {} \\ 
\shoveleft{\hspace{-2ex}\text{we obtain } \E \int \exp \bigl[ V_1(\hf) \bigr]  
\rho(\hf)}  \\ \le \E \int \bigg(\int \exp \bigl[ \La(\hf,f) - \hL(\hf,f)\bigr]  
\pis_{-\gas \bR}(df) \bigg) \pi_{-\ga \bR}(d\hf)\\
     =  \int \bigg(\int \E \exp \bigl[ \La(\hf,f) - \hL(\hf,f) \bigr] 
\pis_{-\gas \bR}(df) \bigg) \pi_{-\ga \bR}(d\hf)
     = 1. 
	\end{multline*}
	
\subsubsection{Proof of $\,\E \Bigl[ \int \exp (V_2) \rho(d \hf) \Bigr]  \le 1$}

It relies on the following result.

\begin{lemma} \label{le:concpart}
Let $\W$ be a real-valued measurable function defined on a product space $\A_1\times\A_2$ and
let $\mu_1$ and $\mu_2$ be probability distributions on respectively $\A_1$ and $\A_2$.
\bi
\item if $\expec{a_1}{\mu_1} \Bigl\{ 
 \log \Bigl[ \expec{a_2}{\mu_2} \bigl\{ \exp \bigl[ -\W(a_1,a_2) \bigr] 
\bigr\} \Bigr] \Bigr\}  < +\infty$, then we have
\begin{multline*}
        - \expec{a_1}{\mu_1} \Bigl\{ 
\log \Bigl[ \expec{a_2}{\mu_2} \bigl\{ \exp \bigl[ -\W(a_1,a_2)\bigr] \bigr\}
\Bigr] \Bigr\}  \\ 
\le - \log \Bigl\{ \expec{a_2}{\mu_2} \Bigl[ \exp \bigl[ 
-\expec{a_1}{\mu_1} \W(a_1,a_2) \bigr] \Bigr] \Bigr\}.
\end{multline*}
\item if $\W>0$ on $\A_1\times\A_2$ and $\expec{a_2}{\mu_2} 
\Bigl\{ \expec{a_1}{\mu_1} \bigl[ \W(a_1,a_2) \bigr]^{-1} 
\Bigr\}^{-1} < +\infty$, then 
$$
\expec{a_1}{\mu_1} \Bigl\{ \expec{a_2}{\mu_2} \Bigl[ \W(a_1,a_2)^{-1} 
\Bigr]^{-1} \Bigr\}  \le \expec{a_2}{\mu_2} \Bigl\{ \expec{a_1}{\mu_1} 
\bigl[ \W(a_1,a_2) \bigr]^{-1} \Bigr\}^{-1}.
$$

\ei
\end{lemma}
\noindent{\sc Proof.}
\bi
\item 
Let $\A$ be a measurable space and $\M$ denote the set of probability distributions on $\A$.
The Kullback-Leibler divergence between a distribution $\rho$ and a distribution~$\mu$ is
\beglab{eq:kl}
K(\rho,\mu) \eqdef \begin{cases}\ds  
\undc{\E}{a\sim\rho} \log \biggl[ \frac{d \rho}{d \mu}(a) \biggr] 
& \text{if } \rho \ll \mu,\\
\ds + \infty & \text{otherwise,}
\end{cases} 
\endlab
where $\ds \frac{d \rho}{d \mu}$ denotes as usual 
the density of $\rho$ $\wrt$ $\mu$.
The Kullback-Leibler divergence satisfies the duality formula (see, e.g., 
\cite[page 159]{Cat01}): for any 
real-valued measurable function $h$ defined on $\A$,
\beglab{eq:legendre}
\und{\inf}{\rho\in\M} \big\{ \undc{\E}{a\sim\rho} h(a) +
K(\rho,\mu) \big\} = -\log \expec{a}{\mu} \Bigl\{ \exp \bigl[ -h(a) \bigr] \Bigr\}.
\endlab
By using twice \eqref{eq:legendre} and Fubini's theorem, we have
        \begin{align*}
        - \expec{a_1}{\mu_1} \Bigl\{ \log  \Bigl\{ & 
\expec{a_2}{\mu_2} \Bigl[ \exp \bigl[
-\W(a_1,a_2) \bigr]\Bigr] \Bigr\} \Bigr\}\\
                & = \expec{a_1}{\mu_1} \Bigl\{ 
\und{\inf}{\rho} \big\{ \expec{a_2}{\rho}  \bigl[ \W(a_1,a_2) \bigr]  + K(\rho,\mu_2) \big\} \Bigr\} \\
        & \le \und{\inf}{\rho} \Bigl\{ \expec{a_1}{\mu_1} 
\Bigl[ \expec{a_2}{\rho} \bigl[ \W(a_1,a_2) \bigr] + K(\rho,\mu_2) 
\Bigr] \Bigr\} \\
        & = - \log \Bigl\{ \expec{a_2}{\mu_2} 
\Bigl[ \exp \bigl\{ -\expec{a_1}{\mu_1} \bigl[ \W(a_1,a_2) \bigr] 
\bigr\} \Bigr] \Bigr\}.
        \end{align*}
\item By using twice \eqref{eq:legendre} and the first assertion of Lemma \ref{le:concpart}, we have
        \begin{multline*}
\expec{a_1}{\mu_1} \Bigl\{ \expec{a_2}{\mu_2} \Bigl[ \W(a_1,a_2)^{-1} 
\Bigr]^{-1} \Bigr\} 
           \\ = \expec{a_1}{\mu_1}\Bigl\{ 
\exp \Bigl\{ -\log \Bigl[ \expec{a_2}{\mu_2} \bigl\{ 
\exp \bigl[-\log \W(a_1,a_2)
\bigr] \bigr\} \Bigr] \Bigr\} \Bigr\} \\
= \expec{a_1}{\mu_1} \Bigl\{ \exp \Bigl\{ \inf_\rho \Bigl[ 
\expec{a_2}{\rho}  \bigl\{ \log \bigl[ \W(a_1,a_2) \bigr] 
\bigr\}  + K(\rho,\mu_2) \Bigr] \Bigr\} \Bigr\}  \\
           \le \inf_\rho \Bigl\{ 
\exp \bigl[ K(\rho,\mu_2) \bigr]  \expec{a_1}{\mu_1} 
\Bigl\{ \exp \Bigl\{  \expec{a_2}{\rho} 
\Bigl[ \log \bigl[ \W(a_1,a_2) \bigr] \Bigr] \Bigr\} \Bigr\} \\
           \le \inf_\rho \Bigl\{ \exp \bigl[  
K(\rho,\mu_2)\bigr]  \exp \Bigl\{ \expec{a_2}{\rho} 
\Bigl\{ \log \Bigl[ \expec{a_1}{\mu_1} \bigl[ \W(a_1,a_2) \bigr] 
\Bigr] \Bigr\} \Bigr\} \\
= \exp \Bigl\{ \inf_\rho \Bigl\{ 
\expec{a_2}{\rho} \Bigl[ \log \bigl\{ \expec{a_1}{\mu_1} 
\bigl[ \W(a_1,a_2) \bigr] \bigr\} \Bigr]  +K(\rho,\mu_2) \Bigr\} \Bigr\} \\
= \exp \Bigl\{ - \log \Bigl\{ 
\expec{a_2}{\mu_2} \Bigl\{ \exp \Bigl[ - \log 
\bigl\{ \expec{a_1}{\mu_1} \bigl[ \W(a_1,a_2) \bigr] \bigr\} \Bigr] \Bigr\} 
\Bigr\} \Bigr\} \\
= \expec{a_2}{\mu_2} \Bigl\{ \expec{a_1}{\mu_1} \bigl[ \W(a_1,a_2) 
\bigr]^{-1} \Bigr\}^{-1}. \quad \square
\end{multline*}
\ei

From Lemma \ref{le:concpart} and Fubini's theorem, since $V_2$ does not depend on $\hf$, we have
\begin{multline*}
\E \Bigl[ \tint  \exp (V_2) \rho(d \hf) 
\Bigr]    = \E \bigl[ \exp (V_2) \bigr] \\
= \tint \exp \bigl[-\cEb(f) \bigr]  \pi(df) 
\E \Bigl\{ \Bigl[ \tint \exp \bigl[ -\hcE(f) \bigr] \pi(df) 
\Bigr]^{-1} \Bigr\} \\
\le \tint \exp \bigl[ -\cEb(f) \bigr]  \pi(df) 
\Bigl\{  \tint \E \Bigl[ \exp \bigl[ \hcE(f) \bigr] \Bigr]^{-1} 
\pi(df) \Bigr\}^{-1}\\
= \tint \exp \bigl[ -\cEb(f) \bigr] \pi(df) \Bigl\{ 
\tint \E \Bigl[ \int \exp \bigl[ \hL(f,f') \bigr] \pis(df') \Bigr]^{-1}  
\pi(df) \Bigr\}^{-1}\\
     = \tint \exp \bigl[ -\cEb(f) \bigr] \pi(df) \Bigl\{ 
 \tint \Bigl[ \tint \exp \bigl[ \Lb(f,f') \bigr]  \pis(df') 
\Bigr]^{-1}  \pi(df) \Bigr\}^{-1}
= 1.
\end{multline*}
This concludes the proof that
for any $\ga \ge 0$, $\gas \ge 0$ and $\eps>0$, with probability 
(with respect to the distribution $P^{\otimes n} \rho$ generating
the observations $Z_1,\dots,Z_n$ and the randomized prediction function $\hf$) 
at least $1-2\eps$:
\[
V_1(\hf)+V_2 \le 2 \logeps.
\]

\subsection{Proof of Lemma \ref{le:complexity}} \label{sec:complexity}

Let us look at $\cF$ from the point of view of $f^*$. Precisely 
let $\S_{\R^d}(O,1)$ be the sphere of $\R^d$ centered at the origin and with radius $1$
and 
	\[
	\S=\big\{ \sum_{j=1}^d \th_j \vp_j ; (\th_1,\dots,\th_d) \in \S_{\R^d}(O,1) \big\}.
	\] 
Introduce 
	$$
	\Omega = \big\{ \sigmb \in \S; \exists u>0 \text{ s.t. } f^* + u \sigmb \in \cF \big\}.
	$$
For any $\sigmb \in \Omega$, let $u_\sigmb = \sup\{u>0:f^*+ u \sigmb \in \cF \}$.
Since $\pi$ is the uniform distribution on the convex set $\cF$ (i.e., the one coming from the uniform distribution on $\cC$), 
we have
\begin{multline*}
	\int \exp \bigl\{ -\alpha[R(f)-R(f^*)] 
\bigr\}  \pi(df) \\ = \int_{\sigmb\in\Omega} \int_0^{u_\sigmb} 
\exp \bigl\{ -\alpha [R(f^*+u\sigmb)-R(f^*)] \bigr\}  
u^{d-1} du d\sigmb.
\end{multline*}
Let $c_\sigmb=\E [\sigmb(X) \ela_Y'(f^*(X))]$
and $a_\sigmb = \E \bigl[ \sigmb^2(X) \bigr]$. 
Since 
$$
f^*\in\argmin_{f\in\cF} \E \bigl\{ \ela_Y \bigl[ f(X) 
\bigr] \bigr\}, 
$$
we have $c_\sigmb\ge 0$ (and $c_\sigmb=0$ if both $-\sigmb$ and $\sigmb$ belong to $\Omega$).
Moreover from Taylor's expansion, 
$$
\frac{b_1 a_\sigmb u^2}{2} \le 
	R(f^*+u\sigmb)-R(f^*) - u c_\sigmb 
		\le \frac{ b_2 a_\sigmb u^2}{2}.
$$
Introduce 
	\[
	\psi_\sigmb = \frac{\int_0^{u_\sigmb} \exp \bigl\{ 
-\alpha[ u c_\sigmb +\demi b_1 a_\sigmb u^2]\bigr\} u^{d-1} du} 
{\int_0^{u_\sigmb} \exp \bigl\{ -\be[ u c_\sigmb+\demi b_2 a_\sigmb u^2]
\bigr\}  u^{d-1} du}.
	\]
For any $0<\alpha<\be$, we have
	\[
	 \frac{\int \exp \bigl\{ -\alpha[R(f)-R(f^*)] \bigr\}  \pi(df)}{
\int \exp \bigl\{ -\be[R(f)-R(f^*)]\bigr\}  \pi(df)} 
		\le \und{\inf}{\sigmb\in\S} \psi_\sigmb.
	\]	
For any $\zeta> 1$, by a change of variable, 
	\begin{align*}
	 \psi_\sigmb & < \zeta^d
		\frac{\int_0^{u_\sigmb} \exp \bigl\{ -\alpha[ \zeta u c_\sigmb +\demi b_1 a_\sigmb \zeta^2 u^2] \bigr\} u^{d-1} du}
		{\int_0^{u_\sigmb} \exp \bigl\{ -\be[ u c_\sigmb +\demi b_2 a_\sigmb u^2]
\bigr\} u^{d-1} du}\\
	 & \le \zeta^d \und{\sup}{u>0}
		\exp \bigl\{ \be[ u c_\sigmb + \tfrac{1}{2} b_2 a_\sigmb u^2]
		 -\alpha[ \zeta u c_\sigmb + \tfrac{1}{2} b_1 a_\sigmb \zeta^2 u^2]\bigr\}.
	\end{align*}
By taking $\zeta=\sqrt{(b_2\be)/(b_1\alpha)}$ when $c_\sigmb=0$ and 
$\zeta=\sqrt{(b_2\be)/(b_1\alpha)} \vee (\be/\alpha)$ otherwise,
we obtain $\psi_\sigmb < \zeta^d$, hence
$$
	\log \bigg( \frac{\int \exp \bigl\{ -\alpha[R(f)-R(f^*)]
\bigr\} \pi(df)}{\int \exp \bigl\{ -\be[R(f)-R(f^*)]
\bigr\}  \pi(df)} \bigg)
 \le \begin{cases}
\ds 		\frac{d}{2} \log\big(\frac{b_2\be}{b_1\alpha}\big) 
 \text{ when } \sup_{\sigmb\in\Omega} c_\sigmb=0,\\[2ex] 
\ds d \log\big(\sqrt{\frac{b_2\be}{b_1\alpha}} 
\vee \frac{\be}{\alpha} \big)  \text{ otherwise,}
		\end{cases} 
$$
which proves the announced result.

\subsection{Proof of Lemma \ref{le:v1b}} \label{sec:pv1b}

For $-(2AH)^{-1} \le \lam \le (2AH)^{-1}$, introduce the random variables
    \[
    F = f(X) \quad \text{ \quad } \quad F^* = f^*(X),
    \]
    \[
    \Omega = \ela'_Y(F^*)+(F-F^*) \int_0^1 (1-t) \ela''_Y(F^*+t(F-F^*)) dt,
    \]
    \[
    L=\lam [\ela(Y,F) - \ela(Y,F^*) ],
    \]
and the quantities
    \[
    a(\lam) = \frac{M^2 A^2 \exp (H b_2/A)}{2\sqrt{\pi}(1-|\lam| AH)}
    \]
and 
    \[
    \tA = H b_2/2 + A \log (M) = \frac{A}{2} \log 
\bigl\{  M^2 \exp \bigl[Hb_2/(2A) \bigr] \bigr\}.
    \]

From Taylor-Lagrange formula, we have
    \[
    L = \lam (F-F^*) \Omega.
    \]
    
Since $\E \bigl[  \exp \bigl( |\Omega|/A 
\bigr)\,|\, X \bigr] \le M \exp \bigl[ H b_2/(2A) \bigr]$, 
Lemma \ref{le:stdb} gives
    \[
    \log 
\Bigl\{ \E \Bigr[ \exp \bigl\{ 
\alpha [\Omega - \E(\Omega | X )]/A \bigr\}\, | \, X \Bigr] \Bigr\}  
\le \frac{M^2 \alpha^2 \exp \bigl(H b_2/A\bigr)}{2\sqrt{\pi} (1-|\alpha|)}
    \]
for any ${-1}<\alpha<1$,
and
    \beglab{eq:upbomega}
    \big| \E(\Omega|X) \big| \le \tA.
    \endlab
By considering $\alpha=A\lam[f(x)-f^*(x)]\in[-1/2;1/2]$ for fixed $x\in\X$, we get
    \beglab{eq:tm1}
    \log \Bigl\{ \E \Bigl[ 
\exp \bigl[ L-\E(L|X) \bigr]\,|\, X \Bigr] \Bigr\} \le \lam^2 (F-F^*)^2 a(\lam).
    \endlab
Let us put moreover 
    \[
    \tL= \E(L|X) + a(\lam) \lam^2 (F-F^*)^2.
    \]
Since $-(2AH)^{-1} \le \lam \le (2AH)^{-1}$, we have
    $
    \tL \le |\lam| H \tA + a(\lam) \lam^2 H^2 \le b'
    $
with $b'=\tA/(2A) + M^2 \exp \bigl( H b_2/A \bigr) /(4\sqrt{\pi})$.
Since $L-\E(L) = L-\E( L|X) + \E( L|X) - \E (L)$, by using Lemma \ref{le:stda},
\eqref{eq:tm1} and \eqref{eq:upbomega}, we obtain
    \begin{align*}
    \log \Bigl\{ \E \Bigl[ \exp \bigl[L-\E(L) \bigr] \Bigr] \Bigr\} 
        & \le \log \Bigl\{ \E \Bigl[ \exp \bigl[ \tL-\E 
(\tL) \bigr] \Bigr] \Bigr\} + \lam^2 a(\lam) \E 
\bigl[ (F-F^*)^2 \bigr] \\
    & \le \E \bigl( \tL^2 \bigr) g( b' ) + \lam^2 a(\lam) \E 
\bigl[ (F-F^*)^2 \bigr] \\
    & \le \lam^2 \E \bigl[ (F-F^*)^2 \bigr] \big[ {\tA}^2 g( b' ) + a(\lam) \big],
    \end{align*}       
with $g(u)= \bigl[ \exp(u) - 1 - u \bigr] /u^2$.
Computations show that for any $-(2AH)^{-1} \le  \lam \le (2AH)^{-1}$,
    \[
    \tA^2 g( b' ) + a(\lam) \le \frac{A^2}{4} \exp \Bigl[ M^2 \exp \bigl(H b_2/A \bigr) \Bigr].
    \]
Consequently, for any $-(2AH)^{-1} \le  \lam \le (2AH)^{-1}$, we have 
\begin{multline*}
    \log 
\Bigl\{ \E \Bigl[  \exp \bigl\{ \lam [\ela(Y,F) - \ela(Y,F^*) ]
\bigr\} \Bigr] \Bigr\} \\ \le \lam[R(f)-R(f^*)] 
    + \lam^2 \E \bigl[ (F-F^*)^2 \bigr] \frac{A^2}{4}  
\exp \Bigl[ M^2 \exp \bigl(H b_2/A \bigr) \Bigr].
\end{multline*}
Now it remains to notice that $\E \bigl[ (F-F^*)^2  
\bigr] \le 2 [R(f) - R(f^*)]/b_1.$
Indeed consider the function $\phi(t) = R(f^*+t(f-f^*))-R(f^*),$
where $f\in\cF$ and $t\in[0;1]$. From the definition of $f^*$
and the convexity of $\cF$, we have $\phi\ge 0$ on $[0;1]$. Besides
we have
	$\phi(t)=\phi(0)+t\phi'(0)+\frac{t^2}{2}\phi''(\zeta_t)$
for some $\zeta_t\in]0;1[$. So we have $\phi'(0)\ge 0$, and using the lower bound on the convexity, 
we obtain for $t=1$
	\beglab{eq:proj}
	\frac{b_1}{2} \E(F-F^*)^2 \le R(f) - R(f^*).
	\endlab

\subsection{Proof of Lemma \ref{le:v2}} \label{sec:pv2}

We have
    \begin{align*}
    & \E\Big( \big\{ [Y-f(X)]^2-[Y-f^*(X)]^2 \big\}^2 \Big)\\
       = \ & \E\Big( [f^*-f(X)]^2\big\{2[Y-f^*(X)]+[f^*-f(X)]\big\}^2 \Big)\\
    = \ & \E\Big( [f^*-f(X)]^2\big\{4\E\big([Y-f^*(X)]^2\big|X\big) \\
    & \qquad \qquad +4\E(Y-f^*(X)|X)[f^*(X)-f(X)]+[f^*(X)-f(X)]^2\big\} \Big)\\
    \le \ & \E\Big( [f^*-f(X)]^2\big\{4\sigma^2 +4\sigma |f^*(X)-f(X)|+[f^*(X)-f(X)]^2\big\} \Big)\\
    \le \ & \E\Big( [f^*-f(X)]^2(2\sigma+H)^2\Big)\\
    \le \ & (2\sigma+H)^2 [R(f)-R(f^*)],
    \end{align*}
where the last inequality is the usual relation between excess risk and $L^2$ distance 
using the convexity of $\cF$ (see above \eqref{eq:proj} for a proof).

\subsection{Proof of Lemma \ref{le:v3}} \label{sec:pv3}

Let $\cS = \{ s\in\Flin: \E[s(X)^2]=1\}$. 
Using the triangular inequality in $\B{L}^2$, we get
    \begin{align*}
    & \E\Big( \big\{ [Y-f(X)]^2-[Y-f^*(X)]^2 \big\}^2 \Big)\\
       = \ & \E\Big( \big\{2[f^*-f(X)][Y-f^*(X)]+[f^*(X)-f(X)]^2\big\}^2 \Big)\\
     \le \ & \Big( 2\sqrt{\E\big\{[f^*(X)-f(X)]^2[Y-f^*(X)]^2\big\}}+\sqrt{\E\big\{[f^*(X)-f(X)]^4\big\}} \Big)^2\\
     \le \ & \bigg[ 2\sqrt{\E\big([f^*(X)-f(X)]^2\big)}\sqrt{\sup_{s\in\cS}\E\big(s^2(X)[Y-f^*(X)]^2\big)}\\
    & \qquad + \E\big([f^*(X)-f(X)]^2\big) \sqrt{\sup_{s\in\cS}\E\big[s^4(X)\big]} \bigg]^2\\
    \le \ & V [R(f)-R(f^*)],
    \end{align*}
with 
  \begin{align*}
  V= \bigg[ 2& \sqrt{\sup_{s\in\cS}\E\big(s^2(X)[Y-f^*(X)]^2\big)}\\
    & \qquad + \sqrt{\sup_{f',f''\in\cF} \E\big([f'(X)-f''(X)]^2\big)} \sqrt{\sup_{s\in\cS}\E\big[s^4(X)\big]} \bigg]^2,   
  \end{align*}
where the last inequality is the usual relation between excess risk and $L^2$ distance 
using the convexity of $\cF$ (see above \eqref{eq:proj} for a proof).

\appendix

\section{Uniformly bounded conditional variance is necessary to reach $d/n$ rate} \label{sec:lb}

In this section, we will see that the target \eqref{eq:exptarget} cannot be 
reached if we just assume that $Y$ has a finite variance and that the functions 
in $\cF$ are bounded. 

For this, consider an input space $\X$ partitioned into two sets $\X_1$ and
$\X_2$: $\X=\X_1\cup\X_2$ and $\X_1\cap\X_2=\emptyset$.
Let $\vp_1(x)=\ds1_{x\in\X_1}$ and $\vp_2(x)=\ds1_{x\in\X_2}$.
Let $\cF= \big\{ \th_1 \vp_1+\th_2 \vp_2 ; (\th_1,\th_2) \in [-1,1]^2 \big\}.$
\begin{thm} \label{th:lb}
For any estimator $\hf$ and any training set size $n\ge 1$, we have
	\beglab{eq:lb}
	\und{\sup}{P} \big\{ \E R(\hf) - R(f^*) \big\}
		\ge \frac{1}{4\sqrt{n}},
	\endlab
where the supremum is taken with respect to all probability distributions
such that $\freg \in \cF$ and $\Var Y \le 1$.
\end{thm}

\begin{proof}
Let $\be$ satisfying $0<\be\le 1$ be some parameter to be chosen later.
Let $P_\sigma$, $\sigma \in\{-,+\}$, be two probability distributions on
$\X \times \R$ such that for any $\sigma \in\{-,+\}$,
	\[P_\sigma(\X_1) = 1-\be,\]
	\[P_\sigma(Y=0|X=x) = 1 \qquad \text{for any } x\in\X_1,\]
and
\begin{multline*}
	P_\sigma\Big(Y=\frac{1}{\sqrt{\be}}\, | \, X=x \Bigr) = 
\frac{1+\sigma\sqrt{\be}}{2} \\ = 1 - P_\sigma \Bigl( Y=-\frac{1}{\sqrt{\be}}
\, | \, X=x \Bigr) 
		\quad \text{for any } x\in\X_2.
\end{multline*}
One can easily check that for any $\sigma\in\{-,+\}$, $\Var_{P_\sigma}(Y) =1-\be^2 \le 1$ and
$\freg(x) = \sigma \vp_2 \in\cF$.
To prove Theorem \ref{th:lb}, it suffices to prove \eqref{eq:lb} when the supremum is taken 
among $P\in\{P_-,P_+\}$. This is done by applying Theorem 8.2 of \cite{Aud08}.
Indeed, the pair $(P_-,P_+)$ forms a $(1,\be,\be)$-hypercube in the sense of Definition 8.2
with edge discrepancy of type I (see (8.5), (8.11) and (10.20) for $q=2$): 
$d_I = 1$. We obtain
	\[
	\und{\sup}{P\in\{P_-,P_+\}} \big\{ \E R(\hf) - R(f^*) \big\}
		\ge \be(1-\be\sqrt{n}),
	\]
which gives the desired result by taking $\be=1/(2\sqrt{n}).$
\end{proof}

\section{Empirical risk minimization on a ball: analysis derived from the work of Birg\'e and Massart} \label{sec:bm}

We will use the following covering number upper bound 
\cite[Lemma~1]{Lor1966}

\begin{lemma} \label{le:infty}
If $\cF$ has a diameter $H>0$ for $L^\infty$-norm (i.e., $\sup_{f_1,f_2\in \cF,x\in\X} |f_1(x)-f_2(x)| = H$), then 
for any $0<\delta\le H$, there exists a set $\Fg \subset \cF$, of cardinality $|\Fg|\le (3H/\delta)^d$ such that
for any $f\in\cF$ there exists $g\in\Fg$ such that $\|f-g\|_{\infty} \le \delta.$
\end{lemma}

We apply a slightly improved version of Theorem~5 in Birg\'e and Massart \cite{BirMas98}.
First for homogeneity purpose, we modify Assumption M2 by replacing the 
condition ``$\sigma^2 \ge D/n$'' by ``$\sigma^2 \ge B^2 D/n$'' where
the constant $B$ is the one appearing in (5.3) of \cite{BirMas98}. 
This modifies Theorem 5 of \cite{BirMas98} to the extent that ``$\vee 1$'' 
should be replaced with ``$\vee B^2$''.
Our second modification is to remove the assumption that $W_i$ and $X_i$ 
are independent. A careful look at the proof shows that the result
still holds when (5.2) is replaced by: for any $x\in\X$, and $m\ge2$
	\[\text{E}_s [M^m(W_i)|X_i=x] \le a_m A^m, \qquad\text{for all } i=1,\dots,n\]
We consider 
	$W = Y - f^*(X)$,
	$\gamma(z,f) = (y-f(x))^2$,
	$\Delta(x,u,v) = |u(x)-v(x)|$,
and
	$M(w)= 2 ( |w| + H )$.
From \eqref{eq:expmom}, for all $m\ge 2$, we have
	$\E \big\{ [(2 (|W|+H)]^m|X=x] \le \frac{m!}{2} [4M (A+H)]^m.$
Now consider $B'$ and $r$ such that Assumption M2 of \cite{BirMas98} holds for $D=d$.
Inequality (5.8) for $\tau=1/2$ of \cite{BirMas98} implies that for any $v\ge \kap \frac{d}{n} (A^2+H^2) \log(2B'+B'r \sqrt{d/n})$,
with probability at least $\ds 1 - \kap \exp \Bigl[ \frac{-n v}{\kap(A^2+H^2)} 
\Bigr]$,
	\[
	R(\hferm)-R(f^*)+r(f^*)-r(\hferm) \le \bigl( \E
\bigl\{ \bigl[\hferm(X)-f^*(X)\bigr]^2 \bigr\} \vee v \bigr)/2  
	\]
for some large enough constant $\kap$ depending on $M$.
Now from Proposition 1 of \cite{BirMas98} and Lemma \ref{le:infty}, one can take 
either $B'=6$ and $r \sqrt{d} = \sqrt{\cR}$ or $B'=3\sqrt{n/d}$ and
$r=1$.  
By using $\E\bigl\{ \bigl[\hferm(X)-f^*(X) \bigr]^2 
\bigr\} \le R(\hferm)-R(f^*)$ (since $\cF$ is convex
and $f^*$ is the orthogonal projection of $Y$ on $\cF$), and
$r(f^*)-r(\hferm)\ge 0$ (by definition of $\hferm$), the desired result
can be derived.

Theorem \ref{th:bmnew} provides a $d/n$ rate provided that the geometrical quantity $\cR$ is at most of order $n$.
Inequality (3.2) of \cite{BirMas98} allows to bracket $\cR$ in terms
of $\bcR = \sup_{f\in\Span \{\vp_1,\dots,\vp_d\}} \fracc{\|f\|_\infty^2}{\E[f(X)]^2}$, namely
	$\bcR \le \cR \le \bcR d$.
To understand better how this quantity behaves and to illustrate some of the presented
results, let us give the following simple example.

{\bf Example 1.} \label{ex1a}
Let $A_1,\dots,A_d$ be a partition of $\X$, i.e., $\X =\sqcup_{j=1}^d A_j$.
Now consider the indicator functions $\vp_j=\ds1_{A_j}, j=1,\dots,d$: $\vp_j$ is equal
to $1$ on $A_j$ and zero elsewhere. Consider 
that $X$ and $Y$ are independent and that $Y$ is a Gaussian random variable with mean $\theta$ and 
variance $\sigma^2$. In this situation: $\flin=\freg=\sum_{j=1}^d \theta \vp_j$.
According to Theorem \ref{th:weakols}, if we know an upper bound $H$ on $\|\freg\|_\infty=\th$, we have that the
truncated estimator $(\hfols\wedge H)\vee -H$ satisfies 
	\[
	\E R(\hfols_H) - R(\flin)	\le \kap \frac{(\sigma^2\vee H^2)d\log n }{n}
	\]
for some numerical constant $\kap$.
Let us now apply Theorem \ref{th:capvit}.
Introduce $p_j=\P(X\in A_j)$ and $p_{\min} = \min_{j} p_j$. 
We have $Q = \big( \E \vp_j(X) \vp_k(X) \big)_{j,k} = \diag(p_j)$,
$\cK=1$ and $\|\th^*\|=\th\sqrt{d}$. We can take $A=\sigma$ and $M=2$.
From Theorem \ref{th:capvit}, for $\lam=d \cL_\eps/n$, as soon as $\lam \le p_{\min}$,
the ridge regression estimator satisfies with probability at least $1-\eps$:
	\beglab{eq:ex1cv}
	R(\hfrlam) - R( \flin ) \le 
		\kap \cL_\eps \frac{d}{n} \bigg( \sigma^2 + \frac{\th^2 d^2 \cL^2_\eps}{n p_{\min}} \bigg)
	\endlab
for some numerical constant $\kap$. When $d$ is large, the term $\fracb{d^2 \cL^2_\eps}{n p_{\min}}$ is felt, 
and leads to suboptimal rates. Specifically, since $p_{\min}\le 1/d$,
the $\rhs$ of \eqref{eq:ex1cv} is greater than $d^4/n^2$, which is much larger than $d/n$ when $d$ is much larger than $n^{1/3}$. 
If $Y$ is not Gaussian but almost surely uniformly bounded by $C<+\infty$, then the randomized estimator proposed in Theorem \ref{th:alqa}
satisfies the nicer property: with probability at least $1-\eps$,
    \[
    R(\hat{f}) - R(\flin) \le \kap (H^2 +C^2) \frac{d \log( 3p_{\min}^{-1} ) + 
        \log ( (\log n) \eps^{-1} ) }{n},
    \]
for some numerical constant $\kap$.
In this example, one can check that $\cR=\cR'=1/p_{\min}$ where $p_{\min}= \min_j \P(X\in A_j).$
As long as $p_{\min}\ge 1/n$, the target \eqref{eq:devtarget} is reached from Corollary \ref{th:bmnew}.
Otherwise, without this assumption, the rate is in $(d\log(n/d))/n$. 
$\blacksquare$ 

\section{Ridge regression analysis from the work of Caponnetto and De Vito} \label{sec:capvit}

From \cite{CapVit07}, one can derive the following risk bound for the ridge estimator.

\begin{thm} \label{th:capvit}
Let $q_{\min}$ be the smallest eigenvalue of the $d \times d$-product matrix $Q=\big( \E \vp_j(X) \vp_k(X) \big)_{j,k}$.
Let $\cK= \sup_{x\in\X} \sum_{j=1}^d \vp_j(x)^2$. Let $\|\th^*\|$ be the Euclidean norm of 
the vector of parameters of $\flin=\sum_{j=1}^d \th^*_j \vp_j$.
Let $0<\eps<1/2$ and $\cL_\eps = \leps$. 
Assume that for any $x\in\X,$ 
	\[\E \Bigl\{ \exp \bigl[ |Y-\flin(X)|/A \bigr]\, | \, X= x 
\Bigr\} \le M.
	\]
For $\lam=\fracl{\cK d \cL_\eps}{n}$, if
$\lam \le q_{\min}$,
the ridge regression estimator satisfies with probability at least $1-\eps$:
	\beglab{eq:capvit}
	R(\hfrlam) - R( \flin ) \le 
		\frac{\kap \cL_\eps d}{n} \bigg( A^2 + \frac{\lam}{q_{\min}} \cK \cL_\eps \|\theta^*\|^2 \bigg)
	\endlab
for some positive constant $\kap$ depending only on $M$.
\end{thm}

\begin{proof} 
One can check that 
	$
	\hfrlam \in \undc{\argmin}{f\in\cH} 
	            r(f) + \lam \sum_{j=1}^d \|f\|_{\cH}^2,
	$
where $\cH$ is the reproducing kernel Hilbert space associated with the kernel $K:(x,x') \mapsto \sum_{j=1}^d \vp_j(x) \vp_k(x')$.
Introduce 
	$
	f^{(\lam)} \in \undc{\argmin}{f\in\cH} R(f) + \lam \sum_{j=1}^d \|f\|_{\cH}^2.
	$
Let us use Theorem 4 in \cite{CapVit07} and the notation defined in their Section 5.2.
Let $\vp$ be the column vector of functions $[\vp_j]_{j=1}^d$, 
$\diag(a_j)$ denote the diagonal $d \times d$-matrix whose $j$-th element on the diagonal is $a_j$,
and $I_d$ be the $d \times d$-identity matrix.
Let $U$ and $q_1,\dots,q_d$ be such that $UU^T=I$ and 
	$Q=U \diag(q_j)U^T$.
We have $\flin = \vp^T \th^*$
and $f^{(\lam)} = \vp^T (Q+\lam I)^{-1} Q \th^*$, hence
	$$
	\flin - f^{(\lam)} = \vp^T U \diag(\lam/(q_j+\lam)) U^T \th^*.
	$$
After some computations, we obtain that the residual, reconstruction error and effective dimension respectively satisfy
$\A(\lam) \le \frac{\lam^2}{q_{\min}} \|\th^*\|^2$, $\cB(\lam) \le \frac{\lam^2}{q_{\min}^2} \|\th^*\|^2$,
and $\N(\lam) \le d$. The result is obtained by noticing that the leading terms in (34) of \cite{CapVit07} are
$\A(\lam)$ and the term with the effective dimension $\N(\lam)$. 
\end{proof}

The dependence in the sample size $n$ is correct since $1/n$ is known to be minimax optimal. 
The dependence on the dimension $d$ is not optimal, as it is observed in the example given page \pageref{ex1a}. 
Besides the high probability bound \eqref{eq:capvit}
holds only for a regularization parameter $\lam$ depending on the confidence level $\eps$. So we do not have a single estimator 
satisfying a PAC bound for every confidence level. 
Finally the dependence on the confidence level is larger than expected. It contains an unusual square.
The example given page \pageref{ex1a} illustrates Theorem \ref{th:capvit}.

\section{Some standard upper bounds on log-Laplace trans\-forms}

\begin{lemma} \label{le:stda}
Let $V$ be a random variable almost surely bounded by $b\in\R$.
Let $g: u\mapsto \bigl[ \exp(u) - 1 - u \bigr]/u^2$.
    \[
    \log \Bigl\{  \E  \Bigr[ \exp \bigl[ V-\E (V) \bigr] 
\Bigr] \Bigr\} \le \E \bigl( V^2 \bigr)  g(b).
    \]
\end{lemma}

\begin{proof}
Since $g$ is an increasing function, we have $g(V) \le g(b)$. By using the inequality
$\log(1+u) \le u$, we obtain
\begin{multline*}
    \log 
\Bigl\{ \E  \Bigl[ \exp \bigl[ V-\E (V) \bigr] \Bigr] \Bigr\} = 
-\E (V) + \log \bigl\{ \E \bigl[ 1+V+V^2g(V)
\bigr] \bigr\} \\ \le \E \bigl[ V^2 g(V)\bigr]  \le \E 
\bigl( V^2 \bigr) g(b).
\end{multline*}
\end{proof}

\begin{lemma} \label{le:stdb}
Let $V$ be a real-valued random variable such that 
$\E \bigl[ \exp \bigl( |V| \bigr) \bigr] \le M$ for some $M>0$.
Then we have $|\E (V)| \le \log M$, and for any $-1<\alpha<1$,
    \[
    \log 
\Bigl\{ \E  \Bigr[ \exp \bigl\{  \alpha  \bigl[ V-\E (V) 
\bigr] \bigr\} \Bigr] \Bigr\} \le \frac{\alpha^2 M^2}{2\sqrt{\pi}(1-|\alpha|)}.
    \]
\end{lemma}

\begin{proof}
First note that by Jensen's inequality, we have $|\E (V)|\le \log(M)$.
By using $\log(u) \le u-1$ and Stirling's formula, for any $-1<\alpha<1$, we have
    \begin{multline*}
    \log 
\Bigl\{ \E \Bigl[ \exp \bigl\{ \alpha  \bigl[ V-\E (V) 
\bigr] \bigr\} \Bigr] \Bigr\} \le \E 
\Bigl[ \exp \bigl\{ \alpha  \bigl[ V-\E (V) \bigr] 
\bigr\} \Bigr] \Bigr\}  - 1\\
     = \E \Bigl\{ \exp \bigl\{ \alpha \bigl[ V-\E (V) 
\bigr] \bigr\} - 1 - \alpha \bigl[ V-\E (V) \bigr] 
\Bigr\} \\
     \le \E 
\Bigl\{ \exp \bigl[ |\alpha| |V-\E (V)| \bigr] - 1 - |\alpha| |V-\E (V)| 
\Bigr\} \\
     \le \E \Bigl\{ \exp \bigl[ |V-\E (V)| 
\bigr] \Bigr\} \sup_{u\ge 0} 
\Bigl\{ \bigl[ \exp (|\alpha| u) - 1 - |\alpha| u 
\bigr]\exp(-u) \Bigr\}\\
     \le \E \Bigl[ \exp \bigl(|V|+|\E (V)| \bigr) \Bigr]  
\sup_{u\ge 0} \sum_{m\ge 2} \frac{|\alpha|^m u^m}{m!} \exp(-u)\\
     \le M^2 \sum_{m\ge 2} \frac{|\alpha|^m}{m!} \sup_{u\ge 0} u^m \exp(-u)
     = \alpha^2 M^2 \sum_{m\ge 2} \frac{|\alpha|^{m-2}}{m!} m^m \exp(-m) \\
     \le \alpha^2 M^2 \sum_{m\ge 2} \frac{|\alpha|^{m-2}}{\sqrt{2\pi m}}
     \le \frac{\alpha^2 M^2}{2\sqrt{\pi}(1-|\alpha|)}.
    \end{multline*}
\end{proof}

\pagebreak

\section{Experimental results for the min-max truncated estimator defined in Section \ref{sec:comput}} \label{app:exp}


\begin{table}[!ht]
\caption{Comparison of the min-max truncated estimator $\hf$
with the ordinary least squares estimator $\hfols$ for the mixture noise (see Section \ref{sec:noise})
with $\rho=0.1$ and $p=0.005$. In parenthesis, the $95\%$-confidence intervals for the estimated quantities.}
\label{tab:b01} 
\begin{center}
\scalebox{0.75}{\begin{tabular}{l|c|c|c|c|c|c|c}
\hline 
& \rotatebox{90}{nb of iterations } & \rotatebox{90}{nb of iter. with $R(\hf)\neq R(\hfols)$} & 
\rotatebox{90}{nb of iter. with $R(\hf)< R(\hfols)$} & \rotatebox{90}{$\E R(\hfols)-R(f^*)$}
& \rotatebox{90}{$\E R(\hf)-R(f^*)$} & \rotatebox{90}{$\E R[(\hfols)|\hf\neq\hfols]-R(f^*)$} 
& \rotatebox{90}{$\E [R(\hf)|\hf\neq\hfols]-R(f^*)$}\\ 
\hline 
INC(n=200,d=1)& $1000$& $419$& $405$& $ 0.567 (\pm  0.083)$& $ 0.178 (\pm  0.025)$& $ 1.191 (\pm  0.178)$& $ 0.262 (\pm  0.052)$\\
INC(n=200,d=2)& $1000$& $506$& $498$& $ 1.055 (\pm  0.112)$& $ 0.271 (\pm  0.030)$& $ 1.884 (\pm  0.193)$& $ 0.334 (\pm  0.050)$\\
HCC(n=200,d=2)& $1000$& $502$& $494$& $ 1.045 (\pm  0.103)$& $ 0.267 (\pm  0.024)$& $ 1.866 (\pm  0.174)$& $ 0.316 (\pm  0.032)$\\
TS(n=200,d=2)& $1000$& $561$& $554$& $ 1.069 (\pm  0.089)$& $ 0.310 (\pm  0.027)$& $ 1.720 (\pm  0.132)$& $ 0.367 (\pm  0.036)$\\
INC(n=1000,d=2)& $1000$& $402$& $392$& $ 0.204 (\pm  0.015)$& $ 0.109 (\pm  0.008)$& $ 0.316 (\pm  0.029)$& $ 0.081 (\pm  0.011)$\\
INC(n=1000,d=10)& $1000$& $950$& $946$& $ 1.030 (\pm  0.041)$& $ 0.228 (\pm  0.016)$& $ 1.051 (\pm  0.042)$& $ 0.207 (\pm  0.014)$\\
HCC(n=1000,d=10)& $1000$& $942$& $942$& $ 0.980 (\pm  0.038)$& $ 0.222 (\pm  0.015)$& $ 1.008 (\pm  0.039)$& $ 0.203 (\pm  0.015)$\\
TS(n=1000,d=10)& $1000$& $976$& $973$& $ 1.009 (\pm  0.037)$& $ 0.228 (\pm  0.017)$& $ 1.018 (\pm  0.038)$& $ 0.217 (\pm  0.016)$\\
INC(n=2000,d=2)& $1000$& $209$& $207$& $ 0.104 (\pm  0.007)$& $ 0.078 (\pm  0.005)$& $ 0.206 (\pm  0.021)$& $ 0.082 (\pm  0.012)$\\
HCC(n=2000,d=2)& $1000$& $184$& $183$& $ 0.099 (\pm  0.007)$& $ 0.076 (\pm  0.005)$& $ 0.196 (\pm  0.023)$& $ 0.070 (\pm  0.010)$\\
TS(n=2000,d=2)& $1000$& $172$& $171$& $ 0.101 (\pm  0.007)$& $ 0.080 (\pm  0.005)$& $ 0.206 (\pm  0.020)$& $ 0.083 (\pm  0.012)$\\
INC(n=2000,d=10)& $1000$& $669$& $669$& $ 0.510 (\pm  0.018)$& $ 0.206 (\pm  0.012)$& $ 0.572 (\pm  0.023)$& $ 0.117 (\pm  0.009)$\\
HCC(n=2000,d=10)& $1000$& $669$& $669$& $ 0.499 (\pm  0.018)$& $ 0.207 (\pm  0.013)$& $ 0.561 (\pm  0.023)$& $ 0.125 (\pm  0.011)$\\
TS(n=2000,d=10)& $1000$& $754$& $753$& $ 0.516 (\pm  0.018)$& $ 0.195 (\pm  0.013)$& $ 0.558 (\pm  0.022)$& $ 0.131 (\pm  0.011)$\\

\hline 
\end{tabular}}
\end{center}
\end{table}
\pagebreak


\begin{table}[!ht]
\caption{Comparison of the min-max truncated estimator $\hf$
with the ordinary least squares estimator $\hfols$ for the mixture noise (see Section \ref{sec:noise})
with $\rho=0.4$ and $p=0.005$. In parenthesis, the $95\%$-confidence intervals for the estimated quantities.} \label{tab:b04}
{
\begin{center}\scalebox{0.75}{\begin{tabular}{l|c|c|c|c|c|c|c}
\hline 
& \rotatebox{90}{nb of iterations } & \rotatebox{90}{nb of iter. with $R(\hf)\neq R(\hfols)$} & 
\rotatebox{90}{nb of iter. with $R(\hf)< R(\hfols)$} & \rotatebox{90}{$\E R(\hfols)-R(f^*)$}
& \rotatebox{90}{$\E R(\hf)-R(f^*)$} & \rotatebox{90}{$\E R[(\hfols)|\hf\neq\hfols]-R(f^*)$} 
& \rotatebox{90}{$\E [R(\hf)|\hf\neq\hfols]-R(f^*)$}\\
\hline 
INC(n=200,d=1)& $1000$& $234$& $211$& $ 0.551 (\pm  0.063)$& $ 0.409 (\pm  0.042)$& $ 1.211 (\pm  0.210)$& $ 0.606 (\pm  0.110)$\\
INC(n=200,d=2)& $1000$& $195$& $186$& $ 1.046 (\pm  0.088)$& $ 0.788 (\pm  0.061)$& $ 2.174 (\pm  0.293)$& $ 0.848 (\pm  0.118)$\\
HCC(n=200,d=2)& $1000$& $222$& $215$& $ 1.028 (\pm  0.079)$& $ 0.748 (\pm  0.051)$& $ 2.157 (\pm  0.243)$& $ 0.897 (\pm  0.112)$\\
TS(n=200,d=2)& $1000$& $291$& $268$& $ 1.053 (\pm  0.079)$& $ 0.805 (\pm  0.058)$& $ 1.701 (\pm  0.186)$& $ 0.851 (\pm  0.093)$\\
INC(n=1000,d=2)& $1000$& $127$& $117$& $ 0.201 (\pm  0.013)$& $ 0.181 (\pm  0.012)$& $ 0.366 (\pm  0.053)$& $ 0.207 (\pm  0.035)$\\
INC(n=1000,d=10)& $1000$& $262$& $249$& $ 1.023 (\pm  0.035)$& $ 0.902 (\pm  0.030)$& $ 1.238 (\pm  0.081)$& $ 0.777 (\pm  0.054)$\\
HCC(n=1000,d=10)& $1000$& $201$& $192$& $ 0.991 (\pm  0.033)$& $ 0.902 (\pm  0.031)$& $ 1.235 (\pm  0.088)$& $ 0.790 (\pm  0.067)$\\
TS(n=1000,d=10)& $1000$& $171$& $162$& $ 1.009 (\pm  0.033)$& $ 0.951 (\pm  0.031)$& $ 1.166 (\pm  0.098)$& $ 0.825 (\pm  0.071)$\\
INC(n=2000,d=2)& $1000$& $80$& $77$& $ 0.105 (\pm  0.007)$& $ 0.099 (\pm  0.006)$& $ 0.214 (\pm  0.042)$& $ 0.135 (\pm  0.029)$\\
HCC(n=2000,d=2)& $1000$& $44$& $42$& $ 0.102 (\pm  0.007)$& $ 0.099 (\pm  0.007)$& $ 0.187 (\pm  0.050)$& $ 0.120 (\pm  0.034)$\\
TS(n=2000,d=2)& $1000$& $47$& $47$& $ 0.101 (\pm  0.007)$& $ 0.099 (\pm  0.007)$& $ 0.147 (\pm  0.032)$& $ 0.103 (\pm  0.026)$\\
INC(n=2000,d=10)& $1000$& $116$& $113$& $ 0.511 (\pm  0.016)$& $ 0.491 (\pm  0.016)$& $ 0.611 (\pm  0.052)$& $ 0.437 (\pm  0.042)$\\
HCC(n=2000,d=10)& $1000$& $110$& $105$& $ 0.500 (\pm  0.016)$& $ 0.481 (\pm  0.015)$& $ 0.602 (\pm  0.056)$& $ 0.430 (\pm  0.044)$\\
TS(n=2000,d=10)& $1000$& $101$& $98$& $ 0.511 (\pm  0.016)$& $ 0.499 (\pm  0.016)$& $ 0.601 (\pm  0.054)$& $ 0.486 (\pm  0.051)$\\

\hline 
\end{tabular}}
\end{center}}
\end{table}
\pagebreak


\begin{table}[!ht]
\caption{Comparison of the min-max truncated estimator $\hf$
with the ordinary least squares estimator $\hfols$ with the heavy-tailed noise (see Section \ref{sec:noise}).}
\label{tab:a201} {
\begin{center}\scalebox{0.75}{\begin{tabular}{l|c|c|c|c|c|c|c}
\hline 
& \rotatebox{90}{nb of iterations } & \rotatebox{90}{nb of iter. with $R(\hf)\neq R(\hfols)$} & 
\rotatebox{90}{nb of iter. with $R(\hf)< R(\hfols)$} & \rotatebox{90}{$\E R(\hfols)-R(f^*)$}
& \rotatebox{90}{$\E R(\hf)-R(f^*)$} & \rotatebox{90}{$\E R[(\hfols)|\hf\neq\hfols]-R(f^*)$} 
& \rotatebox{90}{$\E [R(\hf)|\hf\neq\hfols]-R(f^*)$}\\
\hline 
INC(n=200,d=1)& $1000$& $163$& $145$&  $ 7.72   (\pm  3.46)$& $ 3.92 (\pm  0.409)$& $30.52 (\pm 20.8)$& $ 7.20 (\pm  1.61)$\\
INC(n=200,d=2)& $1000$& $104$& $98$&   $22.69    (\pm 23.14)$& $19.18 (\pm 23.09)$&  $45.36 (\pm 14.1)$& $11.63 (\pm  2.19)$\\
HCC(n=200,d=2)& $1000$& $120$& $117$&  $18.16   (\pm 12.68)$& $ 8.07 (\pm  0.718)$& $99.39 (\pm 105)$& $ 15.34 (\pm  4.41)$\\
TS(n=200,d=2)& $1000$& $110$& $105$&   $43.89    (\pm 63.79)$& $39.71 (\pm 63.76)$&  $48.55 (\pm 18.4)$& $10.59 (\pm  2.01)$\\
INC(n=1000,d=2)& $1000$& $104$& $100$& $ 3.98  (\pm  2.25)$& $ 1.78 (\pm  0.128)$& $23.18 (\pm 21.3)$& $ 2.03 (\pm  0.56)$\\
INC(n=1000,d=10)& $1000$& $253$& $242$&$16.36 (\pm  5.10)$& $ 7.90 (\pm  0.278)$& $41.25 (\pm 19.8)$& $ 7.81 (\pm  0.69)$\\
HCC(n=1000,d=10)& $1000$& $220$& $211$&$13.57 (\pm  1.93)$& $ 7.88 (\pm  0.255)$& $33.13 (\pm  8.2)$& $ 7.28 (\pm  0.59)$\\
TS(n=1000,d=10)& $1000$& $214$& $211$& $18.67  (\pm 11.62)$& $13.79 (\pm 11.52)$&  $30.34 (\pm  7.2)$& $ 7.53 (\pm  0.58)$\\
INC(n=2000,d=2)& $1000$& $113$& $103$& $ 1.56  (\pm  0.41)$& $ 0.89 (\pm  0.059)$& $ 6.74 (\pm  3.4)$& $ 0.86 (\pm  0.18)$\\
HCC(n=2000,d=2)& $1000$& $105$& $97$&  $ 1.66   (\pm  0.43)$& $ 0.95 (\pm  0.062)$& $ 7.87 (\pm  3.8)$& $ 1.13 (\pm  0.23)$\\
TS(n=2000,d=2)& $1000$& $101$& $95$&   $ 1.59    (\pm  0.64)$& $ 0.88 (\pm  0.058)$& $ 8.03 (\pm  6.2)$& $ 1.04 (\pm  0.22)$\\
INC(n=2000,d=10)& $1000$& $259$& $255$&$ 8.77 (\pm  4.02)$& $ 4.23 (\pm  0.154)$& $21.54 (\pm 15.4)$& $ 4.03 (\pm  0.39)$\\
HCC(n=2000,d=10)& $1000$& $250$& $242$&$ 6.98 (\pm  1.17)$& $ 4.13 (\pm  0.127)$& $15.35 (\pm  4.5)$& $ 3.94 (\pm  0.25)$\\
TS(n=2000,d=10)& $1000$& $238$& $233$& $ 8.49  (\pm  3.61)$& $ 5.95 (\pm  3.486)$& $14.82 (\pm  3.8)$& $ 4.17 (\pm  0.30)$\\

\hline 
\end{tabular}}
\end{center}}
\end{table}
\pagebreak

\begin{table}[!ht]
\caption{Comparison of the min-max truncated estimator $\hf$
with the ordinary least squares estimator $\hfols$ with an asymetric variant of the heavy-tailed noise.}
\label{tab:a-201} {
\begin{center}\scalebox{0.75}{\begin{tabular}{l|c|c|c|c|c|c|c}
\hline 
& \rotatebox{90}{nb of iterations } & \rotatebox{90}{nb of iter. with $R(\hf)\neq R(\hfols)$} & 
\rotatebox{90}{nb of iter. with $R(\hf)< R(\hfols)$} & \rotatebox{90}{$\E R(\hfols)-R(f^*)$}
& \rotatebox{90}{$\E R(\hf)-R(f^*)$} & \rotatebox{90}{$\E R[(\hfols)|\hf\neq\hfols]-R(f^*)$} 
& \rotatebox{90}{$\E [R(\hf)|\hf\neq\hfols]-R(f^*)$}\\
\hline 
INC(n=200,d=1)& $1000$& $87$& $77$&    $ 5.49 (\pm  3.07)$& $ 3.00 (\pm  0.330)$& $35.44 (\pm 34.7)$&  $ 6.85 (\pm  2.48)$\\
INC(n=200,d=2)& $1000$& $70$& $66$&    $19.25 (\pm 23.23)$& $17.4  (\pm 23.2)$&   $37.95 (\pm 13.1)$&  $11.05 (\pm  2.87)$\\
HCC(n=200,d=2)& $1000$& $67$& $66$&    $ 7.19 (\pm  0.88)$& $ 5.81 (\pm  0.397)$& $31.52 (\pm 10.5)$&  $10.87 (\pm  2.64)$\\
TS(n=200,d=2)& $1000$& $76$& $68$&     $39.80 (\pm 64.09)$& $37.9  (\pm 64.1)$&   $34.28 (\pm 14.8)$&  $ 9.21 (\pm  2.05)$\\
INC(n=1000,d=2)& $1000$& $101$& $92$&  $ 2.81 (\pm  2.21)$& $ 1.31 (\pm  0.106)$& $16.76 (\pm 21.8)$&  $ 1.88 (\pm  0.69)$\\
INC(n=1000,d=10)& $1000$& $211$& $195$&$10.71 (\pm  4.53)$& $ 5.86 (\pm  0.222)$& $29.00 (\pm 21.3)$&  $ 6.03 (\pm  0.71)$\\
HCC(n=1000,d=10)& $1000$& $197$& $185$&$ 8.67 (\pm  1.16)$& $ 5.81 (\pm  0.177)$& $20.31 (\pm  5.59)$& $ 5.79 (\pm  0.43)$\\
TS(n=1000,d=10)& $1000$& $258$& $233$& $13.62 (\pm 11.27)$& $11.3  (\pm 11.2)$&   $14.68 (\pm  2.45)$& $ 5.60 (\pm  0.36)$\\
INC(n=2000,d=2)& $1000$& $106$& $92$&  $ 1.04 (\pm  0.37)$& $ 0.64 (\pm  0.042)$& $ 4.54 (\pm  3.45)$& $ 0.79 (\pm  0.16)$\\
HCC(n=2000,d=2)& $1000$& $99$& $90$&   $ 0.90 (\pm  0.11)$& $ 0.66 (\pm  0.042)$& $ 3.23 (\pm  0.93)$& $ 0.82 (\pm  0.16)$\\
TS(n=2000,d=2)& $1000$& $84$& $81$&    $ 1.11 (\pm  0.66)$& $ 0.60 (\pm  0.042)$& $ 6.80 (\pm  7.79)$& $ 0.69 (\pm  0.17)$\\
INC(n=2000,d=10)& $1000$& $238$& $222$&$ 6.32 (\pm  4.18)$& $ 3.07 (\pm  0.147)$& $16.84 (\pm 17.5)$&  $ 3.18 (\pm  0.51)$\\
HCC(n=2000,d=10)& $1000$& $221$& $203$&$ 4.49 (\pm  0.98)$& $ 2.98 (\pm  0.091)$& $ 9.76 (\pm  4.39)$& $ 2.93 (\pm  0.22)$\\
TS(n=2000,d=10)& $1000$& $412$& $350$& $ 5.93 (\pm  3.51)$& $ 4.59 (\pm  3.44)$&  $ 6.07 (\pm  1.76)$& $ 2.84 (\pm  0.16)$\\

\hline 
\end{tabular}}
\end{center}}
\end{table}
\pagebreak


\begin{table}[!ht]
\caption{Comparison of the min-max truncated estimator $\hf$
with the ordinary least squares estimator $\hfols$ for standard Gaussian noise.}
\label{tab:a0} {
\begin{center}\scalebox{0.75}{\begin{tabular}{l|c|c|c|c|c|c|c}
\hline 
& \rotatebox{90}{nb of iter. } & \rotatebox{90}{nb of iter. with $R(\hf)\neq R(\hfols)$} & 
\rotatebox{90}{nb of iter. with $R(\hf)< R(\hfols)$} & \rotatebox{90}{$\E R(\hfols)-R(f^*)$}
& \rotatebox{90}{$\E R(\hf)-R(f^*)$} & \rotatebox{90}{$\E R[(\hfols)|\hf\neq\hfols]-R(f^*)$} 
& \rotatebox{90}{$\E [R(\hf)|\hf\neq\hfols]-R(f^*)$}\\
\hline 
INC(n=200,d=1)& $1000$& $20$& $8$& $ 0.541 (\pm  0.048)$& $ 0.541 (\pm  0.048)$& $ 0.401 (\pm  0.168)$& $ 0.397 (\pm  0.167)$\\
INC(n=200,d=2)& $1000$& $1$& $0$& $ 1.051 (\pm  0.067)$& $ 1.051 (\pm  0.067)$& $ 2.566 $& $ 2.757 $\\
HCC(n=200,d=2)& $1000$& $1$& $0$& $ 1.051 (\pm  0.067)$& $ 1.051 (\pm  0.067)$& $ 2.566 $& $ 2.757 $\\
TS(n=200,d=2)& $1000$& $0$& $0$& $ 1.068 (\pm  0.067)$& $ 1.068 (\pm  0.067)$& --& --\\
INC(n=1000,d=2)& $1000$& $0$& $0$& $ 0.203 (\pm  0.013)$& $ 0.203 (\pm  0.013)$& --& --\\
INC(n=1000,d=10)& $1000$& $0$& $0$& $ 1.023 (\pm  0.029)$& $ 1.023 (\pm  0.029)$& --& --\\
HCC(n=1000,d=10)& $1000$& $0$& $0$& $ 1.023 (\pm  0.029)$& $ 1.023 (\pm  0.029)$& --& --\\
TS(n=1000,d=10)& $1000$& $0$& $0$& $ 0.997 (\pm  0.028)$& $ 0.997 (\pm  0.028)$& --& --\\
INC(n=2000,d=2)& $1000$& $0$& $0$& $ 0.112 (\pm  0.007)$& $ 0.112 (\pm  0.007)$& --& --\\
HCC(n=2000,d=2)& $1000$& $0$& $0$& $ 0.112 (\pm  0.007)$& $ 0.112 (\pm  0.007)$& --& --\\
TS(n=2000,d=2)& $1000$& $0$& $0$& $ 0.098 (\pm  0.006)$& $ 0.098 (\pm  0.006)$& --& --\\
INC(n=2000,d=10)& $1000$& $0$& $0$& $ 0.517 (\pm  0.015)$& $ 0.517 (\pm  0.015)$& --& --\\
HCC(n=2000,d=10)& $1000$& $0$& $0$& $ 0.517 (\pm  0.015)$& $ 0.517 (\pm  0.015)$& --& --\\
TS(n=2000,d=10)& $1000$& $0$& $0$& $ 0.501 (\pm  0.015)$& $ 0.501 (\pm  0.015)$& --& --\\

\hline 
\end{tabular}}
\end{center}}
\end{table}\pagebreak

\begin{figure}[!ht] 
\vspace*{-1.5cm}
\caption{Surrounding points are the points of the training set generated several times from $TS(1000,10)$ (with the mixture noise with $p=0.005$ and $\rho=0.4$) that are not taken into account in the min-max truncated estimator (to the extent that the estimator would not change by removing simultaneously all these points). 
The min-max truncated estimator $x\mapsto\hf(x)$ appears in dash-dot line, while $x\mapsto\E(Y|X=x)$ is in solid line. In these six simulations, 
it outperforms the ordinary least squares estimator.} \label{fig:1}
\centering
   \vspace*{-.0cm}\begin{minipage}{.49\linewidth}
      \hspace*{-0.8cm}\includegraphics[width=1.15\linewidth]{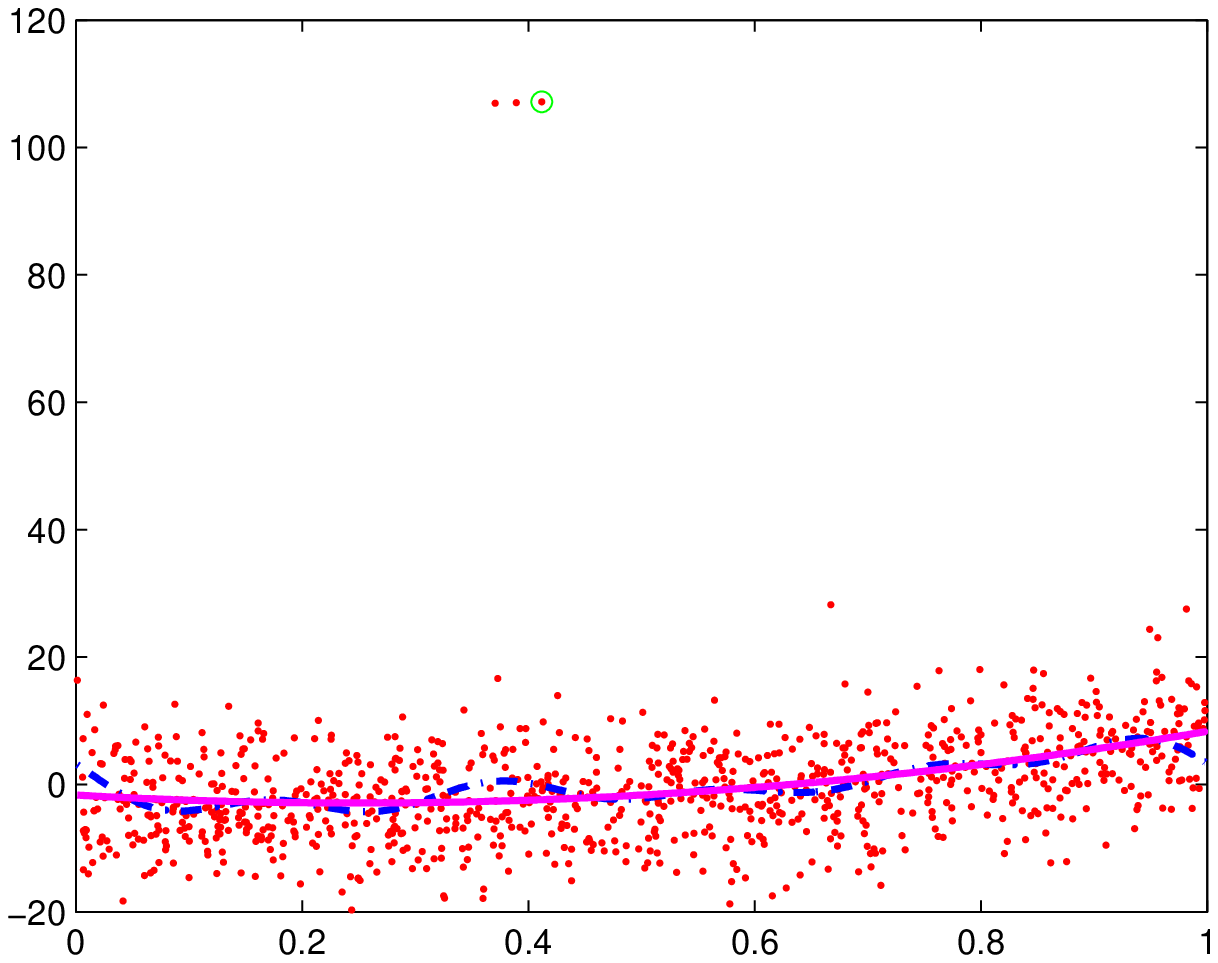}
   \end{minipage} 
   \hfill
   \begin{minipage}{0.49\linewidth}
      \hspace*{-0.4cm}\includegraphics[width=1.15\linewidth]{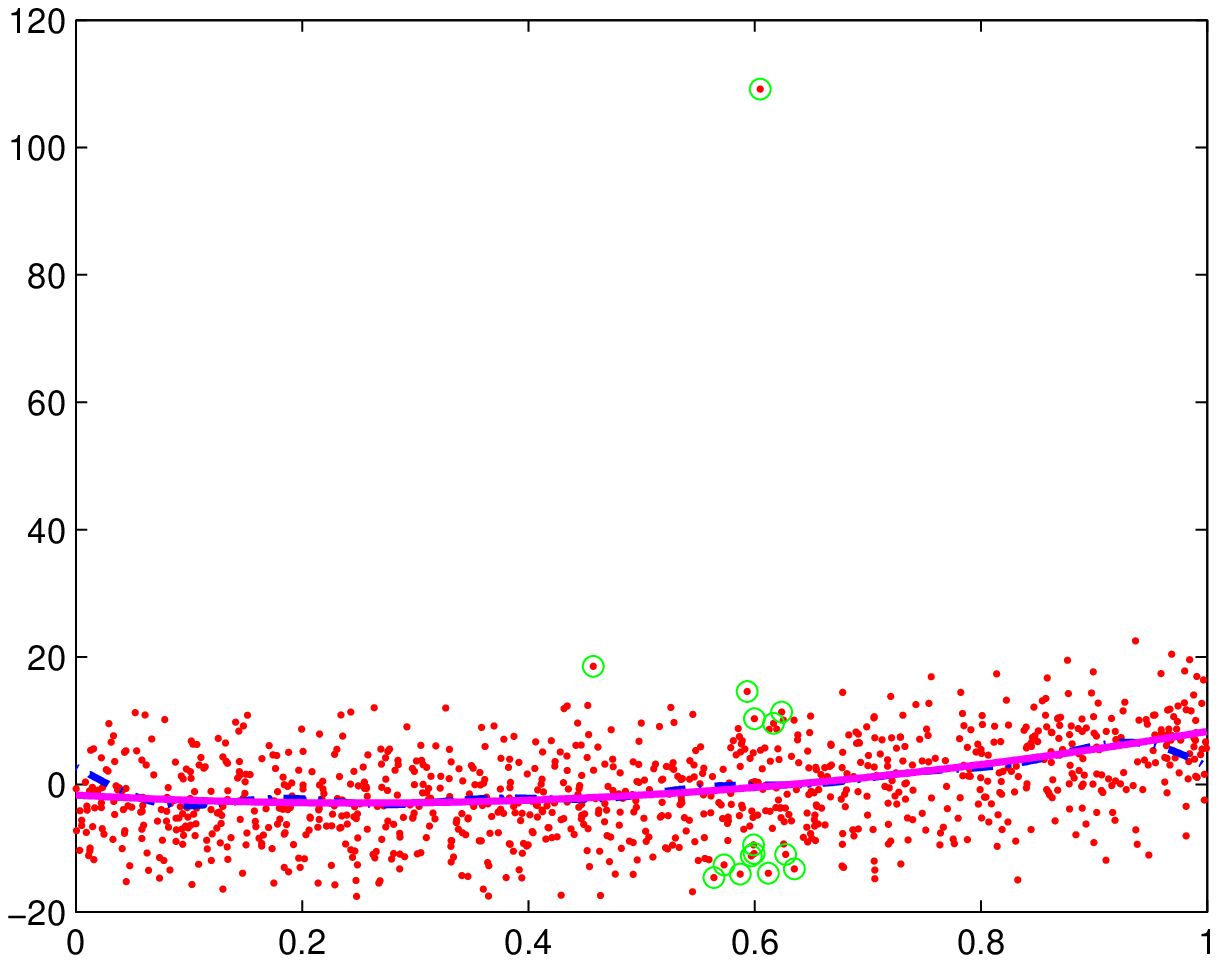}
   \end{minipage}
   
   \vspace*{-.4cm}\begin{minipage}{.49\linewidth}
      \hspace*{-0.8cm}\includegraphics[width=1.15\linewidth]{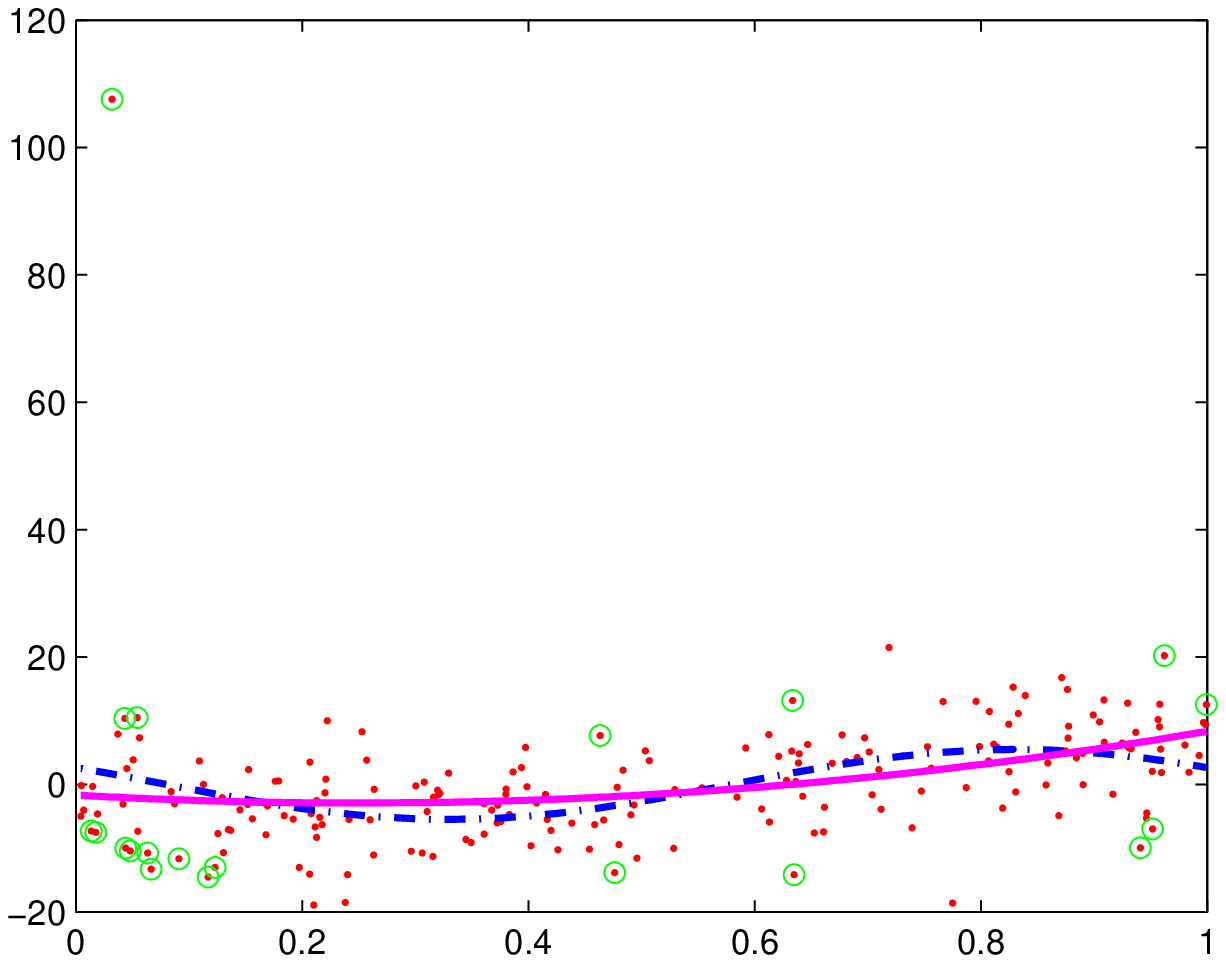}
   \end{minipage} 
   \hfill
   \begin{minipage}{0.49\linewidth}
      \hspace*{-0.4cm}\includegraphics[width=1.15\linewidth]{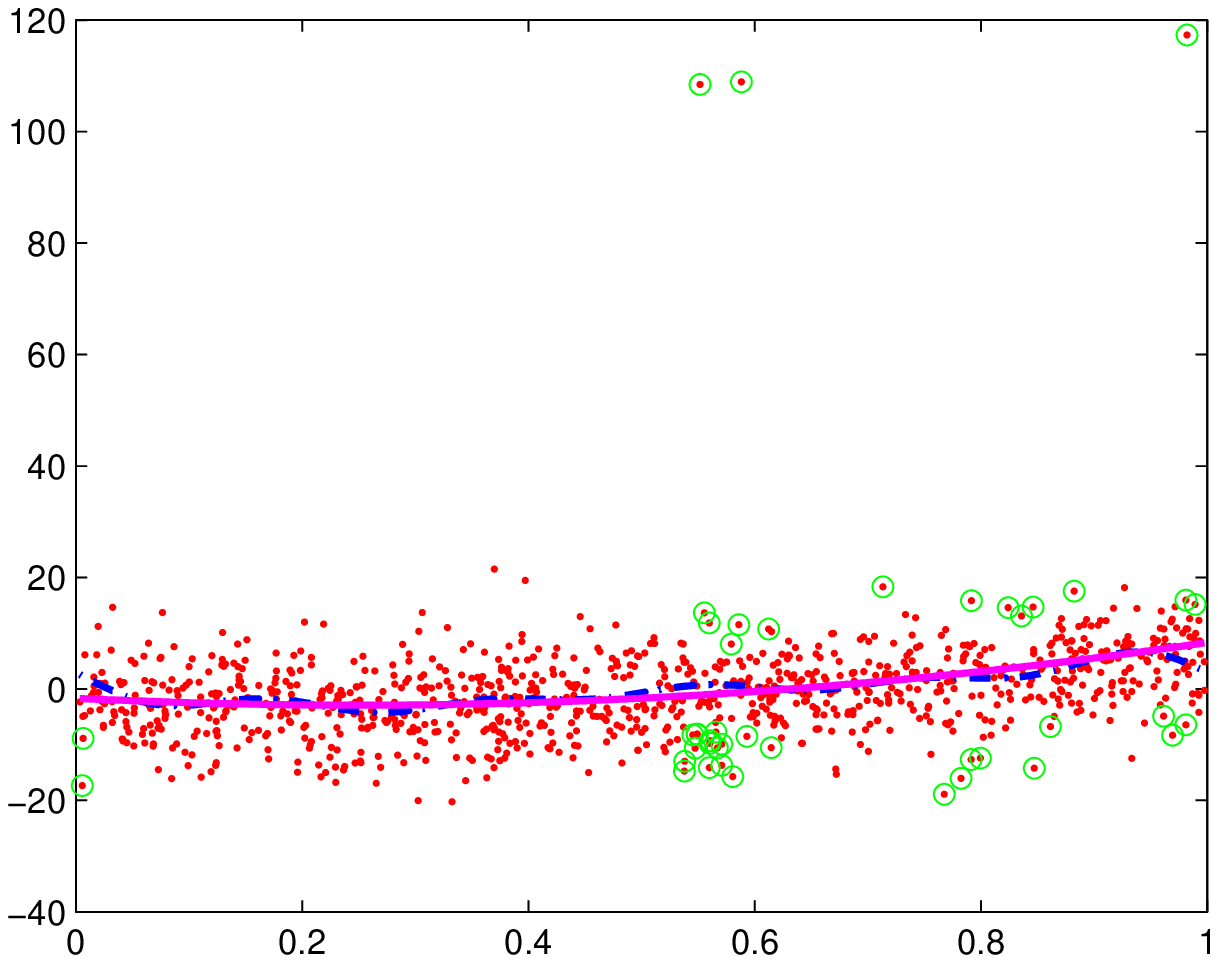}
   \end{minipage}
   
   \vspace*{-.4cm}\begin{minipage}{.49\linewidth}
      \hspace*{-0.8cm}\includegraphics[width=1.15\linewidth]{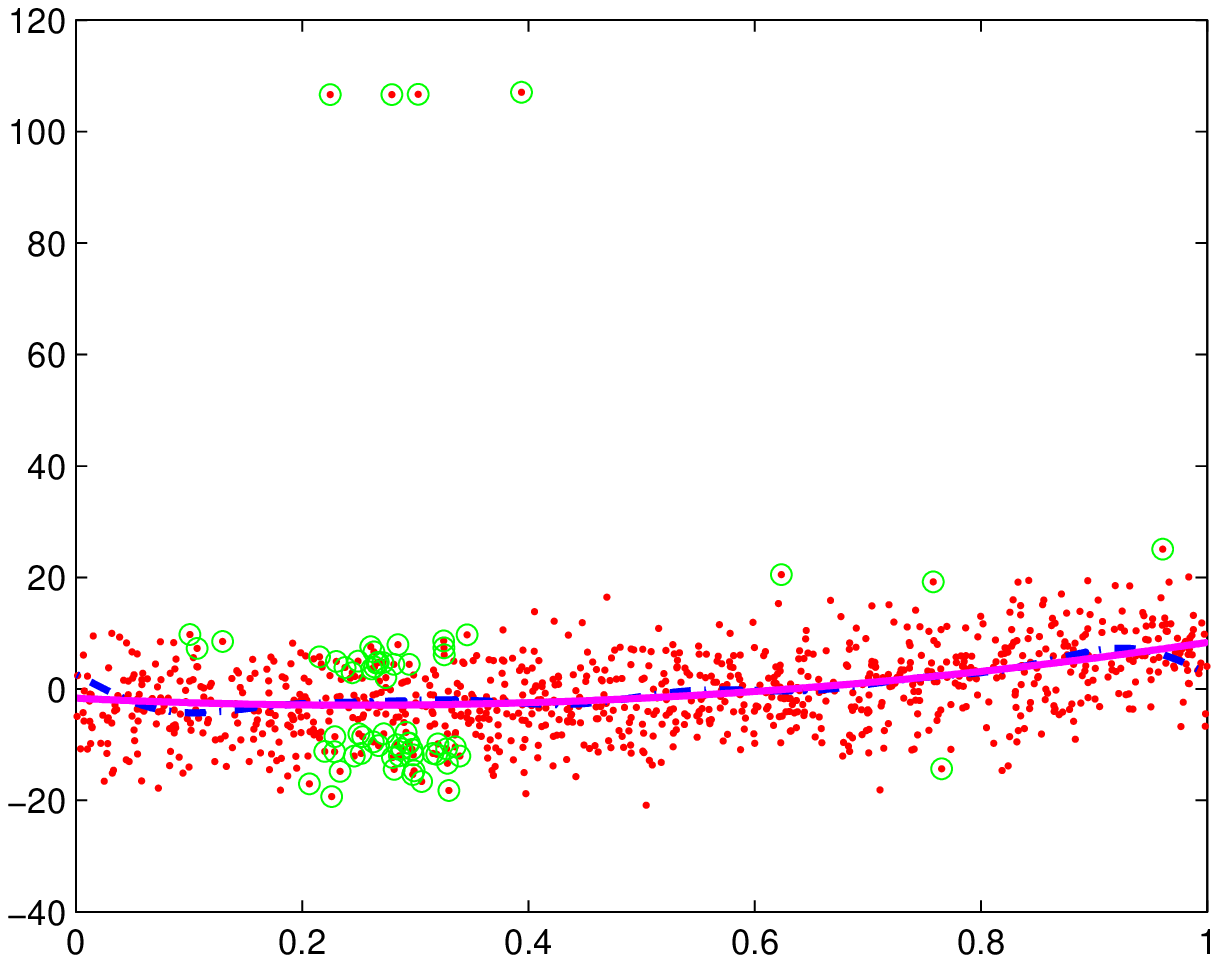}
   \end{minipage} 
   \hfill
   \begin{minipage}{0.49\linewidth}
      \hspace*{-0.4cm}\includegraphics[width=1.15\linewidth]{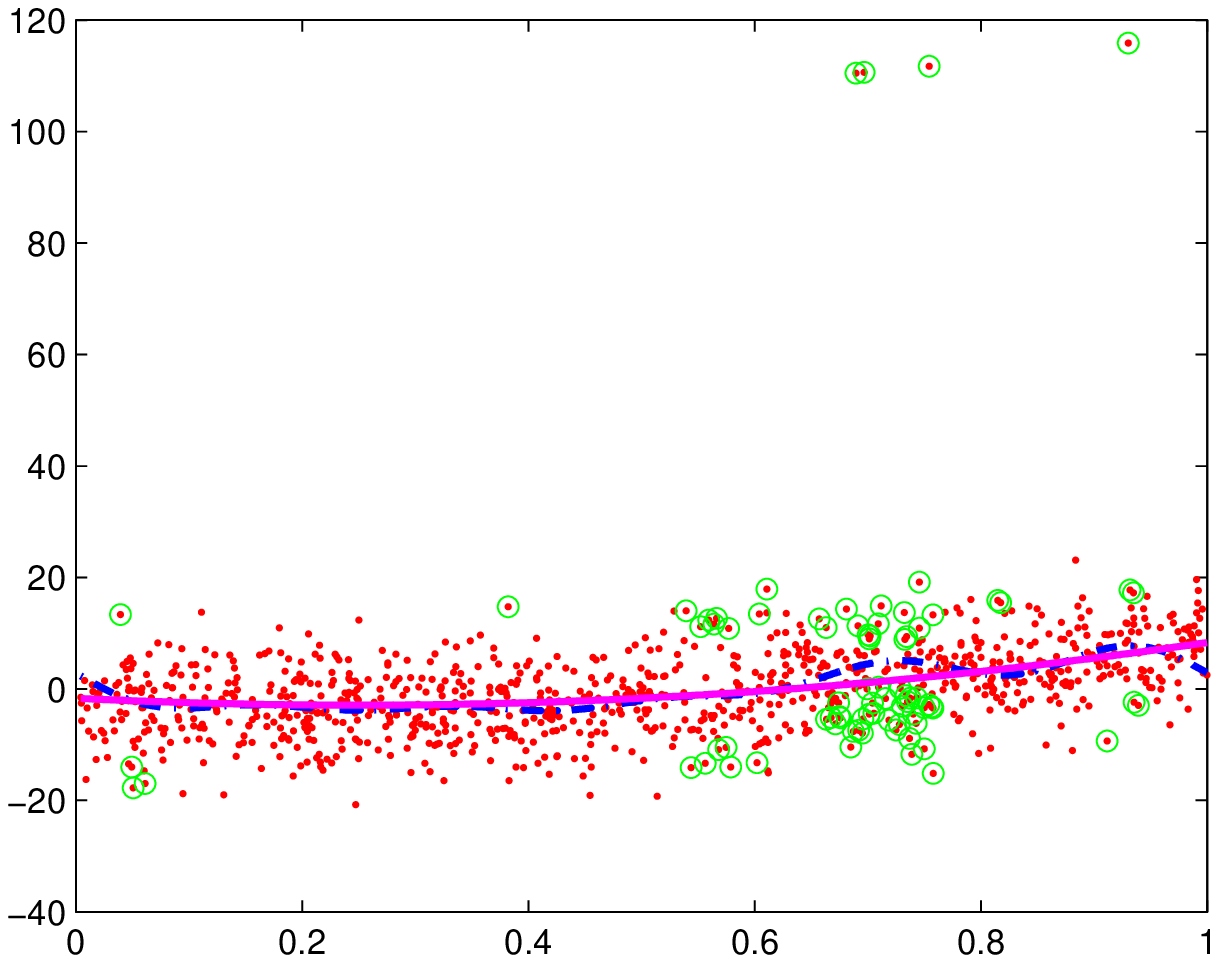}
   \end{minipage}
\end{figure}

\pagebreak

\begin{figure}[!ht] 
\vspace*{-1.5cm}
\caption{Surrounding points are the points of the training set generated several times from $TS(200,2)$ (with the heavy-tailed noise) that are not taken into account in the min-max truncated estimator (to the extent that the estimator would not change by removing these points). 
The min-max truncated estimator $x\mapsto\hf(x)$ appears in dash-dot line, while $x\mapsto\E(Y|X=x)$ is in solid line. In these six simulations, it outperforms the ordinary least squares estimator. Note that in the last figure, it does not consider $64$ points among the $200$ training points.} \label{fig:2}
\centering
   \vspace*{-0cm}\begin{minipage}{.49\linewidth}
      \hspace*{-0.8cm}\includegraphics[width=1.15\linewidth]{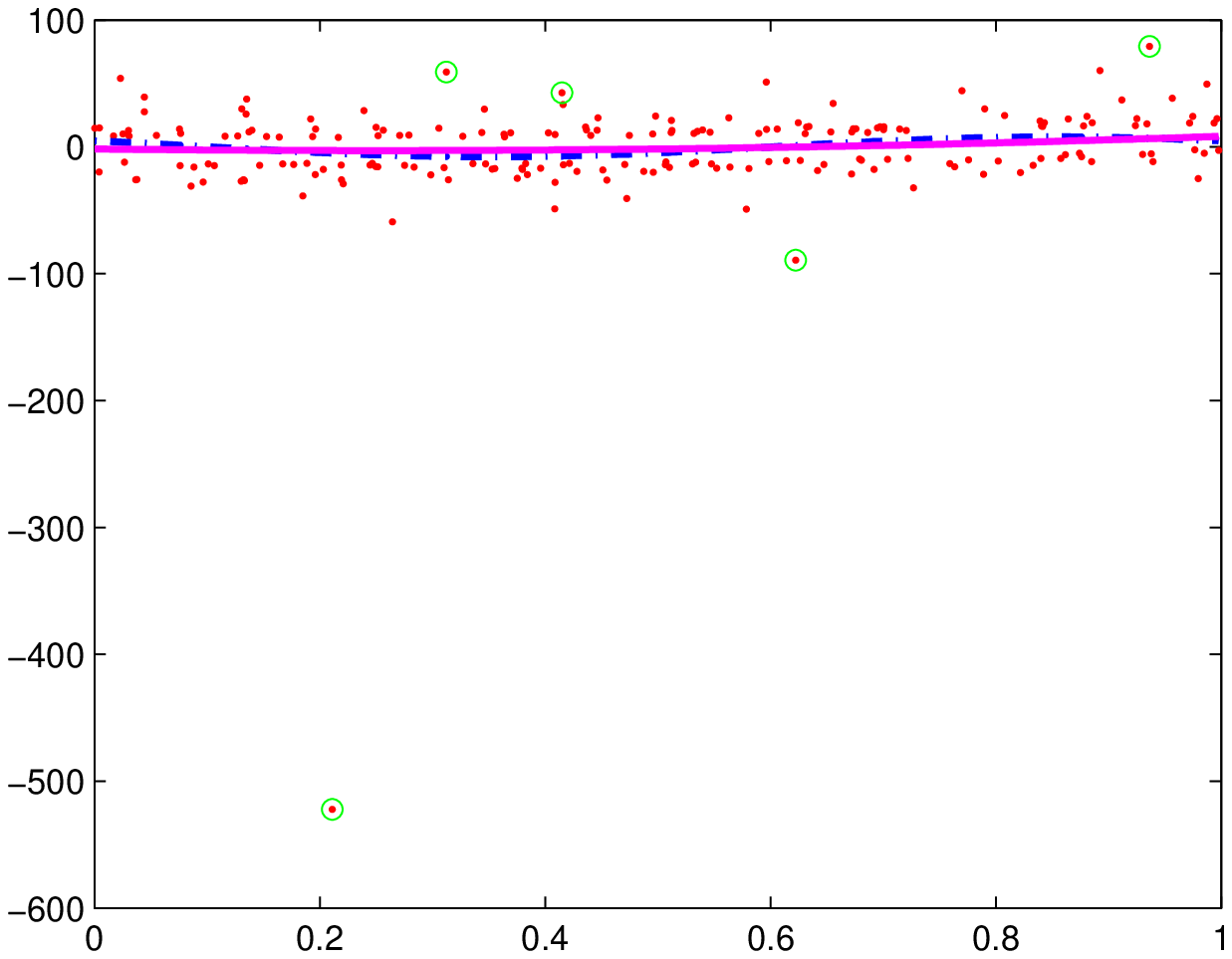}
   \end{minipage} 
   \hfill
   \begin{minipage}{0.49\linewidth}
      \hspace*{-0.4cm}\includegraphics[width=1.15\linewidth]{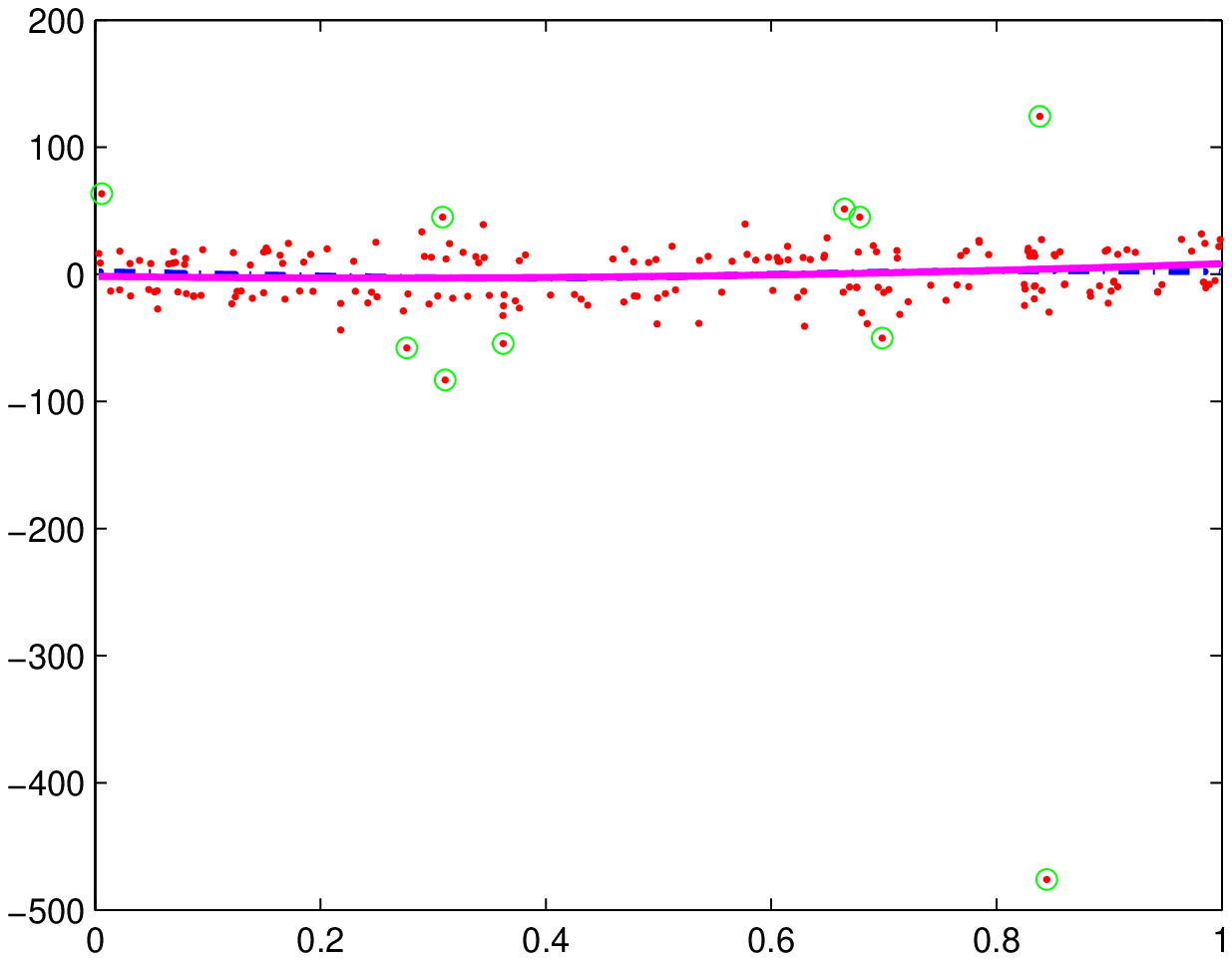}
   \end{minipage}
   
   \vspace*{-.4cm}\begin{minipage}{.49\linewidth}
      \hspace*{-0.8cm}\includegraphics[width=1.15\linewidth]{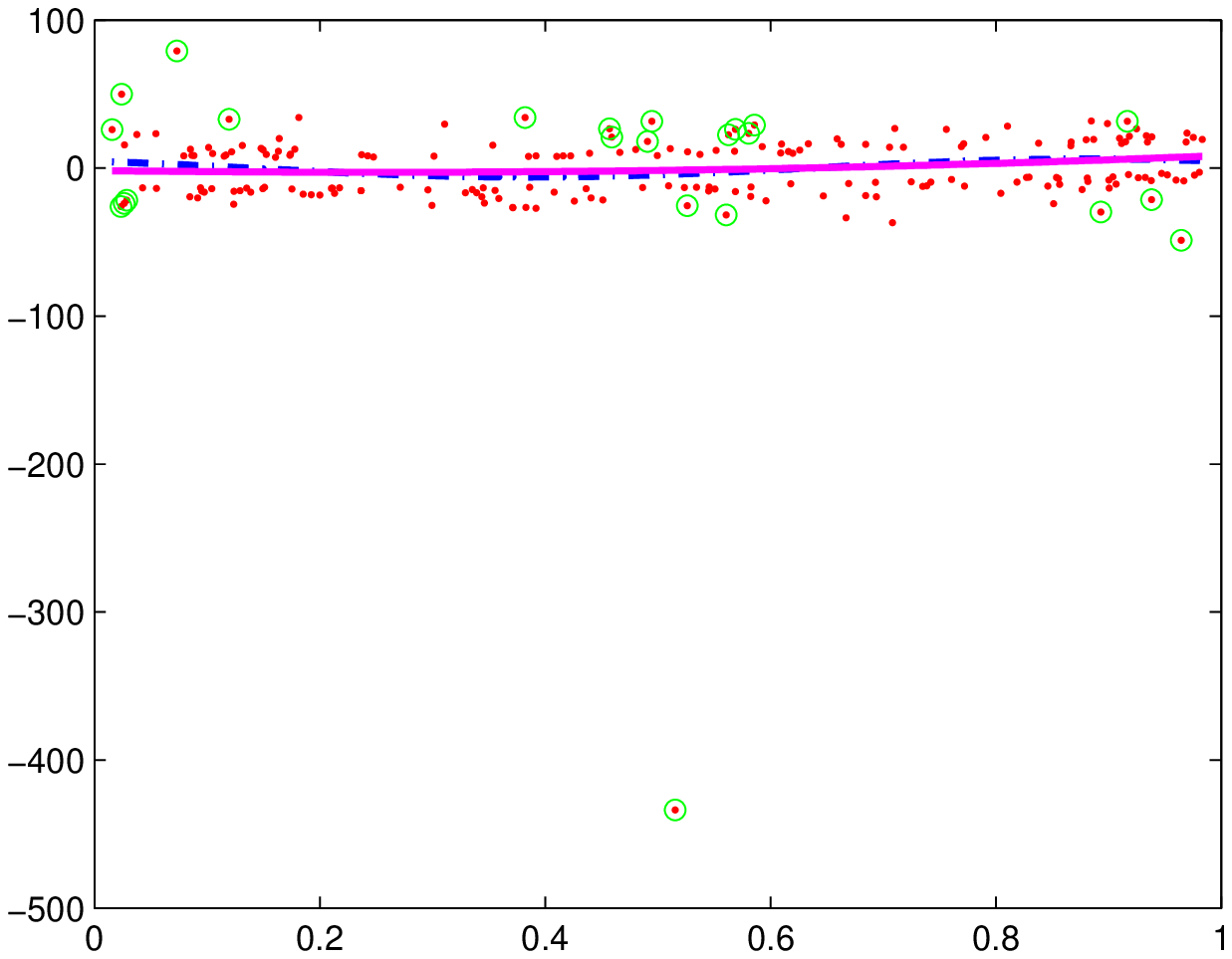}
   \end{minipage} 
   \hfill
   \begin{minipage}{0.49\linewidth}
      \hspace*{-0.4cm}\includegraphics[width=1.15\linewidth]{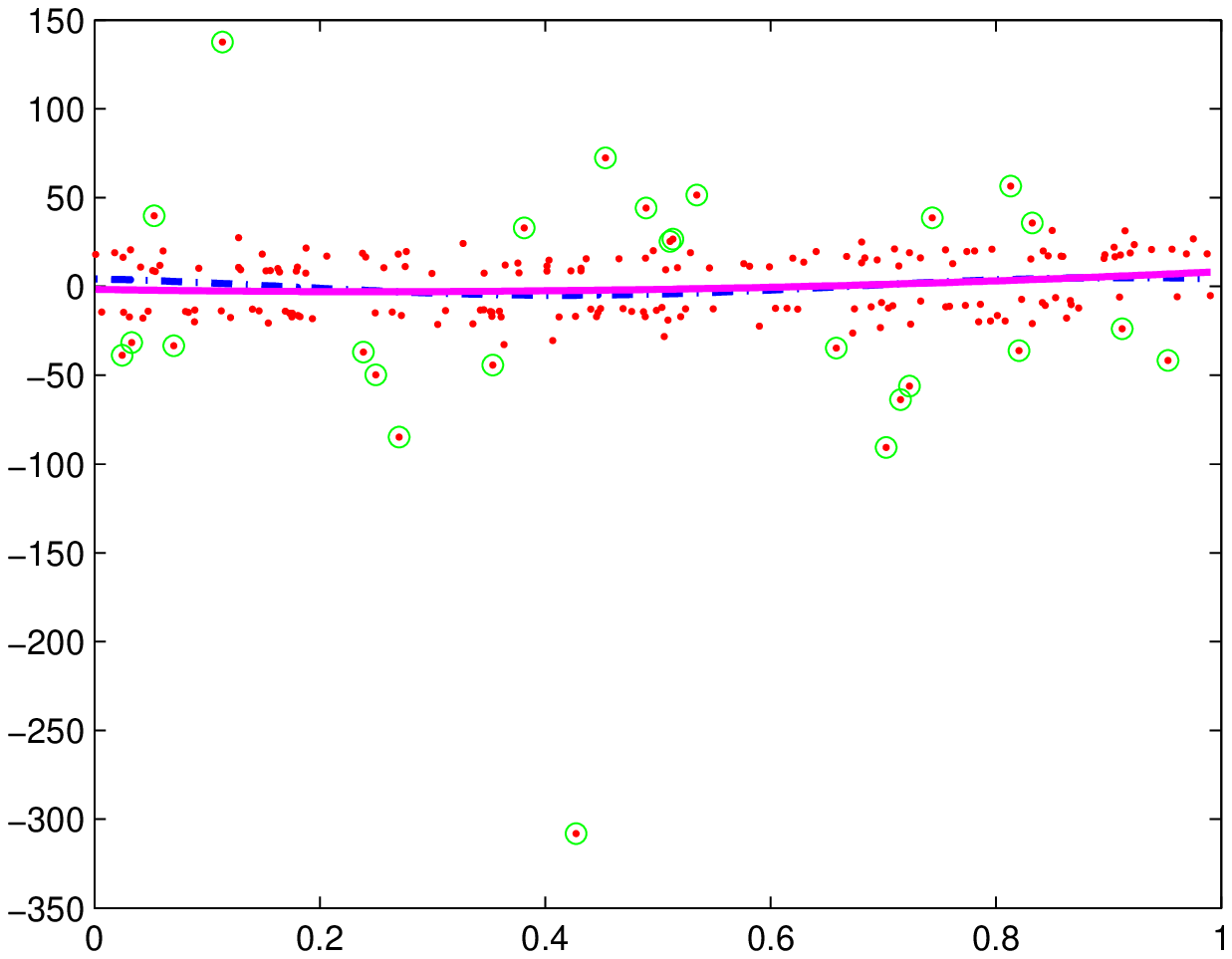}
   \end{minipage}
   
   \vspace*{-.4cm}\begin{minipage}{.49\linewidth}
      \hspace*{-0.8cm}\includegraphics[width=1.15\linewidth]{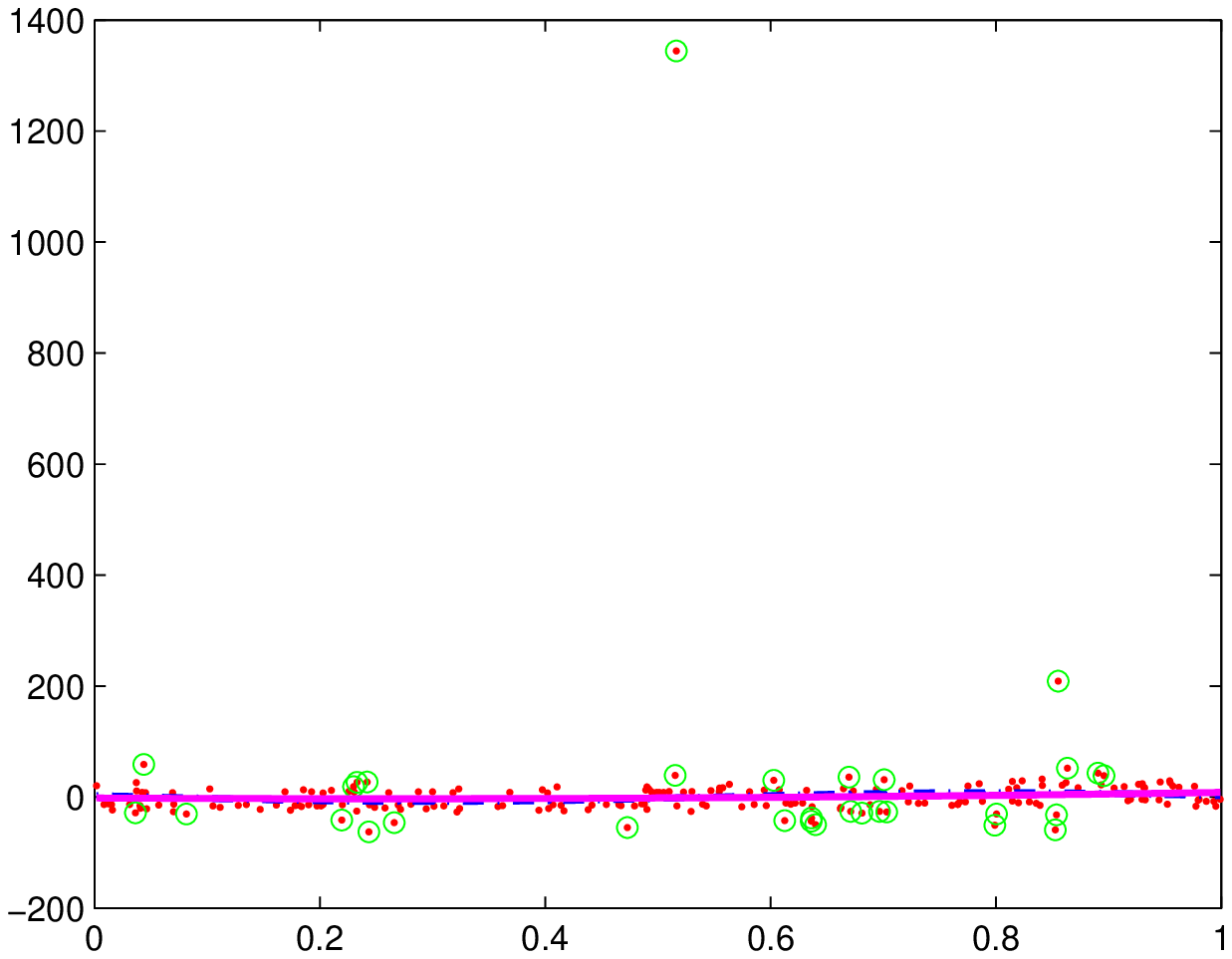}
   \end{minipage} 
   \hfill
   \begin{minipage}{0.49\linewidth}
      \hspace*{-0.4cm}\includegraphics[width=1.15\linewidth]{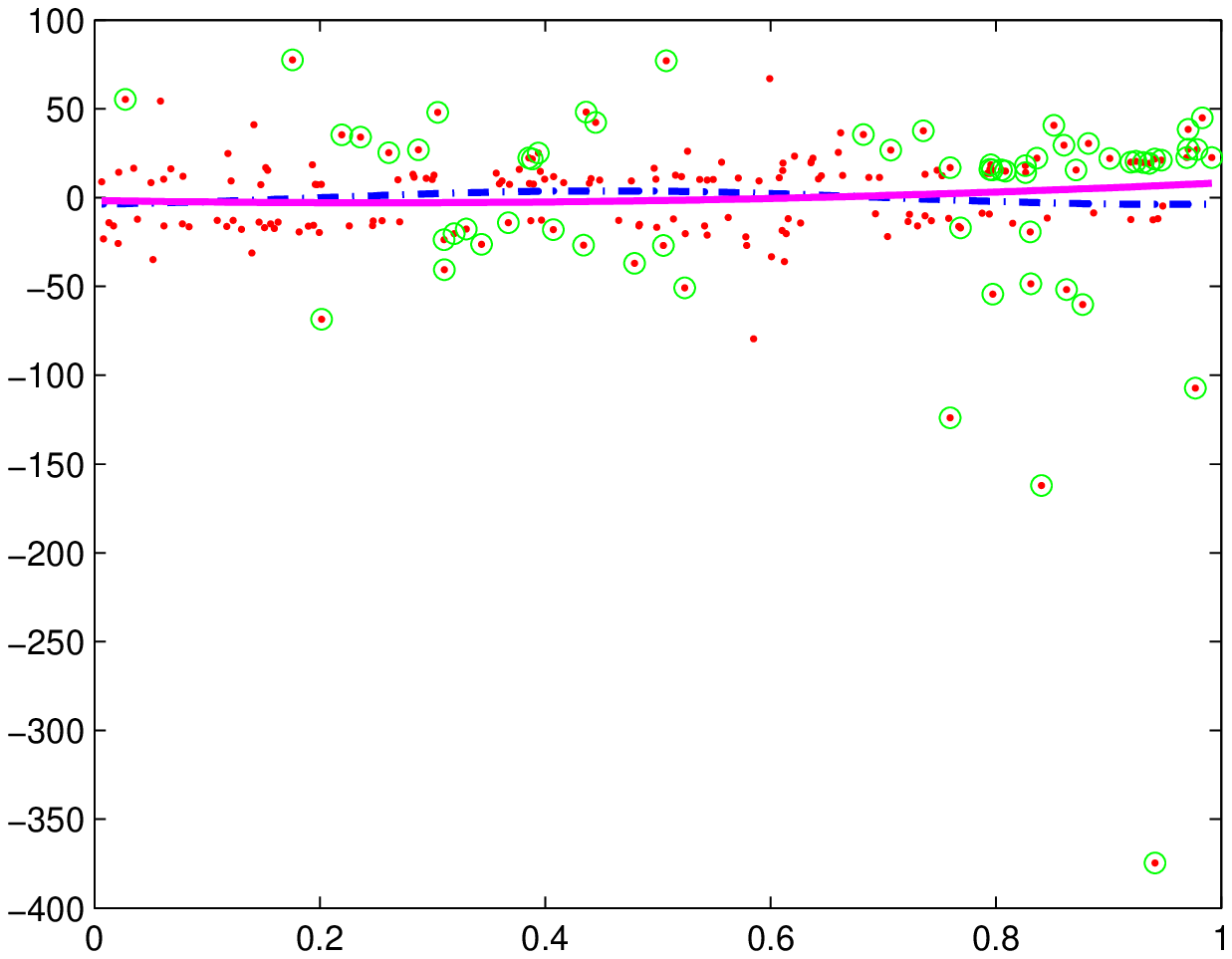}
   \end{minipage}
\end{figure}

\pagebreak

\bibliographystyle{plain}
\bibliography{ref}

\end{document}